\documentclass[12pt]{article}
\usepackage{amsmath}
\usepackage{graphicx,psfrag,epsf}
\usepackage{enumerate}
\usepackage{natbib}
\usepackage{url} 
\usepackage{amsfonts}
\usepackage{verbatim}
\usepackage{ntheorem}
\usepackage{multirow}

\usepackage{algorithm}
\usepackage[noend]{algpseudocode}

\makeatletter
\def\algbackskip{\hskip-\ALG@thistlm}
\makeatother

\newtheorem{theorem}{Theorem}
\newtheorem{lemma}{Lemma}

\newtheorem*{proof}{Proof}[section]
\newtheorem*{remark}{Remark}
\newtheorem{corollary}{Corollary}[theorem]
\newcommand{\blind}{1}

\DeclareMathOperator*{\argmin}{arg\,min}
\begin{document}

\def\spacingset#1{\renewcommand{\baselinestretch}%
{#1}\small\normalsize} \spacingset{1}


\if1\blind
{
  \title{\bf Quantile Off-Policy Evaluation via Deep Conditional Generative Learning}
  \author{
Yang Xu$^{1}$, Chengchun Shi$^{2}$, Shikai Luo$^{3}$, Lan Wang$^4$ and Rui Song $^{1}$\\
\textit{$^{1}$North Carolina State University}\\
\textit{$^{2}$London School of Economics and Political Science}\\
\textit{$^{3}$ByteDance}\\
\textit{$^{4}$University of Miami} 
}
\date{\empty}
  \maketitle
} \fi

\if0\blind
{
  \bigskip
  \bigskip
  \bigskip
  \begin{center}
    {\LARGE\bf Quantile Off-Policy Evaluation via Deep Conditional Generative Learning}
\end{center}
  \medskip
} \fi

\bigskip

\begin{abstract}
Off-Policy evaluation (OPE) is concerned with evaluating a new target policy using offline data generated by a potentially different behavior policy. It is
critical in a number of sequential decision making problems ranging from healthcare to technology industries. 
Most of the work in existing literature is focused on evaluating the mean outcome of a given policy, and ignores the variability of the outcome. However, in a variety of applications, criteria other than the mean may be more sensible. For example, when the reward distribution is skewed and asymmetric, quantile-based metrics are often preferred for their robustness. 
In this paper, we propose a doubly-robust inference procedure for quantile OPE in sequential decision making and study its asymptotic properties. In particular, we propose utilizing state-of-the-art deep conditional generative learning methods to handle parameter-dependent nuisance function estimation. We demonstrate the advantages of this proposed estimator through both simulations and a real-world dataset from a short-video platform. In particular, we find that our proposed estimator outperforms classical OPE estimators for the mean in settings with heavy-tailed reward distributions.  
\end{abstract}

\noindent%
{\it Keywords: quantile treatment effect, off-policy evaluation, sequential decision making, deep generative model, uncertainty quantification.}  
\vfill

\newpage
\spacingset{1.45} 

\section{Introduction}
\label{sec:intro}
Policy evaluation plays an important role in a large variety of real-world applications. In precision medicine \citep[see e.g.,][]{murphy2003optimal,kosorok2019precision,tsiatis2019dynamic}, physicians wish to know the impact of implementing an individualized treatment strategy. In technological companies, decision makers conduct A/B testing to evaluate the performance of a newly developed product \citep[see e.g.,][]{bojinov2019time,shi2022dynamic}. In economics and social science, program evaluation methods are widely applied to access the effects of certain intervention policies \citep[see][for an overview]{abadie2018econometric}. In the aforementioned applications, a new treatment decision rule and/or policy needs to be evaluated offline before online validation. This motivates us to study the off-policy evaluation (OPE) problem that aims to evaluate the mean outcome under a target policy from a historical dataset.  

Offline evaluation of a fixed, data-dependent or optimal policy has been widely studied in statistics \citep{zhang2012robust,chakraborty2013inference,zhang2013robust,matsouaka2014evaluating,luedtke2016statistical,luckett2020estimating,liao2020batch,shi2020statistical,liao2021off} and machine learning \citep{dudik2011doubly, jiang2016doubly, thomas2016data, cai2020deep, kallus2020confounding, kallus2020double,hao2021bootstrapping}.

Nonetheless, all the aforementioned works focused on the mean outcome. In many applications, criteria other than the mean may be more sensible. This includes the case where the outcome distribution is skewed or heavy-tailed. Take our real data application as an example. The data is collected from a world-leading technological company with one of the largest mobile platforms for production, aggregation and distribution of short-form videos. We focus on evaluating the impact of certain advertising policies on user experience. 
The response variable corresponds to the average length of stay (e.g., the average period of time users stay on the platform) and is extremely heavy-tailed. See Figure \ref{fig:real_data_analysis_loglog} in Section \ref{sec:realdata} for details. In other cases, policymakers may care about some distributional effect beyond the mean. In our data application, because policies that raise the lower quartile are more likely to attract potential customers, a policy's effect on the lower quartile of length of stay is of greater interest to policymakers.
Recently, a few methods have been proposed for distributional policy evaluation and/or learning. In particular, \citet{donald2014estimation} and \citet{huang2021off} considered estimating the cumulative distribution function of a target policy's reward. \citet{wang2018quantile} and \citet{qi2019estimating} considered learning an optimal policy to maximize the average lower tail of the potential outcome distribution. \citet{kallus2019localized} developed a localized debiased machine learning method to evaluate the quantile treatment effect. However, all the aforementioned works focused mainly on single-stage decision making.

This paper is concerned with statistical inference of the quantiles of a target policy's return using a pre-collected dataset from multi-stage studies. Efficient off-policy evaluation of the quantile value is known to be difficult \citep{kallus2019localized}. The major challenge lies in that the efficient estimating equation involves parameter-dependent conditional mean functions. Accordingly, efficient estimation of the quantile requires learning the conditional mean function for various values of the target parameter. Hence, directly applying supervised learning algorithms is very computationally intensive.

Our contributions can be summarized as follows: (i) We propose doubly-robust inference procedures for quantile OPE in sequential decision making. In particular, we develop a doubly-robust quantile value estimator and a doubly-robust estimator for its asymptotic variance. To the best of our knowledge, this is the first work that rigorously investigates uncertainty quantification for a target policy's quantile value in multi-stage studies. (ii) We propose utilizing existing state-of-the-art deep conditional generative learning methods to handle parameter-dependent nuisance function estimation. Our proposal is the first to embrace the power of deep conditional generative learning to address a challenging statistical inference problem in OPE. (iii) We further develop a tail-robust doubly-robust estimator for the off-policy mean outcome by aggregating the proposed quantile estimator at multiple quantile values. Our simulation study and real data analysis show that the mean squared error (MSE) of the proposed estimator is around 15\%-50\% less than that of the standard doubly-robust estimator with heavy-tailed responses.

The rest of the paper is organized as follows. In Section \ref{sec:1stage_est}, we introduce the proposed doubly-robust estimation and inference procedure for single-stage decision making. In Section \ref{sec:2stage_est}, we extend our proposal to the multi-stage setting. 
In Section \ref{sec:theory}, we study the asymptotic properties of the proposed procedure. 
Simulation studies and real data analysis are conducted in Sections \ref{sec:numerical} and \ref{sec:realdata}, respectively. Finally, we conclude our paper in Section \ref{sec:conc}. All the proofs are presented in the supplementary material.

\section{Single-Stage Decision Making}\label{sec:1stage_est}
\subsection{Preliminaries}
We first describe the observed data. 
Let $X\in \mathcal{X}$ denote the $d$-dimensional vector that summarizes the baseline information (e.g., age, gender) of a given subject before treatment assignment occurs, $A \in \mathcal{A}=\{1,\cdots,m\}$ denote the treatment assigned to that subject where $m$ denotes the number of treatment options, and $R\in \mathbb{R}$ denote the subject's outcome (the larger the better by convention). 
The observed data can thus be summarized as $N$ i.i.d. covariates-treatment-outcome triplets $\{X_i,A_i,R_i\}_{i=1}^N$ where $N$ denotes the number of subjects. 

We next introduce the potential outcomes. For any $1\le a\le m$, let $R^*(a)$ denote the potential outcome had the subject received the treatment $a$. Let $b: \mathcal{X}\times  \mathcal{A}\mapsto [0,1]$ denote the behavior policy that generates the data, i.e., $b(a|x)=\mathbb{P}(A=a|X=x)$. Likewise, let $\pi:\mathcal{X}\times  \mathcal{A}\mapsto [0,1]$ denote the target policy we wish to evaluate. Notice that $b$ is allowed to be unknown, as in observational studies. 

Finally, we introduce the following standard assumptions: \vspace{-0.5em}
\begin{enumerate}
    \item[(C1)] Consistency: $R=R^*(A)$. 
    \vspace{-1em}
    \item[(C2)] No unmeasured confounders \citep{rosenbaum1983central}: $A\perp (R^*(1),\cdots,R^*(m))|X$. 
    
    \vspace{-1em}
    \item[(C3)] Positivity: There exists some constant $\epsilon>0$, such that $\mathbb{P}(b(A|X)\ge \epsilon)=1$. 
\end{enumerate}
We remark that (C1)-(C3) are commonly imposed in causal inference and individualized treatment regimes literature \citep[see e.g.,][]{zhang2012robust,wang2018quantile,nie2021quasi}. They allow us to infer the potential outcome distribution from the observed data. (C2) and (C3) are automatically satisfied in randomized studies where the behavior policy is usually a strictly positive function independent of $x$. 
\subsection{Quantile Off-Policy Evaluation}
Let $R^*(\pi)$ denote the potential outcome variable following the target policy $\pi$. It corresponds to a mixture of $\{R^*(a): 1\le a\le m\}$ with weights specified by the target policy $\pi$. More specifically, we have
\begin{eqnarray}\label{eqn:quantile}
    \mathbb{P}(R^*(\pi)\le r|X)=\sum_{a=1}^m \pi(a|X) \mathbb{P}(R^*(a)\le r|X).
\end{eqnarray}
For a given quantile level $0<\tau<1$, our objective lies in inferring the $\tau$th quantile of $R^*(\pi)$, denoted by $\eta_{\tau}$. 

Under (C1)-(C3), the conditional potential outcome distribution $\mathbb{P}_{R^*(a)|X}$ is equal to the conditional distribution of observed outcome $\mathbb{P}_{R|X,A}$. This together with \eqref{eqn:quantile} allows us to identify the quantile value from the observed data. In the following, we first introduce the direct method and the inverse probability weighting method, and then combine the two for efficient and robust quantile OPE. 

\textbf{1. Direct method (DM).} Let  $\rho_{\tau}(u)=u\cdot\left(\tau-\mathbb{I}\{u<0\}\right)$ denote the quantile loss function. Notice that $\eta_{\tau}=\argmin_{\eta} \mathbb{E} [\rho_{\tau}(R^*(\pi)-\eta)]=\argmin_{\eta} \mathbb{E}[\mathbb{E} \{\rho_{\tau}(R^*(\pi)-\eta)|X\}]$. Under (C1)-(C3), it follows from \eqref{eqn:quantile} and iterative expectation formula that
\begin{eqnarray}\label{eqn:quantile2}
    \widehat{\eta}_{\text{DM}}=\argmin_{\eta} \sum_{a=1}^m \widehat{\mathbb{E}}[\pi(a|X)\cdot \widehat{\mathbb{E}}\{ \rho_{\tau}(R-\eta)|A=a,X\}],
\end{eqnarray}
where $\widehat{\mathbb{E}}$ denotes the empirical expectation calculated from the observed data. DM first estimates the conditional expectation within the square brackets for each $\eta$ and then plugs in this estimator in the above optimization problem to learn $\eta_{\tau}$. 
As discussed in the introduction, it remains challenging to compute the conditional mean estimator for each $\eta$. We will later introduce our proposal to address this issue. 

\textbf{2. Inverse probability weighting (IPW) method.} By the change of measure theorem, it follows from \eqref{eqn:quantile2} that
\begin{eqnarray*}
    \widehat{\eta}_{\text{IPW}}=\argmin_{\eta} \widehat{\mathbb{E}} \left[\frac{\pi(A|X)}{\widehat{b}(A|X)} \rho_{\tau}(R-\eta)\right],
\end{eqnarray*}
An inverse probability weighting estimator, akin to the off-policy estimator proposed by \cite{wang2018quantile}, can be developed from this optimization problem to derive an estimating equation for $\eta_{\tau}$.  When the behavior policy $b$ is unknown, we can apply existing state-of-the-art supervised learning methods to estimate it.

\textbf{3. Doubly robust (DR) estimator.} The doubly-robust method combines the DM and IPW procedures and eases model misspecification concerns by only requiring one of the nuisance functions to be correctly specified. The new DR estimator is given by
\begin{align}\label{eq:4}
\begin{split}
    \widehat{\eta}_{\textrm{DR}}&=\argmin_{\eta} \frac{1}{N}\sum_{i=1}^N \Psi(\eta;X_i,A_i,R_i,\widehat{b},\widehat{\mathbb{E}})\\
    &\equiv\argmin_{\eta} \frac{1}{N}\sum_{i=1}^N\left[\frac{\pi(A_i|X_i)}{\widehat{b}(A_i|X_i)}\rho_{\tau}(R_i-\eta)+\left(1-\frac{\pi(A_i|X_i)}{\widehat{b}(A_i|X_i)}\right) \widehat{\mathbb{E}}\{\rho_{\tau}(R_i-\eta)|A_i\sim \pi, X_i\}\right],
\end{split}
\end{align}
for some estimated behavior policy $\widehat{b}$ and conditional mean functional $\widehat{\mathbb{E}}\{\cdot\}$. The conditional expectation is estimated given $(A_i\sim \pi,X_i)$, which denotes that $A_i$ is randomly sampled from the target policy $\pi$ with baseline information $X_i$.
 
The first term of the second line in \eqref{eq:4} is identical to the IPW estimator. The second part is an augmentation term offering additional information when the behavior policy is not aligned with the target, which improves estimation efficiency. When the conditional mean function is correctly specified, the expectation of our DR estimator is equivalent to that of the DM estimator. When the behavior policy is correctly specified, the estimator is unbiased to the IPW estimator and thus is also consistent. This verifies that $\widehat{\eta}_{\text{DR}}$, as a combination of DM and IPW estimators, indeed satisfies the doubly-robustness property. 

Similar to DM, DR requires estimating the conditional mean function for each $\eta$. We propose employing a deep conditional generative learning method to compute $\widehat{\mathbb{E}}\{\cdot\}$. The main idea is utilizing deep neural networks to learn a conditional sampler that takes the covariate-treatment pair and some random noises as input and outputs pseudo outcomes whose conditional distribution is similar to that of $\mathbb{P}_{R|A,X}$. By using the Monte Carlo method, these pseudo outcomes can be further used to approximate the conditional mean functions for any $\eta$. This addresses the challenge of parameter-dependent nuisance function estimation. As a powerful approach for generating random samples from complex and high-dimensional conditional distributions, deep conditional generative learning methods have been successfully implemented in a large variety of applications including computer vision, imaging processing, epidemiology simulation and artificial intelligence  \citep{yan2016attribute2image, shu2017bottleneck, wang2018cvpr, davis2020use, jo2021srflow}. We detail our procedures in the next section.

\subsection{Estimation and Inference Procedures}\label{sec:1stage_est_details}

Our proposal consists of four steps: 1) sample splitting, 2) nuisance functions estimation, 3) doubly-robust quantile value and variance estimation, 4) tail-robust mean value estimation. Pseudocode summarizing the proposed algorithm is given in Algorithm \ref{alg:pseudocode}.
  
\begin{algorithm}[tbh]
    \caption{Pseudo Code of quantile off policy evaluation in single stage settings}\label{euclid}
    \label{alg:pseudocode}
    \vspace{0.08in}
    \textbf{Input:} $\{X_i,A_i,R_i\}_{i=1}^N$ triplets; quantile levels $\tau$; number of folds $S$, the maximum \# iterations $\max_{iter}$. \\
    \textbf{Output:} doubly robust quantile estimator $\widehat{\eta}_{\tau}^{\text{DR}}$; tail-robust DR mean estimator $\hat\mu_{\text{DR}}$.
\begin{algorithmic}[1]
\Procedure{A. Doubly Robust Quantile Estimator}{}
\State Randomly split the data into $S$ folds. For any $s\in S$, denotes $\mathcal{I}_s$ as the $s$th subgroup.
    \ForAll{$s\in S$}
        \State Estimate $\widehat{b}$ by GBDT with the data in $\mathcal{I}_s^c$;
        \State Estimate $\widehat{R}|(X,A)$ by MDN with the data in $\mathcal{I}_s^c$;
        \State Generate $\{\widehat{R}_{x_i,a}^j|(X=x_i,A=a)\}_{j=1}^M$ for any $a\in \mathcal{A}$ and $i\in \mathcal{I}_s$;
        \EndFor
        \While{$iter<\max_{iter}$ and $|(\eta_{iter+1}-\eta_{iter})/\eta_{iter}|>\epsilon$}
            \State $iter \leftarrow iter+1$
            \State Update $\eta_{iter}$ by solving formula (\ref{eq:4}) by gradient descent;
        \EndWhile
    \State $\widehat{\eta}_{\text{DR}}\leftarrow \eta_{iter}$
    \EndProcedure
    \Procedure{B. Tail-Robust DR Mean Estimator}{}
    \State For all of the quantile levels of our interest, calculate $\widehat{\eta}_{\tau}$ by procedure A. Averaging among all quantiles yields the tail-robust DR mean Estimator $\widehat{\mu}_{\text{DR}}$.
    \EndProcedure
\end{algorithmic}
\end{algorithm}

\subsubsection{Step 1: Sample Splitting}

Step 1 of the proposed algorithm is to randomly divide samples $\{1,\dots,N\}$ into $S$ disjoint subgroups of equal size. We define $n = \frac{N}{S}\in \mathbb{Z}$ . 
This step couples \eqref{eq:4} with cross-fitting in Step 3 to reduce the bias of the nuisance function estimators obtained by the aforementioned supervised and generative learning methods. For any $1\le s\le S$, let $\mathcal{I}_s$ denote the $s$th subgroup and $\mathcal{I}_s^c$ denote its complement. Sample splitting allows us to use the data in $\mathcal{I}_s^c$ to compute $\widehat{b}$ and $\widehat{\mathbb{E}}\{\cdot\}$, and those in $\mathcal{I}_s$ to construct the estimating equation. In addition, we can get full efficiency by aggregating the estimating equations over different $s$. Such a technique has been implemented for policy evaluation in recent literature \citep{chernozhukov2018double,shi2021deeply,kallus2022efficiently}.

\subsubsection{Step 2: Nuisance Functions Estimation}

Step 2 of the proposed algorithm is to use each data subset in $\mathcal{I}_s^c$ to estimate the behavior policy $b$ and the conditional distribution of $R$ given $A$ and $X$. 

Firstly, notice that estimating the behavior policy is essentially a regression problem. 
As such, we can apply any supervised learning 
algorithm designed for categorical response variables 
to compute $\widehat{b}.$ 
 In our implementation, we employ a Gradient Boosting Decision Tree (GBDT) algorithm and find that it works well in our numerical examples.

Secondly, we need to estimate the conditional expectation ${\mathbb{E}}[\rho_{\tau}({R}-\eta)|X=x,A=a]$. As discussed earlier, one approach is to apply supervised learning to approximate the conditional expectation for any $\eta$. However, such a method is computationally intensive when the search space of $\eta$ is large. To facilitate the computation, we leverage a state-of-the-art deep conditional generative learning method, the mixture density network (MDN), to learn the conditional reward distribution. MDN integrates the Gaussian mixture model with deep neural networks. It assumes 
that the conditional density of $R$ given $X$ and $A$ is a mixture of $J$ conditional Gaussian distributions, i.e.,
\[
f(r|x,a)=\sum_{j=1}^J \alpha_j(x,a)\frac{1}{\sqrt{2\pi}\sigma_j(x,a)}\exp\Big\{-\frac{(r-\mu_j(x,a))^2}{2\sigma_j^2(x,a)}\Big\},
\]
with weights $\{\alpha_j(x,a)\}_{j=1}^J\in[0,1]$, conditional means $\{\mu_j(x,a)\}_{j=1}^J$, and variances $\{\sigma_j(x,a)\}_{j=1}^J$ parameterized via deep neural networks. The universal approximation property of deep neural nets allows to approximate any smooth or non-smooth weight, conditional mean and variance functions \citep{imaizumi2019deep}. This together with the universal approximation property of Gaussian mixture models allows $f$ to approximate any smooth conditional density functions \citep[see e.g.,][]{zhou2022}. 

Finally, we discuss how to compute $\widehat{\mathbb{E}}[\rho_{\tau}(R-\eta)|X=x,A=a]$. The idea is to sample i.i.d. pseudo outcomes $\{\widehat{R}^{(j)}_{x,a}:1\le j\le M\}$ based on the estimated MDN and approximate the expectation via Monte Carlo, i.e., 
\begin{eqnarray*}
\widehat{\mathbb{E}}[\rho_{\tau}(R-\eta)|X=x,A=a]\approx \frac{1}{M}\sum_{j=1}^M \rho_{\tau}(\widehat{R}_{x,a}^{(j)}-\eta).
\end{eqnarray*}
The number of Monte Carlo samples $M$ represents a trade-off. Notice that the bias of the estimator is the same for any $M$. A large $M$ improves the accuracy of the estimator at the cost of increasing the computation time. In theory, as shown in Section \ref{sec:theory}, we require $M\to\infty$ to ensure the resulting value estimator achieves the minimum asymptotic variance. In our numerical implementation, we set $M$ to $50$ to achieve a good balance between estimation accuracy and computation cost. 

To sample $\widehat{R}^{(j)}_{x,a}$, we generate standard Gaussian noise $Z$ and categorical variable $U$ from a multinomial distribution with support $\{1,\cdots,J\}$ such that for any $j$, $\mathbb{P}(U=j)=\widehat{\alpha}_j(x,a)$, the estimated weight function. We next set $\widehat{R}_{x,a}^{(j)}=\widehat{\mu}_U(x,a)+\widehat{\sigma}_{U}(x,a) Z$ where  $\{\widehat{\mu}_j\}_{j=1}^J$ and $\{\widehat{\sigma}_{j}\}_{j=1}^J$ correspond to the estimated conditional mean and variance functions. As such, $\widehat{R}^{(j)}_{x,a}$ follows the learned conditional Gaussian mixture model. 

\subsubsection{Step 3. Doubly-Robust Quantile Value and Variance Estimation}\label{sec:single-J0}
Step 3 of the proposed algorithm is to employ cross-fitting to estimate the quantile value and construct the associated confidence interval. In particular, we develop two doubly-robust estimators, one for the quantile value itself and another for the asymptotic variance of the quantile estimator. 

Firstly, let $\widehat{b}_s$ and $\widehat{\mathbb{E}}_s$ denote the estimated behavior policy and conditional mean function, respectively. We use cross-fitting to construct the following estimating equation to compute the doubly-robust estimator,
\begin{eqnarray*}
    \widehat{\eta}_{\tau}^{\textrm{DR}}=\argmin_{\eta}\sum_{s=1}^S\sum_{i\in \mathcal{I}_s} \Psi(\eta;W_i,\widehat{b}_s,\widehat{\mathbb{E}}_s),
\end{eqnarray*}
where $\Psi$ is defined in \eqref{eq:4} and $W_i$ is a shorthand for the triplet $(X_i,A_i,R_i)$. 

Secondly, as shown in Theorem \ref{thm:1} (see Section \ref{sec:theory}), the asymptotic variance of $\widehat{\eta}_{\tau}^{\textrm{DR}}$ is given by the sandwich formula $ J_0^{-2} \textrm{Var}(\partial_{\eta}\Psi(\eta_{\tau}; W_i, b,\mathbb{E}))$ where $J_0:=f_{R^*(\pi)}(\eta_{\tau})$ equals the probability density function of the potential outcome $R^*(\pi)$ evaluated at $\eta_{\tau}$. While the variance term $\textrm{Var}(\partial_{\eta}\Psi(\eta_{\tau}; W_i, b,\mathbb{E}))$ can be consistently estimated via the sampling variance estimator, the pdf is more difficult to deal with. 

Notice that $f_{R^*(\pi)}(\eta_{\tau})$ can be regarded as either the marginal pdf of the reward following $\pi$ and baseline distribution $\mathbb{G}$, or the expectation of a conditional pdf for ${f}_{{R}^{*}({\pi})|X}(\eta_\tau|X)$. When estimating $\widehat{\eta}_\tau^{\text{DR}}$, we have already posited a model for $R|(X,A)$ by MDN. Therefore, when the actions of subjects don't follow policy $\pi$, one would expect a better performance by combining ${f}_{{R}^{*}({\pi})|X}(\eta_\tau|X)$ with Kernel Density Estimation (KDE) to construct a doubly robust estimator for $f_{R^*(\pi)}(\eta_{\tau})$.

Denote $\widehat{f}_s$ as the conditional probability density function obtained by MDN from the data in $\mathcal{I}_s^c$. The direct estimator for $f_{R^*(\pi)}(\eta_{\tau})$ is given by 
\[
\widehat{J}_0^{\text{DM}}=\widehat{f}_{R^*(\pi)}(\eta_\tau)=\frac{1}{N}\sum_{s=1}^S\sum_{i\in \mathcal{I}_s} \sum_a \pi(a|X_i)\widehat{f}_s(\widehat{\eta}_{\textrm{DR}}|a, X_i):=\frac{1}{N}\sum_{s=1}^S\sum_{i\in \mathcal{I}_s} \widehat{f}_{R^*(\pi),s}(\widehat{\eta}_{\textrm{DR}}| X_i).
\]
Next, suppose all the observed rewards follow the target policy $\pi$. We can apply kernel density estimation to the rewards to estimate $J_0$. This yields the IPW estimator
\[
\widehat{J}_0^{\text{IPW}}=\frac{1}{N} \sum_{s=1}^S \sum_{i\in \mathcal{I}_s} \frac{\pi(A_i|X_i)}{\widehat{b}_s(A_i|X_i)}\frac{1}{h} K\left(\frac{R_i-\widehat{\eta}_{\textrm{DR}}}{h}\right)
\]
for some kernel function $K$ and a bandwidth $h>0$. Combining both yields the  doubly-robust estimator, given by
\begin{eqnarray}\label{eq:J0_1stage}
    \widehat{J}_0^{\text{DR}}=\frac{1}{N}\sum_{s=1}^S\sum_{i\in \mathcal{I}_s}\left[\frac{\pi(A_{i}|X_{i})}{\widehat{b}_s(A_{i}|X_{i})}\frac{1}{h}K\left(\frac{\widehat{\eta}_{\text{DR}}-R_{i}}{h}\right)+\left(1-\frac{\pi(A_{i}|X_{i})}{\widehat{b}_s(A_{i}|X_{i})}\right)\widehat{f}_{{R}^{*}({\pi})}(\widehat{\eta}_{\text{DR}}|X_{i})\right].
\end{eqnarray}

A Wald-type confidence interval can thus be derived for $\eta_{\tau}$. 

\subsubsection{Step 4. Tail-Robust Mean Value Estimation}
Step 4 of the proposed algorithm is to develop a mean value estimator based on $\widehat{\eta}_{\tau}^{\textrm{DR}}$. A key observation is that, the target policy's mean value can be represented as an average of the quantile values, i.e., $\int_{0}^1 \eta_{\tau} d\tau$. This motivates us to consider the following plug-in estimator $\int_{0}^1 \widehat{\eta}_{\tau}^{\textrm{DR}} d\tau$ for the mean value. The integral can be approximated numerically based on midpoint rule, trapezoidal rule or Simpson's rule. 

Compared to the standard doubly-robust estimator, the proposed estimator is tail robust in that its consistency requires very mild moment conditions. Specifically, it only requires a well-defined mean value, $\mathbb{E} |R|<\infty$. On the contrary, the standard doubly-robust estimator requires the reward to have at least a finite second-order moment. In our numerical studies, we find that the proposed tail-robust estimator achieves a much smaller MSE when reward is heavy-tailed. 

\section{Sequential Decision Making}\label{sec:2stage_est}
\subsection{Preliminaries}
In this section, we extend our proposal to more general sequential decision making problems with $K$ stages.
The data trajectory for a single subject is given by 
\begin{eqnarray}\label{eqn:datatra}
   (X_1,A_1,R_1,X_2,A_2,R_2,\dots,X_K,A_K,R_K),
\end{eqnarray}
where $(X_k,A_k,R_k)$ denotes the covariates-treatment-outcome triplet observed at the $k$th decision stage and $K$ denotes the horizon (e.g., number of decision stages). 

A (history-dependent) policy $\pi=(\pi_1,\cdots,\pi_K)$ maps the observed data history to a probability distribution over the action space at each time. Specifically, let 
$H_k$ 
denote the historical data collected prior to the assignment of $A_k$, i.e., $H_k=\{X_1,A_1,R_1,\cdots,X_k\}$.
Following $\pi$, the agent will set $A_k$ according to the probability mass function $\pi_k(\bullet|H_k)$ at each time $k$. In addition, we use $b=(b_1,\cdots,b_K)$ to denote the behavior policy that generates the data, which consists of $N$ i.i.d. realizations of \eqref{eqn:datatra}, denoted by
\begin{eqnarray}\label{eqn:data}
   \left\{(X_{1,i},A_{1,i},R_{1,i},X_{2,i},A_{2,i},R_{2,i},\dots,X_{K,i},A_{K,i},R_{K,i}):1\le i\le N\right\}.
\end{eqnarray}
Usually, the final outcome can be defined as a function of rewards at all stages, i.e. $f(R_1,\dots,R_K)$. However, in most cases when we periodically evaluate a specific kind of reward of interest, it is natural to infer the cumulative reward $\sum_{k=1}^K R_k$ following target policy $\pi$ from \eqref{eqn:data}. Note that in the following sections, although we focus on the estimation and inference of $\sum_{k=1}^K R_k$ at the $\tau$th quantile, the whole framework can be easily extended to any reward function $f(R_1,\dots,R_K)$. 

Similar to Section \ref{sec:1stage_est}, we assume versions of (C1$'$) consistency, (C2$'$) no unmeasured confounders, which is better known as the sequential randomization assumption, and (C3$'$) positive assumptions hold in multiple stage settings. An explicit version of these assumptions are provided in Section \ref{appendix:Assumption} of the supplementary material.\citep[see e.g.,][for the detailed definitions of these assumptions]{zhang2013robust}.

\subsection{Quantile Off-Policy Evaluation}

In $K$-stage settings, our objective is to infer the $\tau$th quantile of the sum of potential outcomes at all stages under fixed target policy sequence ${\pi}$. Define $R^*_k(\pi)$ as the potential outcome one would observe at stage $k$ by executing policy $\pi$, and define $R^*(\pi)=\sum_{k}R^*_k(\pi)$ as the cumulative reward of interest. Under the potential outcome framework, 
\begin{eqnarray}
   \label{eq:multi_indentify}
    \mathbb{P}\bigg(\sum_{k=1}^K R_k^*(\pi)\le r\Big|X_1\bigg)=\sum_{\{a_k\}_{k=1}^K}\Bigg\{\prod_{k=1}^K\pi_k(a_k|H_k) \cdot\mathbb{P}\bigg(\sum_{k=1}^K R_k^*(a_k)\le r\Big|X_1\bigg)\Bigg\},
\end{eqnarray}
where the summation over $\{a_k\}_{k=1}^K$ denotes the traverse of all actions such that $\{a_1,\dots,a_K\}\in |\mathcal{A}|^K$.
By assuming (C1$'$)-(C3$'$), formula (\ref{eq:multi_indentify}) can be successfully estimated from the observed data. For any given quantile $\tau$, we solve the optimization problem
\begin{equation} \label{eq:6}
    \eta_\tau=\arg\min_{\eta}\mathbb{E}_{X_1\sim \mathbb{G}}\left[\rho_{\tau}\bigg(\sum_{k=1}^KR_k-\eta\bigg)\Big| {A}\sim {\pi}\right],
\end{equation}
where $\mathbb{E}_{X_1\sim \mathbb{G}}[\bullet| {A}\sim {\pi}]$ denotes that the history data $(H_K,A_K)$, starting from the baseline information $X_1\sim\mathbb{G}$, follows treatment sequence ${\pi}=\{\pi_1,\dots,\pi_K\}$ at all stages up to $K$. In the following sections, we will first introduce the direct method and inverse probability weighting estimator, and then combine the above two to derive the doubly robust estimator.

\textbf{1. DM estimator.} Under (C1$'$)-(C3$'$), it follows from  (\ref{eq:multi_indentify}) that 
\begin{eqnarray}\label{eqn:quantile3}
    \widehat{\eta}_{\text{DM}}=\argmin_{\eta} \sum_{\{a\}_{k=1}^K}\Bigg\{ \widehat{\mathbb{E}}\bigg[\prod_{k=1}^K\pi_k(a_k|H_k)\cdot  \widehat{\mathbb{E}}\Big\{ \rho_{\tau}\Big(\sum_{k=1}^K R_k-\eta\Big)\Big|\{A_k\}_{k=1}^m=\{a_k\}_{k=1}^K,X_1\Big\}\bigg]\Bigg\}.
\end{eqnarray}
Similar to the single-stage setting, we first estimate the conditional probability function for $\sum_{k=1}^K R_k$ given $(H_K,A_K)$ and then leverage the MC method to calculate the conditional expectation in (\ref{eqn:quantile3}). Details will be provided in Section \ref{sec:2stage_est_details}.

\textbf{2. IPW estimator.} According to the the change of measure theorem,
\begin{eqnarray*}\label{eqn:quantile5}
    \widehat{\eta}_{\text{IPW}}=\argmin_{\eta} \widehat{\mathbb{E}} \left[\prod_{k=1}^K\frac{\pi(A_k|H_k)}{\widehat{b}_k(A_k|H_k)}\cdot \rho_{\tau}\Big(\sum_{k=1}^K R_k-\eta\Big)\right].
\end{eqnarray*}
Notice that when the number of stages increases, the proportion of data we observe that aligns with target policy will decrease drastically. This may lead to a poor IPW estimator. Akin to the single-stage setting, one can implement any supervised learning algorithm to learn $\widehat{b}_k$ that satisfies some mild convergence rate assumptions specified in Section \ref{sec:theory}. 

\textbf{3. DR estimator.} By incorporating the DM and IPW estimators, one can derive a doubly robust estimator for $K$-stage quantile off-policy evaluation.
\begin{eqnarray}\label{eq:8}
    \begin{aligned}
        \widehat{\eta}_{\text{DR}}=\argmin_{\eta} \frac{1}{N}\sum_{i=1}^N& \Psi(\eta;W_i,\widehat{b},\widehat{\mathbb{E}})\\
        \equiv \arg\min_{\eta} \frac{1}{N}\sum_{s=1}^S&\sum_{i\in \mathcal{I}_s}  \left[ \frac{\prod_{k'=1}^K \pi_{k'}(A_{k',i}|H_{k',i})}{\prod_{k'=1}^K \widehat{b}_{k'}(A_{k',i}|H_{k',i})}\rho_{\tau}\Big(\sum_{k=1}^KR_{i,k}-\eta\Big)\right.\\
         &\left. +\sum_{k=1}^{K}\left(\prod_{k'=1}^{k-1}\frac{\pi_{k'}(A_{k',i}|H_{k',i})}{\widehat{b}_{k'}(A_{k',i}|H_{k',i})}\right)\left(1-\frac{\pi_{k}(A_{k,i}|H_{k,i})}{\widehat{b}_{k}(A_{k,i}|H_{k,i})}\right)\widehat{L}_k(H_{k,i})\right],
    \end{aligned}
\end{eqnarray}
where $W_i$ is a shorthand for the data trajectory 
\begin{eqnarray}
   \left\{(X_{1,i},A_{1,i},R_{1,i},X_{2,i},A_{2,i},R_{2,i},\dots,X_{K,i},A_{K,i},R_{K,i}):1\le i\le N\right\},
\end{eqnarray}
$\widehat{L}_k(H_{k})=\widehat{\mathbb{E}}\left[\rho_{\tau}(R_{1}+\dots+{R}_{k-1}+\widehat{R}_{k}+\dots+\widehat{R}_{K}-\eta)|\{A_{k'}\}_{k'=1}^{k-1},\{A_{k'}\}_{k'=k}^K\sim{\pi},X_1\right]$,
in which $\widehat{R}_{k},\dots, \widehat{R}_{K}$ are the estimated rewards obtained by the data generating models following treatment $(\pi_k,\dots,\pi_K)$ starting from stage $k$. 

Formula (\ref{eq:8}) is the average of the sum of $K+1$ terms. The first term serves as an IPW estimator, where the coefficient $\frac{\prod_{k'=1}^K \pi_{k'}}{\prod_{k'=1}^K \hat{b}_{k'}}$ is the probability of a subject's actions at all $K$ stages being consistent with policy $\pi$. The last $K$ terms, shown in the second line, denote the augmentation terms estimating the quantile loss originating from the $k$th level of missing data. The $k$th augmentation term is the multiplication of two main components: the product of the two policy ratios in parenthesis, and a direct estimator $\widehat{L}_k(H_{k})$.  It aims to leverage subjects whose observed actions are consistent with $\pi$ during the first $k-1$ stages, but inconsistent starting from the $k$th stage. The product of the two policy ratios in this case roughly describes the probability of this inconsistency occurring. $\widehat{L}_k(H_{k})$ is the expected quantile loss estimated by combining the observed reward before stage $k$ with the estimated rewards from stage $k$ onwards. Using this DR estimator, one can maximize the utilization of the observed information relating to policy $\pi$. The estimation details are described in Section \ref{sec:2stage_est_details}.

The proposed estimator $\widehat{\eta}_{\text{DR}}$ is doubly robust in the sense that if the propensity score functions $\{\widehat{b}_k(A_{k}|H_k)\}_{k=1}^K$ or the outcome regression models $\{\widehat{L}_k(H_k)\}_{k=1}^K$ are consistently estimated at each stage, then $\widehat{\eta}_{\text{DR}}$ is a consistent estimator of $\eta_\tau$. Compared with the DM and IPW estimators, the DR estimator $\widehat{\eta}_{\text{DR}}$ tends to be more robust w.r.t. model misspecification. This also enables us to implement more flexible machine learning methods and leverage their advantages in estimating propensity score and outcome models. Details are proved in Section \ref{sec:theory}.

\subsection{Estimation details}\label{sec:2stage_est_details}
In this subsection, we focus on the estimation procedure under general $K$-stage settings. Similar to single stage, our approach contains four steps in sequential decision making process: sample splitting, nuisance functions estimation, doubly-robust quantile value and variance estimation, tail-robust mean value estimation. The pseudo code for multi-stage quantile estimation is summarized in Algorithm \ref{alg:pseudocode_multistage}.
\begin{algorithm}[tbh]
    \caption{Pseudo code of quantile off policy evaluation in multiple stage settings}
    \label{alg:pseudocode_multistage}
    \vspace{0.08in}
    \textbf{Input:} Data trajectories $   \left\{(X_{1,i},A_{1,i},R_{1,i},X_{2,i},A_{2,i},R_{2,i},\dots,X_{K,i},A_{K,i},R_{K,i}):1\le i\le N\right\}$; quantile levels $\tau$; number of folds $S$; a large integer $M$, the maximum \# iterations $\max_{iter}$. \\
    \textbf{Output:} doubly robust quantile estimator $\widehat{\eta}_{\tau}^{\text{DR}}$; tail-robust DR mean estimator $\hat\mu_{\text{DR}}$.
\begin{algorithmic}[1]
\Procedure{A. Doubly Robust Quantile Estimator}{}
\State Randomly split the data into $S$ folds. For any $s\in S$, denotes $\mathcal{I}_s$ as the $s$th subgroup.
    \ForAll{$s\in S$}
        \State Estimate $\{\widehat{b}_k\}_{k=1}^K$ by GBDT with the data in $\mathcal{I}_s^c$;
        \State Estimate $\{\widehat{R}_k|(H_k,A_k)\}_{k=1}^K$ by MDN with the data in $\mathcal{I}_s^c$;
        \State Generate $\{\widehat{R}_{h,a}^j|(H_k=H_{k,i},A_k=a)\}_{j=1}^M$ for any $a\in \mathcal{A}$ and $i\in \mathcal{I}_s$;
        \EndFor
        \While{$iter<\max_{iter}$ and $|(\eta_{iter+1}-\eta_{iter})/\eta_{iter}|>\epsilon$}
            \State $iter \leftarrow iter+1$
            \State Update $\eta_{iter}$ by solving formula (\ref{eq:8}) with gradient descent;
        \EndWhile
    \State $\widehat{\eta}_{\text{DR}}\leftarrow \eta_{iter}$
    \EndProcedure
    \Procedure{B. Tail-Robust DR Mean Estimator}{}
    \State For all of the quantile levels of our interest, calculate $\widehat{\eta}_{\tau}$ by procedure A. Averaging among all quantiles yields the tail-robust DR mean Estimator $\widehat{\mu}_{\text{DR}}$.
    \EndProcedure
\end{algorithmic}
\end{algorithm}
\subsubsection{Step 1: Sample Splitting}
Just as we do in the single stage setting, we denote $S$ as the number of splits we used for cross-fitting. In step 1, we randomly split the data into $S$ subgroups with equal sample size. For any $s\in \{1,\dots,S\}$,  denote $\mathcal{I}_s$ as the $s$th fold of the data, and $\mathcal{I}_s^c$ as its complement. Later on, we'll introduce how to leverage the data in $\mathcal{I}_s^c$ to conduct nuisance functions estimation, and use the data in $\{\mathcal{I}_s\}_{s=1}^S$ to obtain the quantile estimator $\widehat{\eta}_{\tau}^{\text{DR}}$.

\subsubsection{Step 2: Nuisance Functions Estimation}
Step 2 of this approach aims to estimate the nuisance functions involved in formula (\ref{eq:8}). Unlike the single-stage setting, the estimation involves two different model functions: propensity score functions $\{b_{k}(H_k)\}_{k=1}^K$, and the conditional expectation functions $\{\widehat{L}_k(H_k)\}_{k=1}^K$.

The estimation of $\{b_{k}(H_k)\}_{k=1}^K$ is essentially the same as that in single stage. As it is a supervised learning problem, one can use any state-of-the-art method that fits the best for the categorical action space. For illustration purpose, we use GBDT in simulations and real data studies.

The estimation of $\{\widehat{L}_k(H_k)\}_{k=1}^K$ needs more attention. Firstly, the conditional distribution functions for $\{\widehat{R}_k|(H_k,A_k)\}_{k=1}^K$ can be obtained by MDN. After getting these reward-generators at each stage, $\widehat{L}_k(H_k)$ can be estimated through a Monte Carlo approximation. As long as the number of replicates is large enough, $\widehat{L}_k(H_k)$ will converge to the true value if the model is correctly specified. Specifically,
\begin{equation}
\begin{aligned}
    \widehat{L}_k(H_{k})&=\widehat{\mathbb{E}}\left[\rho_{\tau}(R_1+\dots+{R}_{k-1}+\widehat{R}_{k}+\dots+\widehat{R}_K-\eta_{\tau}^{\pi})|\{A_{k'}\}_{k'=1}^{k-1},\{A_{k'}\}_{k'=k}^K\sim{\pi},X_1\right]\\
    &=\frac{1}{{M}}\sum_{j=1}^{M} \rho_{\tau}(R_1+\dots+{R}_{k-1}+\widehat{R}_{k}^{(j)}+\dots+\widehat{R}_K^{(j)}-\eta_{\tau}^{\pi}),
\end{aligned}
\end{equation}
where $\{\widehat{R}_{k}^{(j)},\dots,\widehat{R}_{K}^{(j)}\}_{j=1}^M$ are obtained from the reward generator for $\{\widehat{R}_{k'}|(H_{k'},A_{k'})\}_{{k'}=k}^K$ estimated by MDN, and $\{A_{k'}\}_{{k'}=k}$ in this case follows target policy $\pi$. Note that the whole procedure for estimating $\widehat{L}_k(H_{k})$ only involves the information in the observed $H_k$. The level of information we utilize in estimating $R^*(\pi)$ depends on how well the observed actions are in line with the target policy, which is the core idea of double robustness.

\subsubsection{Step 3: Doubly-Robust Quantile Value and Variance Estimation}\label{sec:multi-J0}
Step 3 consists of two parts: utilizing cross-fitting to estimate the DR quantile estimator in multiple stages, and constructing a doubly robust variance estimator for the estimated quantile to obtain a stable confidence interval.

First, define $\widehat{b}_s$ and $\widehat{\mathbb{E}}_s$ as the estimated behavior policy and conditional mean estimator obtained from the data in $\mathcal{I}_s^c$. The final doubly robust quantile estimator can be obtained by summing the objective functions we obtain under the $s$th fold, i.e.
\begin{eqnarray*}
    \widehat{\eta}_{\tau}^{\textrm{DR}}=\argmin_{\eta}\sum_{s=1}^S\sum_{i\in \mathcal{I}_s} \Psi(\eta;W_i,\widehat{b}_s,\widehat{\mathbb{E}}_s).
\end{eqnarray*}

Second, we illustrate a doubly robust estimation procedure for the variance of $\widehat{\eta}_\tau^{\text{DR}}$. Similar to the single-stage setting, the variance of $\widehat{\eta}_\tau^{\text{DR}}$ is obtained by the sandwich formula, i.e. $\sigma^2=J_0^{-2}\textrm{Var}(\partial_\eta\Psi(\eta_{\tau}; W_i, b,\mathbb{E}))$, where $J_0={f}_{{R}^{*}({\pi})}(\eta_\tau)$ and ${R}^{*}({\pi})=\sum_{k=1}^K{R}_k^{*}({\pi})$ denotes the cumulative true reward by executing policy $\pi$ at each stage. As shown earlier in single stage settings, since $\textrm{Var}(\partial_\eta\Psi(\eta_{\tau}; W_i, b,\mathbb{E}))$ can be estimated by the sampling variance, we now need to estimate $J_0={f}_{{R}^{*}({\pi})}(\eta_\tau)$ in a doubly robust way. 

Define ${f}_{R^*(\pi),s}$ as the conditional pdf of $R^*(\pi)|H_{1}$ obtained by MDN from the data in $\mathcal{I}_s^c$. The direct estimator for $J_0$ is
\[
\widehat{J}_0^{\text{DM}}=\widehat{f}_{R^*(\pi)}(\eta)=\frac{1}{N}\sum_{s=1}^S\sum_{i\in \mathcal{I}_s} \widehat{f}_{R^*(\pi),s}(\widehat{\eta}_{\textrm{DR}}| H_{1,i}).
\]
The IPW estimator of $J_0$, by applying the change of measure theorem, is given by
\[
\begin{aligned}
\widehat{J}_0^{\text{IPW}}=\frac{1}{N} \sum_{s=1}^S \sum_{i\in \mathcal{I}_s} \prod_{k=1}^K\frac{\pi_k(A_{k,i}|H_{k,i})}{\widehat{b}_k(A_{k,i}|H_{k,i})}\frac{1}{h} K\left(\frac{\sum_{k=1}^KR_{i,k}-\widehat{\eta}_{\textrm{DR}}}{h}\right)
\end{aligned}
\]

Thus, the doubly robust estimator for ${J}_0$ is similarly derived as
\begin{eqnarray}\label{eq:J0_2stages}
    \begin{aligned}
    \widehat{J}_0^{\text{DR}}=\arg\min_{\eta} \frac{1}{N}&\sum_{s=1}^S\sum_{i\in \mathcal{I}_s}  \left[ \frac{\prod_{k'=1}^K \pi_{k'}(A_{k',i}|H_{k',i})}{\prod_{k'=1}^K \widehat{b}_{k'}(A_{k',i}|H_{k',i})}\frac{1}{h} K\left(\frac{\sum_{k=1}^KR_{i,k}-\widehat{\eta}_{\text{DR}}}{h}\right)\right.\\
     &\left. +\sum_{k=1}^{K}\left(\prod_{k'=1}^{k-1}\frac{\pi_{k'}(A_{k',i}|H_{k',i})}{\widehat{b}_{k'}(A_{k',i}|H_{k',i})}\right)\left(1-\frac{\pi_{k}(A_{k,i}|H_{k,i})}{\widehat{b}_{k}(A_{k,i}|H_{k,i})}\right)\widehat{f}_{\bar{R}_k^{*}({\pi})}(\widehat{\eta}_\text{DR}|H_{k,i})\right]
\end{aligned}
\end{eqnarray}
where $\widehat{f}_{\bar{R}_k^{*}({\pi})}(\bullet|H_{k,i})$ is the estimated conditional pdf of $\bar{R}_k^{*}({\pi}):=\sum_{k'=k}^K R_{k'}^*(\pi)$ given $H_k$ and $A_{k}\sim\pi$ obtained by MDN.

One simple way to estimate the conditional pdf for $\bar{R}_k^{*}({\pi})|(H_k,A_k)$ is to use the data that aligns with $\{\pi_{k'}\}_{k'=k}^K$. In this case, an MDN can be fitted by regarding $\bar{R}_k^{*}({\pi})$ as a function of $(H_k,A_k)$. After getting the estimated pdf for $\bar{R}_k^{*}({\pi})|(H_k,A_k)$, we can let $A_k=\pi_k(H_k)$ to obtain $\widehat{f}_{\bar{R}_k^{*}({\pi})}(\bullet|H_{k,i})$. 

A Wald-type confidence interval can thus be constructed based on our DR quantile estimator $ \widehat{\eta}_{\tau}^{\textrm{DR}}$ and its DR variance estimator obtained from formula (\ref{eq:J0_2stages}).

\subsubsection{Step 4: Tail-Robust Mean Value Estimation}
Step 4 focuses on estimating the mean reward from quantile estimators $\widehat{\eta}_\tau^{\text{DR}}$ evaluated at multiple levels of $\tau$. An easy way of doing this is by choosing equally spaced quantile levels on $[0,1]$, and averaging the quantile DR estimators at each level to obtain $\widehat{\mu}_{\text{DR}}$. Simulation results will illustrate the effectiveness of our method in decreasing the MSE of the mean reward in heavy-tailed distributions.

\section{Theory}
\label{sec:theory}
In this section, we will provide asymptotic guarantees for our doubly robust quantile estimator $\widehat{\eta}_\tau^{\text{DR}}$.

Define $\psi(W_i;\eta,\hat{\alpha})\equiv\partial_{\eta}\Psi(W_i;\eta,\widehat{b},\widehat{\mathbb{E}})$, which is the derivative of the objective function we are trying to optimize. Similarly, we define $\psi^*(W_i;\eta)=\partial_{\eta}\Psi(W_i;\eta,{b},{\mathbb{E}})$ where the nuisance functions for $\widehat{b}$ and $\widehat{\mathbb{E}}$ are replaced by their population functions\footnote{The explicit expressions for $\psi(W_i;\eta,\hat{\alpha})$ and $\psi^*(W_i;\eta)$ are given in Section \ref{sup:1} of the supplementary material.}. 

Solving the optimization problem in  (\ref{eq:8}) is thus equivalent to obtaining the solution $\widehat{\eta}_\tau^{\text{DR}}$ to the following estimating equation:
\begin{equation}
    \frac{1}{N}\sum_{s=1}^S\sum_{i\in \mathcal{I}_s}\psi(W_i;\eta,\hat{\alpha}_s)=0,
\end{equation}
where $\hat{\alpha}_s$ is the collection of nuisance parameters in estimating $\widehat{b}$ and $\widehat{\mathbb{E}}$ using data in $\mathcal{I}_s^c$.

\subsection{The asymptotic properties of $\widehat{\eta}_{\tau}^{\text{DR}}$}
\begin{lemma}\label{lemma:2}
$\widehat{\eta}_\tau^{\text{DR}}$ is a consistent estimator of $\eta_{\tau}$, as long as one of the following two parts of models is consistently estimated: 
\begin{enumerate}
    \item[(1)] The propensity score functions at each stage, i.e. $\{\widehat{b}_k(H_{k})\}_{k=1}^K$.
    \item[(2)] The conditional expectation functions at each stage, i.e. $\{\widehat{L}_k(H_{k})\}_{k=1}^K$,\\
    or equivalently the data generating models for $\{\widehat{R}_k|(H_k,A_k)\}_{k=1}^K$.
\end{enumerate}
\end{lemma}
The proof of Lemma \ref{lemma:2} along with other auxiliary theoretical results are provided in Section \ref{sup:theory} of the supplementary material. Next, we will derive the asymptotic properties of our DR quantile estimator $\widehat{\eta}_\tau^{\text{DR}}$.

Define $\left\|\cdot\right\|_{P,q}$ as the $L^q(P)$ norm which satisfies $\left\|f(W)\right\|_{P,q}=\left(\int |f(\omega)|^q dP(\omega)\right)^{1/q}$. Let $\delta(F_1,F_2)$ denote the total variation distance between two probability measures, where $F_1$ and $F_2$ are the CDF of the two distributions. Define ${f}_{\bar{R}^{*}_k({\pi})|H_k}(r|H_k)$ and ${F}_{\bar{R}^{*}_k({\pi})|H_k}(r|H_k)$ as the true conditional pdf and cdf of $\bar{R}^{*}_k({\pi})=\sum_{k'=k}^K R_{k'}^*(\pi)$ given the historical data $H_k$, respectively. Similarly, define $\widehat{f}_{{\bar{R}^{*}_k}({\pi})|H_k}(r|H_k)$ and $\widehat{F}_{{\bar{R}}^{*}_k({\pi})|H_k}(r|H_k)$ as the conditional pdf and cdf of $\bar{R}^{*}_k({\pi})$ estimated by MDN, respectively.

Now let's first introduce the assumptions needed in our main theorem.
\vspace{0.1in}\\
\textbf{Assumptions}:
\begin{description}
\item[(A1)] $\left\|\widehat{b}_k(A_k|H_k)-{b}_k(A_k|H_k)\right\|_{P,2}=o(n^{-1/4})$, $\forall$ $k\in\{1,\dots,K\}$.
\item[(A2)] $\left\|\delta\left(\widehat{F}_{\bar{R}_k^{*}({{\pi}})|H_1},F_{\bar{R}_k^{*}({{\pi}})|H_1}\right)\right\|_{P,2}=o(n^{-1/4})$, $\forall$ $k\in\{1,\dots,K\}$. 
\item[(A3)] $f_{\bar{R}_k^{*}({\pi})|H_k}(r|H_k)$ is uniformly bounded in $r$ and $H_k$, $\forall$ $k\in\{1,\dots,K\}$.
\item[(A4)] $\exists$ a constant $C_1>0$, such that ${f}_{{R}^{*}({\pi})}(\eta)\geq C_1$ holds for all $\eta$ in a neighbor of $\eta_\tau$.
\item[(A5)] $\exists$ $\epsilon>0$, s.t. $\mathbb{P}(\epsilon\leq \widehat{b}_k(A_k|H_k)\leq 1-\epsilon)=1$, $\forall$ $H_k$ and $k\in\{1,\dots,K\}$.
\item[(A6)] $\partial_{r}{f}_{R^{*}({\pi})}(r)\big|_{r=\eta_\tau}$ is bounded.
\end{description}

\begin{remark}
Assumptions (A1) and (A2) measure the accuracy in estimating propensity score and outcome at each stage. These are mild conditions that can be achieved by a lot of machine-learning-based methods. In Assumption (A3), we require the conditional pdfs to be uniformly bounded, which aims to guarantee the continuity of the cdf of the corresponding reward functions. Assumption (A4) ensures the identifiability of the quantile $\eta_\tau$ of our interest, and Assumption (A5) is an adjunctive condition on the estimated propensity score, which is commonly assumed in related literature \citep{chernozhukov2018double}. Assumption (A6) requires the derivative of $f_{R^*(\pi)}(r)$ to be bounded only at the true value $\eta_\tau$. Since the true value is unknown to us, we may need a stronger condition that requires uniform boundedness on the support of $\eta$.
\end{remark}

\begin{theorem} \label{thm:1}
    Suppose that Assumptions (A1)-(A6) are satisfied. When $M\rightarrow \infty$, $\widehat{\eta}^{\text{DR}}_\tau$ is asymptotically normal. Specifically,
    \begin{equation}
        \sigma^{-1}\sqrt{N}(\widehat{\eta}^{\text{DR}}_\tau-\eta_\tau)=-\sigma^{-1}J_0^{-1}\left(\frac{1}{\sqrt{N}}\sum_{i=1}^N\psi^*(W_i;\eta_\tau)\right)+o_p(1)\stackrel{\mathcal{D}}{\longrightarrow} \mathcal{N}(0,1),
    \end{equation}
    where 
    \begin{equation}\label{eq:28.5}
        \sigma^2=J_0^{-1}\mathbb{E}[\psi^{*2}(W;\eta_{\tau})](J_0^{-1})',\quad \text{and}\quad J_0=\partial_{\eta}\{\mathbb{E}_W[\psi^*(W;\eta)]\}|_{\eta=\eta_\tau}.
    \end{equation}
Furthermore,
\begin{equation}\label{eq:38}
    J_0=\partial_{\eta}\{\mathbb{E}_W[\psi^*(W;\eta)]\}|_{\eta=\eta_\tau}=\mathbb{E}_{H_1\sim \mathbb{G}}\left[{f}_{{R}^{*}({\pi})|H_1}(\eta_\tau|H_1)\right]={f}_{{R}^{*}({\pi})}(\eta_\tau).
\end{equation}
    The results still hold when $\sigma^2$ is replaced by its doubly robust estimator $\widehat{\sigma}^2_{\text{DR}}$, given by
    \begin{equation}\label{eq:29}
        \widehat{\sigma}^2_{\text{DR}}=\left(\widehat{J}_0^{\text{DR}}\right)^{-1} \frac{1}{S}\sum_{s=1}^S \widehat{\mathbb{E}}_{n,s}[\psi^2(W;\widehat{\eta}_s,\hat{\alpha}_{s})]\left(\widehat{J}_0^{\text{DR}}\right)^{-1},
    \end{equation}
    where $\mathbb{E}_{n,s}$ denotes the empirical expectation calculated with the data in fold $s$, and $\widehat{J}_0^{\text{DR}}$ is estimated under the doubly robust framework, as we've discussed in formula (\ref{eq:J0_1stage}) and (\ref{eq:J0_2stages}).
\end{theorem}
 Estimation details are provided in Section \ref{sec:single-J0} and Section \ref{sec:multi-J0} for both single stage and multiple stage settings.
\begin{corollary}
According to Theorem \ref{thm:1}, a Wald-type $\alpha$-level confidence interval of $\eta_\tau$ can be constructed by
\begin{equation}\label{eq:CI}
    \text{CI}=\left[\widehat{\eta}^{\text{DR}}_\tau\pm \frac{1}{\sqrt N}\Phi^{-1}(1-{\alpha}/{2})\cdot\widehat{\sigma}_{\text{DR}}\right],
\end{equation}
where $\Phi$ is the cdf of the standard normal distribution. The confidence interval constructed in (\ref{eq:CI}) achieves an asymptotic coverage rate $1-\alpha$.
\end{corollary}

The proof of Theorem 1 is summarized in Section \ref{sup:4} of the supplementary material. 

\section{Numerical Results}\label{sec:numerical}
In this section, we will justify the performance of our DR quantile estimator in both single-stage and multi-stage settings under different levels of heavy-tailness of the reward distribution. Throughout this paper, we use student-t distributions with different degrees of freedom to measure the level of heavy tail.

In the following sections, we will first illustrate how close our estimated quantile is to the true reward distribution, and then report the empirical coverage probability of our estimator. To show the performance of quantile aggregation in estimating the mean outcome in the heavy-tailed cases, we compare our quantile-based DR mean estimator \textit{Rquantile}, with the DR mean estimator \textit{Rmean} used in common literature\footnote{The expression of the DR mean estimator is summarized in Section \ref{sup:1} of the supplementary material.}.

Before we proceed, let's discuss how to choose the number of folds in cross fitting. Although a larger value of $S$ may yield an intuitively better performance in quantile estimation, according to some comparison results in existing work \citep{chernozhukov2018double}, there is no strong relationship between the number of folds and estimation accuracy in the context of double machine learning. To balance the computational complexity and the potential advantage of a large number of $S$, we set $S=5$ throughout the simulation and real data analysis. For the ease of calculation, we first compute the DR quantile estimator under each fold $s$ and average over all folds to obtain our final estimator.

\subsection{Quantile Estimation Performance}
Figure \ref{fig:3} shows the estimation accuracy of our DR quantile estimator. As we can see, our DR quantile estimator performs quite well in evaluating the true reward distributions under the target policy of our interest.

\begin{figure}[htbp]
\centering
\begin{minipage}[t]{0.45\textwidth}
\centering
\includegraphics[width=7cm]{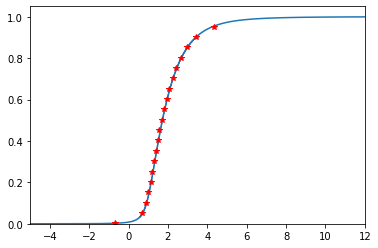}
\end{minipage}
\begin{minipage}[t]{0.45\textwidth}
\centering
\includegraphics[width=7cm]{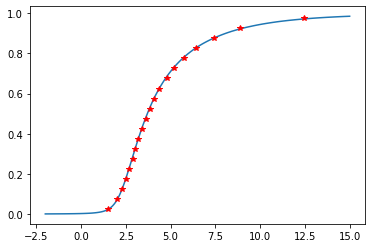}
\end{minipage}
\caption{\small Quantile estimation result in single-stage (left) and two-stage (right) settings. The blue curve denotes the true reward distribution (the CDF under target policy $\pi$), and red stars denote the estimated quantiles at $20$ equally spaced quantile levels. Details about the data generating processes are introduced in Section \ref{appendix:DGP} of the supplementary material.}\label{fig:3}
\end{figure}

\subsection{Coverage Probability}
According to the estimation details elaborated in Section \ref{sec:1stage_est_details} and \ref{sec:2stage_est_details}, we calculate the quantile $\widehat{\eta}^{\text{DR}}_\tau$ and its doubly robust variance estimator $\widehat{\sigma}^2_{\text{DR}}$ by formula (\ref{eq:29}). For each quantile level $\tau$, we repeat the estimation procedure $500$ times to calculate the empirical coverage probability, which is shown in Figure \ref{fig:5}. 

Notice that when calculating the asymptotic variance, it is required to select a proper bandwidth in the kernel function. After comparing several commonly used bandwidth selection methods, we finally choose fixed bandwidth under this heavy-tailed circumstance due to its robustness in preventing the over-smoothing issue in kernel density estimation.\footnote{See Section \ref{appendix:bandwidth} in the supplementary material for details about bandwidth selection.}

\begin{figure}[htbp]
\centering
\begin{minipage}[t]{0.32\textwidth}
\centering
\includegraphics[width=5.4cm]{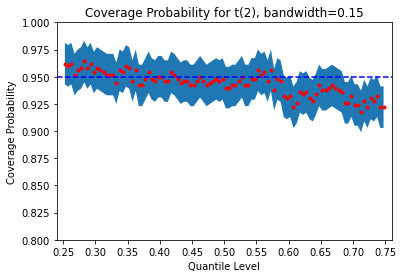}
\end{minipage}
\begin{minipage}[t]{0.32\textwidth}
\centering

\includegraphics[width=5.4cm]{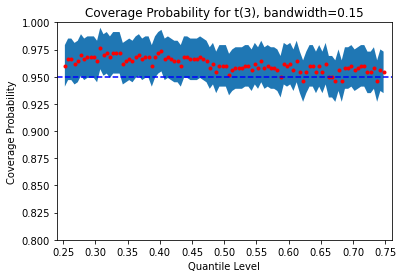}
\end{minipage}
\begin{minipage}[t]{0.32\textwidth}
\centering

\includegraphics[width=5.4cm]{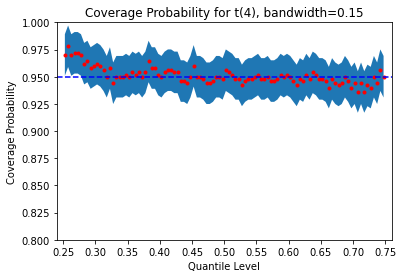}
\end{minipage}
\caption{\small Coverage Probability under noise level t(2), t(3), and t(4). The red dots correspond to the empirical coverage probability estimated from $500$ times of replicates, and the blue area denotes the confidence band at each quantile level. It is clear from the plot that $0.95$ falls into almost all of the confidence bands, which means that our CI achieves nominal coverage in most cases.}\label{fig:5}
\end{figure}

\subsection{Comparison between \textit{Rquantile} and \textit{Rmean}}

To compare our quantile-based mean DR estimator \textit{Rquantile} with the classical DR mean estimator \textit{Rmean}, we calculated the MSE for both estimators under different heavy-tailed levels of the reward distributions. Detailed comparisons are summarized in Table \ref{table:1} and Table \ref{table:2}. 
As we can see from the Mean Square Error (MSE), \textit{Rquantile} always yields smaller MSE than \textit{Rmean}. The more heavy-tailed the reward distribution is, the more powerful our method tends to be.

\begin{table}[ht]
\centering
\footnotesize
\begin{tabular}{c|l|llllllll}
\hline
\multicolumn{2}{l|}{Heavy-tailed level} & t(1.2) & t(1.5) & t(1.8) & t(2) & t(2.5) & t(3) & t(3.5) & t(4) \\ \hline
\multirow{2}{*}{MSE} & Rquantile & 0.005995  & 0.001689 & 0.000898 & 0.000993 &  0.000545 & 0.000432 & 0.000312 & 0.000371     \\
& Rmean  & 0.594783  & 0.031051 & 0.002025 & 0.001863 & 0.000620 & 0.000445 & 0.000324 &  0.000381  \\ \hline
\end{tabular}
\caption{Single-stage comparison between \textit{Rquantile} and \textit{Rmean}. We use 5-folds cross-fitting with $N=2500$. All the results are obtained from $100$ times of replicates.}\label{table:1}
\end{table}

\begin{table}[ht]
\begin{center}
\begin{tabular}{c|l|lllll}
\hline
\multicolumn{2}{c|}{Level of heavy tail} 
& t(2) & t(4)  & t(6) & t(8)  & $\mathcal{N}(0,1)$    \\ \hline
\multirow{2}{*}{MSE} 
& Rquantile & 0.006708 & 0.002729  & 0.002447  & 0.002427  & 0.001549  \\
 & Rmean  & 0.027780 & 0.003945 &  0.003558  & 0.003710  & 0.002062\\ \hline
\end{tabular}    
\end{center}
\caption{Two-stage comparison between \textit{Rquantile} and \textit{Rmean}. We use 5-folds cross-fitting with $N=2500$. All the results are obtained from $100$ times of replicates.}\label{table:2}
\end{table}

To illustrate the advantage of our estimator more clearly, we visualize the estimation results of 100 times of replicates for both \textit{Rquantile} and \textit{Rmean} when the reward follows t(1.2) distribution, which is shown in Figure \ref{fig:6}. 

\begin{figure}[htbp]
\centering
\includegraphics[width=8cm]{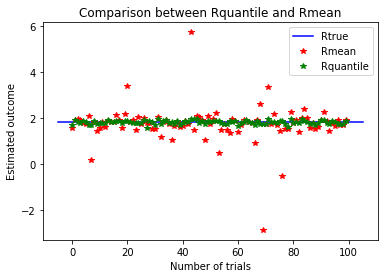}
\caption{\small Comparison when the noise follows t(1.2) distribution. The blue line denotes the true value of mean reward, red stars denote the estimated value of \textit{Rmean}, and green stars denote the estimated value of \textit{Rquantile}.
}\label{fig:6}
\end{figure}

It is obvious from Figure \ref{fig:6} that under this super heavy-tailed setting, \textit{Rquantile} outperforms \textit{Rmean} in estimating the true reward expectation. There are some ``outliers" when estimating \textit{Rmean}, which performs much worse than other points. This is because the estimated propensity scores of some subjects are quite close to $0$, leading to extremely unstable estimators for \textit{Rmean}. On the contrary, our approach divides the mean estimation procedure into many quantile levels, and solving each estimating equation won't affect as much as what we may have in obtaining \textit{Rmean}. This intuition explains the tail robustness of our estimator. Since each quantile we estimate also enjoys the double robustness property, it's reasonable that \textit{Rquantile} can achieve better results than \textit{Rmean}.

\section{Real Data Analysis} \label{sec:realdata}
In this section, we consider an advertisement experiment conducted at a world-leading tech company. The company plans to add an advertising position with three candidate choices. This experiment serves as an exploration applied to three groups of customers (180,000 users per group) randomly selected from daily active users. The algorithmic researchers conduct causal inference on this data, and come up with four different personalized intervention policies. The key response of interest is a newly defined metric (the larger, the better) related to user experience, which has a heavy-tailed distribution as shown by the histogram, Q-Q plot and log-log frequency plot in Figure \ref{fig:real_data_analysis_loglog}. For the log-log frequency plot, we divide all users into percentiles based on their metrics, and plot the log of frequency over the log of average absolute metrics within each percentile. 
\begin{figure}[tbh]
\includegraphics[width=\textwidth]{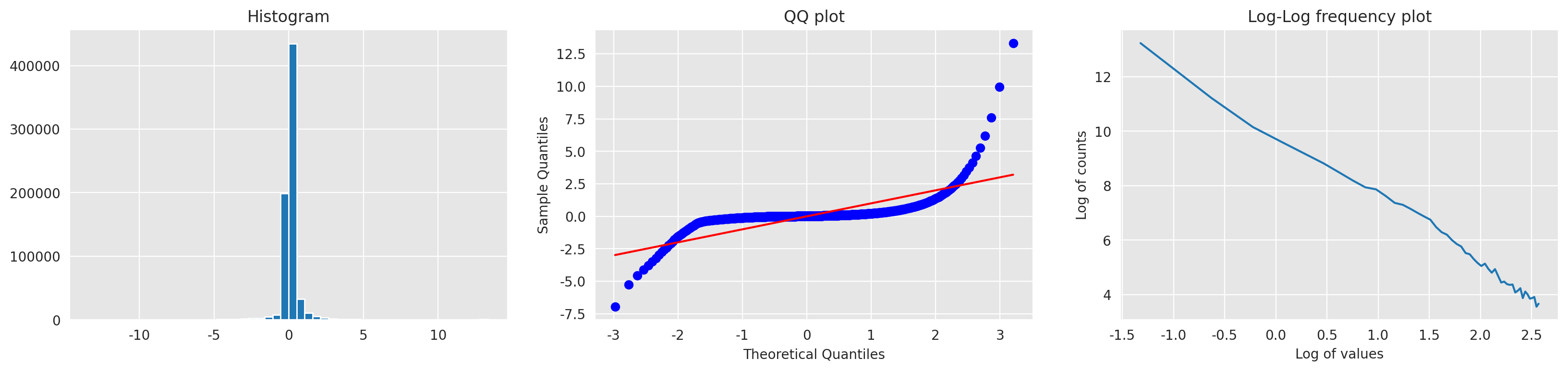}
\caption{Log-Log plot to illustrate the heavy tail of the reward distribution}
\label{fig:real_data_analysis_loglog}
\end{figure}
\begin{table}[tbh]
    \centering
    \begin{tabular}{l|c|c|c|c|c|c}
    \hline
    & \multicolumn{2}{c|}{bias} &\multicolumn{2}{c|}{std}&\multicolumn{2}{c}{RMSE}\\
    \hline
    & mean & quantile & mean & quantile & mean & quantile \\
    \hline
    $S_1$ & 0.006 & -0.013 & 0.076 & 0.069 & 0.076 & 0.071 \\
    \hline
    $S_2$ & -0.008 & -0.026 & 0.060 & 0.055 & 0.060 & 0.061 \\
    \hline
    $S_3$ & 0.007 & -0.014 & 0.066 & 0.059 & 0.067 & 0.061 \\
    \hline
    $S_4$ & 0.001 & -0.017 & 0.077 & 0.068 & 0.077 & 0.070 \\
    \hline
    \end{tabular}
    \caption{Comparison of bias, std and rmse between mean and quantile approaches. $S_1$, $S_2$, $S_3$, and $S_4$ denote four personalized policies.}
    \label{tab:real_data_analysis}
\end{table}
We wish to illustrate that our method is able to quantify the expected revenue of the four policies before any online A/B test, even with only a small sub-sample of the entire dataset. We repeat the following process 1000 times. Each time, we randomly select one percent of users from each group, and use our method to estimate 999 quantiles from 0.001 to 0.999 of the response distribution of the four policies. We then use the mean of the 999 quantiles as the estimated expected value. The true values of the four policies are obtained via another online A/B test. The mixture density network consists of two hidden layers of eight neurons, and four components.

The bias, standard deviation and root MSE are presented in Table \ref{tab:real_data_analysis}. Our doubly robust quantile-based mean estimator outperforms the traditional doubly robust mean estimator on this heavy-tailed real dataset.  Table \ref{tab:real_data_analysis_cp} shows that the coverage probabilities of the $25\%$, $50\%$ and $75\%$ quantiles of the reward distribution of the four personalized policies are all close to the $95\%$ nominal level.  Furthermore, we also compare the performance of DM, IPW and DR methods. Details are shown in Section \ref{sup:realdata} of the Supplmentary Materials. 

\begin{table}[tbh]
	\centering
	\begin{tabular}{c|c|c|c|c}
		\hline
		& $S_1$ & $S_2$ & $S_3$ & $S_4$ \\
		\hline
		$25\%$ & 0.926 & 0.944 & 0.933 & 0.911\\
		\hline
		$50\%$ & 0.961 & 0.958 & 0.957 & 0.948\\
		\hline
	    $75\%$ & 0.949 & 0.926 & 0.942 & 0.915\\
		\hline
	\end{tabular}
	\caption{Coverage probabilities of the first, second, third quartiles of the value distribution corresponding to the four personalized policies.}
	\label{tab:real_data_analysis_cp}
\end{table}


\section{Conclusion}
\label{sec:conc}

In this paper, we conducted comprehensive research on the doubly robust off policy evaluation procedure for the entire reward distribution in sequential decision making problems. We construct our estimating equation by combining the framework of doubly robust estimation with quantile regression, and provide an algorithm that allows a broad class of machine-learning-based tools to be utilized in estimating propensity scores and reward functions at each stage.

Moreover, the implementation of deep conditional generative learning models also enables us to handle the nuisance functions with the parameter of our interest. With theoretical guarantees, we can derive the asymptotic distribution of our doubly robust quantile estimator. Furthermore, a doubly robust variance estimator was proposed to improve the stability of our confidence interval. 

Based on our estimated DR quantiles, we proposed a tail-robust doubly-robust mean estimator which significantly outperforms the classical DR mean estimator used in the earlier literature. The simulation and real data results both illustrate the power of our estimator in decreasing the MSE when the reward is heavy-tailed.

\bibliographystyle{chicago}
\bibliography{Reference}


\clearpage
\pagenumbering{arabic}
\setcounter{page}{1}
\appendix
\begin{center}
{\Large\bf SUPPLEMENTARY MATERIAL }
\vspace{0.1in}
{\large \bf \\for ``Quantile Off-Policy Evaluation via Deep
Conditional Generative Learning"}
\end{center}

\section{Assumptions and Notations}\label{sup:1}
In this section, we aim to elaborate more on the explicit expressions of some assumptions and formulas mentioned in the main paper.
\subsection{Assumption (C1$'$)-(C3$'$)}\label{appendix:Assumption}
In multi-stage setting, we also accept the potential outcome framework in causal inference. 
\begin{enumerate}
    \item[(C1$'$)] Consistency (or SUTVA) \citep{rubin1990comment}: For any $k\in \{1,\dots,K\}$,
    \[
    R_k=R_k^*(\bar{A}_k)=\sum_{\bar{a}_k\in\bar{\mathcal{A}}_k}R^*(\bar{a}_k)\mathbb{I}\{\bar{A}_k=\bar{a}_k\};
    \]
    When $k\geq 2$,
    \[
    X_k=X_k^{*}(\bar{A}_{k-1})=
    \sum_{\bar{a}_{k-1}\in\bar{\mathcal{A}}_{k-1}}X_k^{*}(\bar{a}_{k-1})\mathbb{I}\{\bar{A}_{k-1}=\bar{a}_{k-1}\}.
    \] 
    \item[(C2$'$)] Sequential Randomization Assumption (SRA) \citep{rosenbaum1983central}: Given any treatment sequence $\bar{b}_K$, we have 
    \[A_k\perp\{R_k(\bar{b}_k),X_{k+1}(\bar{b}_k),R_{k+1}(\bar{b}_{k+1}),\dots,R_K(\bar{b}_K)\}|H_k,\quad \forall k\in \{1,\dots,K\}.\]
    \item[(C3$'$)] Positivity Assumption: For any $k\in \{1,\dots,K\}$, there exists a constant $\epsilon>0$, s.t. $\mathbb{P}(\epsilon\leq b_k(A_k|H_k)\leq 1-\epsilon)=1$ for any $A_k\in\{0,1\}$.
\end{enumerate}
Note that (C1$'$)-(C3$'$) are just natural extensions of (C1)-(C3) in multi-stage setting, which is also widely assumed in causal inference literature.

\subsection{$\psi(W_i;\eta,\hat{\alpha})$ and $\psi^*(W_i;\eta)$}
The estimating equation of our interest is defined as
\begin{eqnarray}
    \begin{aligned}       \psi(W_i;\eta,\hat{\alpha})\equiv&\partial_{\eta}\Psi(W_i;\eta,\widehat{b},\widehat{\mathbb{E}})\\
     =\frac{1}{N}\sum_{s=1}^S\sum_{i\in \mathcal{I}_s} & \left[ \frac{\prod_{k'=1}^K \pi_{k'}(A_{k',i}|H_{k',i})}{\prod_{k'=1}^K \widehat{b}_{k'}(A_{k',i}|H_{k',i})}\Big(\mathbb{I}\Big\{\sum_{k=1}^K R_{i,k}<\eta\Big\}-\tau\Big)\right.\\
         &\left. +\sum_{k=1}^{K}\left(\prod_{k'=1}^{k-1}\frac{\pi_{k'}(A_{k',i}|H_{k',i})}{\widehat{b}_{k'}(A_{k',i}|H_{k',i})}\right)\left(1-\frac{\pi_{k}(A_{k,i}|H_{k,i})}{\widehat{b}_{k}(A_{k,i}|H_{k,i})}\right)\widehat{l}_k(H_{k,i})\right],
    \end{aligned}
\end{eqnarray}
where $\widehat{l}_k(H_{k})=\widehat{\mathbb{E}}\left[\mathbb{I}\{R_1+\dots+{R}_{k-1}+\widehat{R}_{k}+\dots+\widehat{R}_K<\eta\}-\tau|\{A_{k'}\}_{k'=1}^{k-1},\{A_{k'}\}_{k'=k}^K\sim{\pi},X_1\right]$, in which $\widehat{R}_{k},\dots, \widehat{R}_{K}$ are the estimated rewards obtained by the data generating models following treatment $(\pi_k,\dots,\pi_K)$ starting from stage $k$. 

When the nuisance functions are replaced by the true models, we have
\begin{eqnarray}
    \begin{aligned}
        \psi^*(W_i;\eta)=\frac{1}{N}\sum_{s=1}^S\sum_{i\in \mathcal{I}_s} & \left[ \frac{\prod_{k'=1}^K \pi_{k'}(A_{k',i}|H_{k',i})}{\prod_{k'=1}^K {b}_{k'}(A_{k',i}|H_{k',i})}\Big(\mathbb{I}\Big\{\sum_{k=1}^K R_{i,k}<\eta\Big\}-\tau\Big)\right.\\
         &\left. +\sum_{k=1}^{K}\left(\prod_{k'=1}^{k-1}\frac{\pi_{k'}(A_{k',i}|H_{k',i})}{{b}_{k'}(A_{k',i}|H_{k',i})}\right)\left(1-\frac{\pi_{k}(A_{k,i}|H_{k,i})}{{b}_{k}(A_{k,i}|H_{k,i})}\right){l}_k(H_{k,i})\right].
    \end{aligned}
\end{eqnarray}

In the rest of this supplementary material, we will focus on the case $K=2$. Therefore, the expressions become
\begin{equation}\label{eq:17}
\begin{aligned}
    &\psi(W_i;\eta,\hat{\alpha})\equiv\partial_{\eta}\Psi(W_i;\eta,\widehat{b},\widehat{\mathbb{E}})\\
    =&\left[\frac{\pi_1(A_{1,i}|H_{1,i})\pi_2(A_{2,i}|H_{2,i})}{\widehat{b}_1(A_{1,i}|H_{1,i},\hat{\alpha})\widehat{b}_2(A_{2,i}|H_{2,i},\hat{\alpha})}(\mathbb{I}\{R_{1,i}+{R}_{2,i}<\eta\}-\tau)\right.\\
    & \left.+\left(1-\frac{\pi_1(A_{1,i}|H_{1,i})}{\widehat{b}_1(A_{1,i}|H_{1,i})}\right)\widehat{\mathbb{E}}[(\mathbb{I}\{\widehat{R}_{1}+\widehat{R}_{2}<\eta\}-\tau)|X_{1,i},(A_{1,i},A_{2,i})\sim \pi]\right.\\
    &\left.+\left(\frac{\pi_1(A_{1,i}|H_{1,i})}{\widehat{b}_1(A_{1,i}|H_{1,i})}\right)\left(1-\frac{\pi_2(A_{2,i}|H_{2,i})}{\widehat{b}_2(A_{2,i}|H_{2,i})}\right)\widehat{\mathbb{E}}[(\mathbb{I}\{{R}_{1,i}+\widehat{R}_{2}<\eta\}-\tau)|H_{2,i},A_{2,i}\sim \pi]\right],
\end{aligned}
\end{equation}
and
\begin{equation}
\begin{aligned}
    \psi^*(W_i;\eta)= & \left[\frac{\pi_1(A_{1,i}|H_{1,i})\pi_2(A_{2,i}|H_{2,i})}{{b}_1(A_{1,i}|H_{1,i}){b}_2(A_{2,i}|H_{2,i})}(\mathbb{I}\{R_{1,i}+{R}_{2,i}<\eta\}-\tau)+a_1^*\widehat{\mathbb{E}}[(\mathbb{I}\{{R}_{1}+{R}_{2}<\eta\}\right.\\
    & \left.-\tau)|X_{1,i},(A_{1,i},A_{2,i})\sim \pi]+a_2^*\widehat{\mathbb{E}}[(\mathbb{I}\{{R}_{1,i}+{R}_{2}<\eta\}-\tau)|H_{2,i},A_{2,i}\sim \pi]\right],
\end{aligned}
\end{equation}
where we define $a_1^*$ and $a_2^*$ as
\begin{eqnarray}\label{eq:13_*}
    \begin{aligned}
        a_1^*&=\left(1-\frac{\pi_1(A_{1,i}|H_{1,i})}{{b}_1(A_{1,i}|H_{1,i})}\right),\quad a_2^*=\left(\frac{\pi_1(A_{1,i}|H_{1,i})}{{b}_1(A_{1,i}|H_{1,i})}\right)\left(1-\frac{\pi_2(A_{2,i}|H_{2,i})}{{b}_2(A_{2,i}|H_{2,i})}\right).
    \end{aligned}
\end{eqnarray}
\subsection{\textit{Rmean}}
In simulation studies, we compare our quantile-based DR mean estimator with the traditional DR mean estimator used in common literature, named as \textit{Rmean}. 
In single-stage setting, \textit{Rmean} is estimated from
\begin{equation}\label{eq:32}
    {Rmean}=\frac{1}{n}\sum_{i=1}^n \left[\frac{\pi_1(A_{1,i}|H_{1,i})}{\widehat{b}_1(A_{1,i}|H_{1,i})}R_{1,i}+\left(1-\frac{\pi_1(A_{1,i}|H_{1,i})}{\widehat{b}_1(A_{1,i}|H_{1,i})}\right)\widehat{\mathbb{E}}[\widehat{R}_{1}|X_1=X_{1,i},A_1\sim\pi]\right].
\end{equation}
In two-stage settings, 
\begin{equation}\label{eq:33}
    \begin{aligned}
      {Rmean}=\frac{1}{n}\sum_{i=1}^n &\left[\frac{\pi_1(A_{1,i}|H_{1,i})\pi_2(A_{2,i}|H_{2,i})}{\widehat{b}_1(A_{1,i}|H_{1,i})\widehat{b}_2(A_{2,i}|H_{2,i})}(R_{1,i}+{R}_{2,i})\right.\\
      &\left.+\left(1-\frac{\pi_1(A_{1,i}|H_{1,i})}{\widehat{b}_1(A_{1,i}|H_{1,i})}\right)\widehat{\mathbb{E}}[(\widehat{R}_{1}+\widehat{R}_{2})|H_1=H_{1,i},(A_1,A_2)\sim\pi]\right.\\
      &\left.+\left(\frac{\pi_1(A_{1,i}|H_{1,i})}{\widehat{b}_1(A_{1,i}|H_{1,i})}\right)\left(1-\frac{\pi_2(A_{2,i}|H_{2,i})}{\widehat{b}_2(A_{2,i}|H_{2,i})}\right)\widehat{\mathbb{E}}[(R_{1,i}+\widehat{R}_{2})|H_2=H_{2,i},A_2\sim \pi]\right].
    \end{aligned}
\end{equation}

\section{More on Simulation and Real Data Analysis}
\subsection{Simulation}
\subsubsection{Data Generating Process}\label{appendix:DGP}
We first introduce the data generating processes in both single stage and multiple stage settings.

In single stage quantile estimation problem, we generate the data as follows:
\begin{equation}
    \begin{aligned}
      &X\sim\mathbb{G}=\mathcal{N}(0,1);\\
      &A= b(X), \text{ where } b(X)=\mathbb{I}\{X+\frac{1}{4}\epsilon>0\};\\
      &R=(1-X+2AX)(1+\frac{1}{4}\epsilon'),
    \end{aligned}
\end{equation}
where $\epsilon$ and $\epsilon'$ are noise terms with different levels of heavy tail. The more heavy-tailed $\epsilon$ and $\epsilon'$ are, the harder it would be to estimate the reward distribution. As mentioned in the main paper, we use the student-t distribution with different degrees of freedom to generate heavy-tailed distributions. To illustrate the performance of our estimator under different levels of heavy tail, we let $\epsilon \text{ and }\epsilon'$ follow t distribution t(1.5), t(1.8), t(2), t(2.5), t(3), t(3.5), and t(4) respectively. The target policy here is defined as $\pi(A|X)=\mathbb{I}\{X>0\}$.

Let's take t(3) as an example. Figure \ref{fig:1} shows the PDF of reward $R^*(\pi)$ when $\epsilon \text{ and }\epsilon'$ follows t distribution with $df=3$. From the probability plot shown in Figure \ref{fig:2}, the rewards generated in this setting, shown by the blue dots, aligns very well with the baseline distribution t(3) which is shown by the red line. Therefore, the heavy tail level of $\epsilon$ and $\epsilon'$ indeed represents the heavy tail level of the reward distribution.

\begin{figure}[htbp]
\centering
\begin{minipage}[t]{0.48\textwidth}
\centering
\includegraphics[width=8cm]{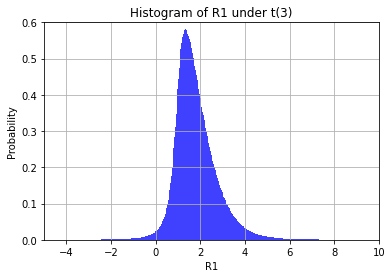}
\caption{\small PDF of $R^*(\pi)$ when $\epsilon\sim$ t(3)}
\label{fig:1}
\end{minipage}
\begin{minipage}[t]{0.48\textwidth}
\centering
\includegraphics[width=8cm]{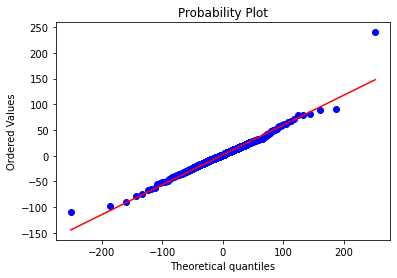}
\caption{\small Probability plot of $R^*(\pi)$ compared with baseline distribution t(3), shown by the red line. The reward distribution is indeed heavy-tailed.}
\label{fig:2}
\end{minipage}
\end{figure}

In multiple stage settings, we will only consider the case when $K=2$ for illustration purpose. The observational covariates-action-reward triplet $(X_1,A_1,R_1,X_2,A_2,R_2)$ is generated as follows:
\begin{equation}
    \begin{aligned}
      &X_1\sim\mathbb{G}=\mathcal{N}(0,1);\\
      &A_1|X_1= b_1(X_1), \text{ where } b_1(X_1)=\mathbb{I}\{X_1+\frac{1}{4}\epsilon_1>0\};\\
      &R_1|(X_1,A_1)=(1-X_1+2A_1X_1)(1+\frac{1}{4}\epsilon_2);\\
      &X_2|(X_1,A_1)\sim \frac{1}{2}X_1+\frac{1}{2}\epsilon_3;\\
      &A_2|(X_1,A_1,X_2)=b_2(H_2), \text{ where }  b_2(H_2)=\mathbb{I}\{X_2+\frac{1}{4}\epsilon_4>0\};\\
      &R_2|(X_1,A_1,X_2,A_2)=(1+0.5X_1+A_1X_1-X_2+3A_2X_2)(1+\frac{1}{4}\epsilon_5),
    \end{aligned}
\end{equation}
\begin{figure}[htbp]
\centering
\begin{minipage}[t]{0.48\textwidth}
\centering
\includegraphics[width=8cm]{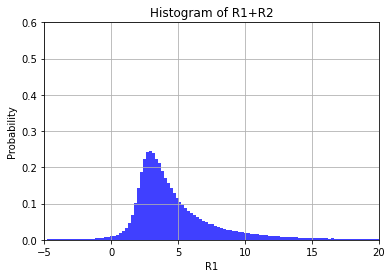}
\caption{\small PDF of $R^*(\pi)$ when $\epsilon\sim$ t(2)}\label{fig:7}
\end{minipage}
\begin{minipage}[t]{0.48\textwidth}
\centering
\includegraphics[width=8cm]{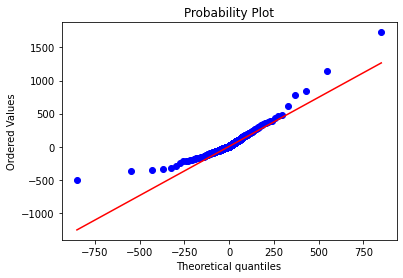}
\caption{\small Probability plot of $R^*(\pi)$ compared with baseline distribution t(2), shown by the red line. }\label{fig:8}
\end{minipage}
\end{figure}
where we define $\epsilon=(\epsilon_1,\epsilon_2,\epsilon_3,\epsilon_4,\epsilon_5)$ as the noise terms which denote student-t distributions with different degrees of freedom to control the heavy tail level of the cumulative reward distribution.
The target policy sequence at each stage is defined as
\begin{equation}
    \begin{aligned}
      &\pi_1(A_1|X_1)=\mathbb{I}\{X_1>0\},\\
      &\pi_2(A_2|X_1,A_1,X_2)=\mathbb{I}\{X_2>0\}.
    \end{aligned}
\end{equation}

Consider the cases when $\epsilon$ follows t(2), t(4), t(6), t(8) and $\mathcal{N}(0,1)$. Unlike the single-stage setting, here we use relatively larger degrees of freedom in t distribution since the estimation is harder when the number of decision stages increases. 

The PDF of $R^*(\pi)=R^*_1(\pi)+R^*_2(\pi)$ is shown in Figure \ref{fig:7}, and the probability plot of $R^*(\pi)$ comparing with t(2) distribution is shown in Figure \ref{fig:8}. The true reward distribution $R^*(\pi)$ is indeed quite heavy-tailed.

\subsubsection{Quantile Estimation Bias and Variance}

\begin{figure}[tbh]\label{fig:4}
\centering
\begin{minipage}[t]{0.48\textwidth}
\centering
\includegraphics[width=7.5cm]{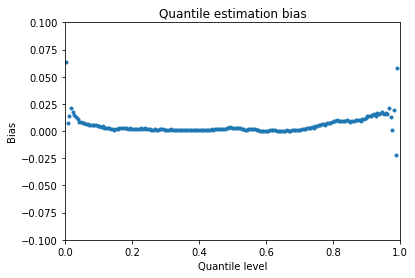}
\end{minipage}
\begin{minipage}[t]{0.48\textwidth}
\centering

\includegraphics[width=7.5cm]{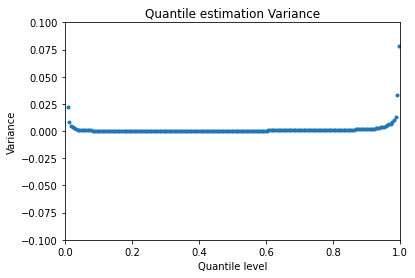}
\end{minipage}
\caption{\small Quantile estimation bias and variance at different quantile levels. The estimation bias and variance are quite close to $0$ at most of the quantile levels in the middle, and the performance tends to become unstable when $\tau$ approaches $0$ or $1$. The poor performance at extreme quantile levels is reasonable because of the lack of data points, which makes the statistical inference even harder to handle.
}
\end{figure}

\subsubsection{Bandwidth Selection}\label{appendix:bandwidth}
Since the estimation of $\widehat{J}_0^{\text{DR}}$ involves KDE, the choice of bandwidth $h$ is also an important problem to deal with. In Figure \ref{fig:10}, we tried three commonly used methods: cross-validation, Scott's method, and fixed bandwidths where $h=0.10$, $0.15$ and $0.20$, to compare their performances in estimating the standard error of $\widehat{\eta}_{\tau}^{\text{DR}}$. Surprisingly, both cross-validation and Scott's method encountered some over-smoothing issues, which is potentially caused by the heavy tail of the reward distribution. On the contrary, a reasonable fixed bandwidth tends to stabilize the estimation of $\widehat{J}_0^{\text{DR}}$, yielding a smaller MSE of the standard error estimated at all quantile levels. Since the performances under different fixed bandwidths are all pretty well, we will fix $h=0.15$ in simulation studies.

\begin{figure}[htbp]
\centering
\begin{minipage}[t]{0.32\textwidth}
\centering
\includegraphics[width=5.4cm]{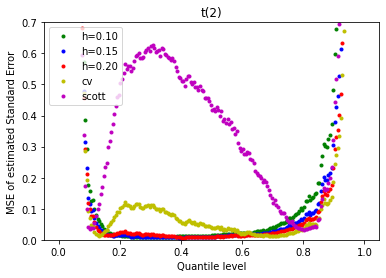}
\end{minipage}
\begin{minipage}[t]{0.32\textwidth}
\centering

\includegraphics[width=5.4cm]{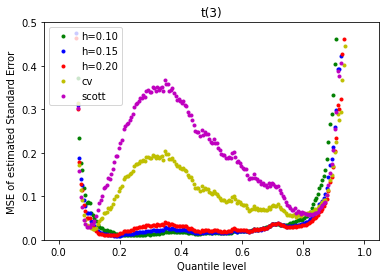}
\end{minipage}
\begin{minipage}[t]{0.32\textwidth}
\centering

\includegraphics[width=5.4cm]{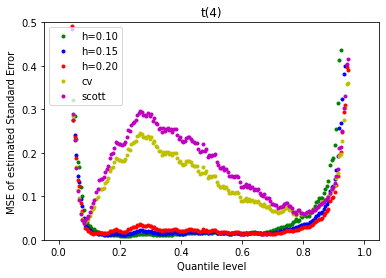}
\end{minipage}
\caption{\small The comparison between the MSE of the standard error at each quantile level $\tau$ with noise distribution t(2), t(3), and t(4).}\label{fig:10}
\end{figure}
 \subsubsection{Comparison between DM, IPW and DR}
 
\begin{figure}[htbp]
\centering
\begin{minipage}[t]{0.48\textwidth}
\centering
\includegraphics[width=7.5cm]{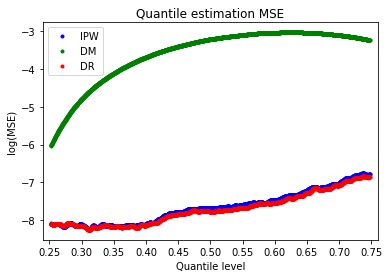}
\end{minipage}
\begin{minipage}[t]{0.48\textwidth}
\centering

\includegraphics[width=7.5cm]{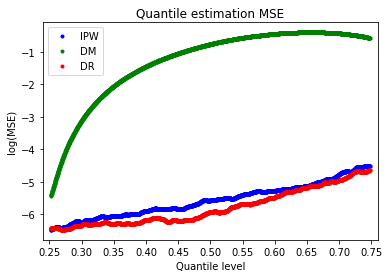}
\end{minipage}
\caption{\small The comparison between DM, IPW, and DR estimators with heavy tail level t(4) under single-stage (left) and multi-stage (right) settings. Logarithm transformation was taken to better distinguish the performance of the three methods.}\label{fig:13}
\end{figure}
 
 As we can see from the comparisons in Figure \ref{fig:13}, the performance of DM is clearly worse than IPW and DR. The performance of IPW estimator wasn't badly influenced due to the similarity of behavior policy and target policy under this specific simulation setting. As is expected to us, our DR estimator always yields the best result.

\subsection{Real Data Analysis}\label{sup:realdata}
In this section, we add the comparison between DM, IPW and DR in our real dataset.
\begin{figure}[tbh]
	\begin{center}
		\includegraphics[width=\textwidth]{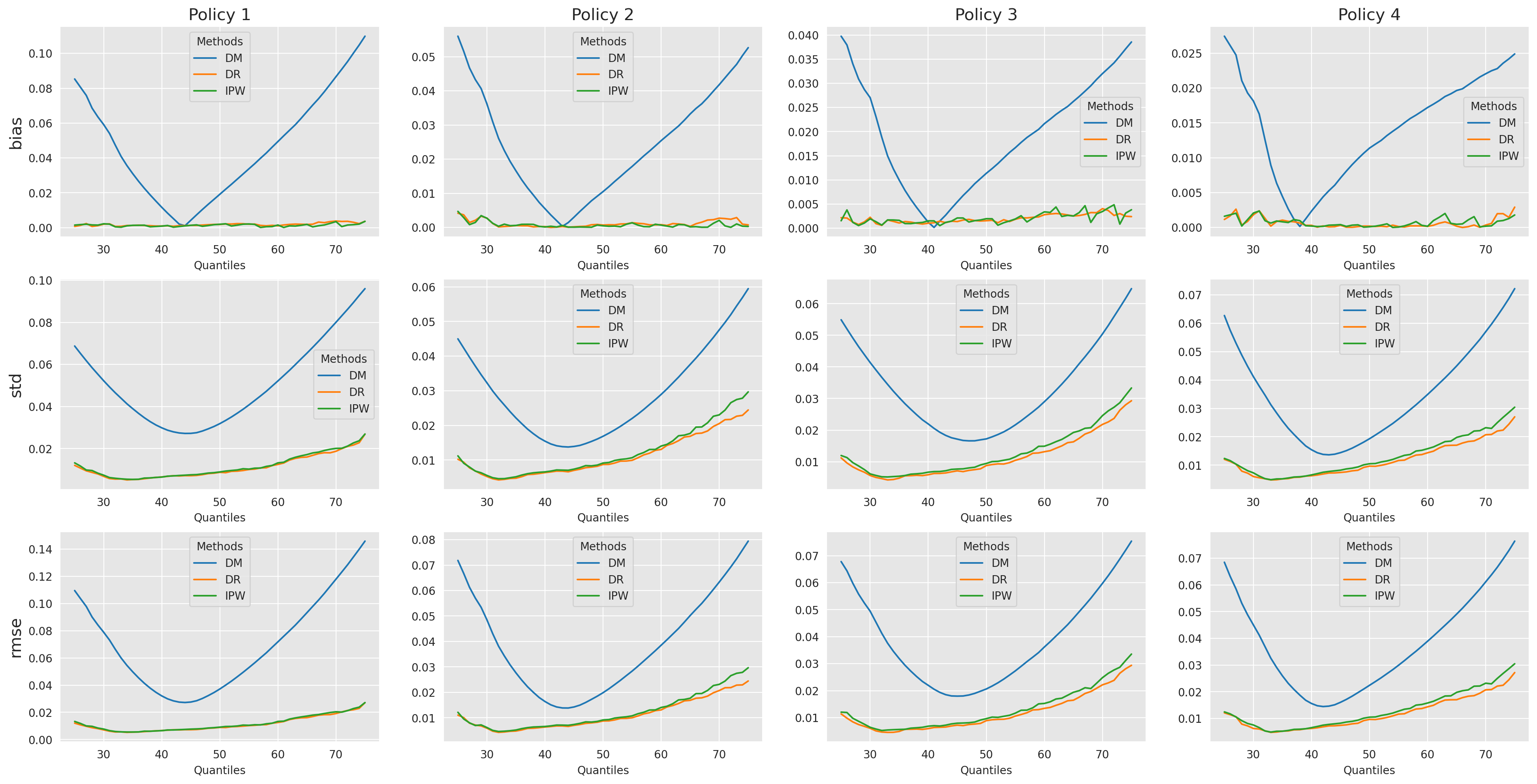}
	\end{center}
	\caption{Bias, standard deviation, and root mean squared error of quantiles over 100 replicates. The four columns correspond to four target policies, and each plot shows the comparison of three methods: DM-direct method, IPW-inverse probability weighting method and DR-doubly robust method.}
	\label{fig:real_data_analysis_cmp}
\end{figure}

As show in Figure \ref{fig:real_data_analysis_cmp}, our doubly robust approach performs the best in terms of bias, standard deviation and root mean squared errors, for all the four policies we considered.

\clearpage
\section{More on Theory Section}\label{sup:theory}
In this section, we provide several lemmas and proofs that are omitted in the main paper. 
\subsection{Lemma \ref{lemma:1}: Some Basic Results}\label{sup:2}
\begin{lemma}\label{lemma:1} Suppose $W=\{(X_k,A_k,R_k)\}_{k=1}^K$ is the full data with baseline information $X_1\sim \mathbb{G}$, and $\eta_\tau$ is the $\tau$th quantile of the potential cumulative reward function $R^*(\pi)$. Then we have
\begin{equation}
    \mathbb{E}_W[\psi^*(W;\eta_\tau)]=0.
\end{equation}
\end{lemma}

\begin{proof}
For the brevity of content, we'll only show the proof when $K=2$. The results in $K$ stage settings can be derived using the same logic.

To prove Lemma \ref{lemma:1}, we first separate $\psi^*(W;\eta_\tau)$ into three parts. Let $\psi(W^*;\eta_\tau)=a+b+c$, where
\begin{equation}
\begin{aligned}
    a &=\frac{\pi_1(A_{1}|H_{1})\pi_2(A_{2}|H_{2})}{{b}_1(A_{1}|H_{1}){b}_2(A_{2}|H_{2})}(\mathbb{I}\{R_{1}+{R}_{2}<\eta_\tau\}-\tau), \\
    b &=a_1^*\cdot \widehat{\mathbb{E}}[(\mathbb{I}\{{R}_{1}+{R}_{2}<\eta_\tau\}-\tau)|X_{1},(A_{1},A_{2})\sim \pi], \\
    c &=a_2^*\cdot \widehat{\mathbb{E}}[(\mathbb{I}\{{R}_{1}+\widehat{R}_{2}<\eta_\tau\}-\tau)|H_{2},A_{2}\sim \pi],
\end{aligned}
\end{equation}
where $a_1^*$ and $a_2^*$ are defined in Formula (\ref{eq:13_*}). Now it suffices to show $\mathbb{E}[a]=\mathbb{E}[b]=\mathbb{E}[c]=0$.

\noindent \textbf{1. $\boldsymbol{\mathbb{E}[a]=0}$.}

By change of measure theorem, 
\[
\begin{aligned}
\mathbb{E}[a]&=\mathbb{E}\left[\frac{\pi_1(A_{1}|H_{1})\pi_2(A_{2}|H_{2})}{{b}_1(A_{1}|H_{1}){b}_2(A_{2}|H_{2})}(\mathbb{I}\{R_{1}+{R}_{2}<\eta_\tau\}-\tau)\right]\\
&=\mathbb{E}\left[\mathbb{I}\{R_{1}^*(\pi)+{R}_{2}^*(\pi)<\eta_\tau\}\right]-\tau=\tau-\tau=0.
\end{aligned}
\]

\noindent \textbf{2. $\boldsymbol{\mathbb{E}[b]=0}$.}

\begin{equation}
\begin{aligned}
    \mathbb{E}[b]&=\mathbb{E}\left[\left(1-\frac{\pi_1(A_{1}|H_{1})}{{b}_1(A_{1}|H_{1})}\right) \widehat{\mathbb{E}}[(\mathbb{I}\{{R}_{1}+{R}_{2}<\eta_\tau\}-\tau)|X_{1},(A_{1},A_{2})\sim \pi]\right] \\
    &=\mathbb{E}\left[\mathbb{E}\left[\left(1-\frac{\pi_1(A_{1}|H_{1})}{{b}_1(A_{1}|H_{1})}\right)\Big|H_1\right]\cdot  {\mathbb{E}}\left[(\mathbb{I}\{{R}_{1}^{*}(\pi_1)+{R}_{2}^{*}(\pi)<\eta_\tau\}-\tau)\right]\right]\\ 
    &=\mathbb{E}\left[0\cdot {\mathbb{E}}\left[(\mathbb{I}\{{R}_{1}^{*}(\pi_1)+{R}_{2}^{*}(\pi)<\eta_\tau\}-\tau)\right]\right]=0.
\end{aligned}
\end{equation}

\noindent \textbf{3. $\boldsymbol{\mathbb{E}[c]=0}$.}

Similarly, since 
\begin{equation}
\begin{aligned}
    &\mathbb{E}[a_2^*|H_2]=\mathbb{E}\left[\left(\frac{\pi_1(A_{1}|H_{1})}{{b}_1(A_{1}|H_{1})}\right)\cdot\left(1-\frac{\pi_2(A_{2}|H_{2})}{{b}_2(A_{2}|H_{2})}\right)\Big|H_2\right]\\
    &=\mathbb{E}\left[\left(\frac{\pi_1(A_{1}|H_{1})}{{b}_1(A_{1}|H_{1})}\right)\cdot\mathbb{E}\left[\left(1-\frac{\pi_2(A_{2}|H_{2})}{{b}_2(A_{2}|H_{2})}\right)\Big|H_2\right]\right]
    =\mathbb{E}\left[\left(\frac{\pi_1(A_{1}|H_{1})}{{b}_1(A_{1}|H_{1})}\right)\cdot 0\right]=0,
\end{aligned}
\end{equation}
then we have
\begin{equation}
    \begin{aligned}
        \mathbb{E}[c]&=\mathbb{E}\left[a_2^*\cdot \widehat{\mathbb{E}}[(\mathbb{I}\{{R}_{1}+\widehat{R}_{2}<\eta\}-\tau)|H_{2},A_{2}\sim \pi]\right]\\
        &=\mathbb{E}\left[\mathbb{E}\left[a_2^*|H_2\right]\cdot\widehat{\mathbb{E}}[(\mathbb{I}\{{R}_{1}^{*}(\pi_1)+{R}_{2}^{*}(\pi)<\eta\}-\tau)|H_{2},A_{2}\sim \pi]\right]\\
        &=\mathbb{E}[0\cdot \widehat{\mathbb{E}}[(\mathbb{I}\{{R}_{1}^{*}(\pi_1)+{R}_{2}^{*}(\pi)<\eta\}-\tau)|H_{2},A_{2}\sim \pi]]=0.
    \end{aligned}
\end{equation}
The proof is thus complete.
\end{proof}

\subsection{Proof of Lemma \ref{lemma:2}}\label{sup:3}
Let's consider the case when $K=2$ for simplicity.
\begin{proof}
Suppose without the loss of generality that the sample size According to our estimating procedure, the quantile estimator $\widehat{\eta}_\tau^{\text{DR}}$ is the solution to the following estimating equation:
    \begin{equation}
        \frac{1}{S}\sum_{s=1}^S\mathbb{E}_{n,s}\left[ \psi(W_{s},{\eta},\hat{\alpha}_s)\right]=0,
    \end{equation}
where $W_{s}$ denotes the data in $\mathcal{I}_s$, and $\mathbb{E}_{n,s}$ is the empirical expectation over $s$th subgroup.
    
Suppose that $\Delta_n$ is a sequence of positive numbers that converge to $0$. We define $\mathcal{N}_{\Delta_n}(\eta_\tau)=\{\eta:\left\|\eta-\eta_{\tau}\right\|\leq \Delta_n\}$ as a shrinking neighborhood of $\eta_\tau$. According to Uniform Law of Large Numbers (ULLN), it suffices to show the following three conditions hold:
\begin{enumerate}
    \item $\sup_{\eta\in \mathcal{N}_{\Delta_n}(\eta_{\tau})}\left|\frac{1}{S}\sum_{s=1}^S\mathbb{E}_{n,s}\left[\psi(W;\eta,\hat{\alpha})\right]-\mathbb{E}\left[\psi^*(W;\eta)\right]\right|=o_p(1)$;
    \item $\forall \epsilon>0$, $\inf\{|\mathbb{E}\left[\psi^*(W;\eta)\right]|:\left\|\eta-\eta_{\tau}\right\|\geq \epsilon \}>0=\mathbb{E}\left[\psi^*(W;\eta_{\tau})\right]$;
    \item $\frac{1}{S}\sum_{s=1}^S\mathbb{E}_{n,s}\left[\psi(W;\widehat{\eta}_\tau^{\text{DR}},\hat{\alpha})\right]=o_p(1)$.
\end{enumerate}

Condition 2 can be easily proved by the identifiability of $\eta_{\tau}$. According to Assumption (A4), there exists a positive constant $C_1$, such that
\begin{equation}
    \partial_{\eta}\{\mathbb{E}_W[\psi^*(W;\eta)]\}={f}_{{R}^{*}({\pi})}(\eta)\geq C_1
\end{equation}
for all $\eta\in \mathcal{N}_{\Delta_n}(\eta_{\tau})$. According to Lemma 1, $\eta_\tau$ is the solution to $\mathbb{E}_W[\psi^*(W;\eta)]=0$. Therefore, the claim of condition 2 follows.

Condition 3 guarantees the estimated $\widehat{\eta}_\tau^{\text{DR}}$ to be close to the solution to the empirical estimating equation, which naturally holds under a valid estimating procedure. To complete the consistency proof, it remains to establish Condition 1.

Define the empirical process $\mathbb{G}_{n,s}\left[\psi(W_{s};\eta,\hat{\alpha}_s)\right]$ as a linear operator on measurable functions $\psi$ via
\[
\mathbb{G}_{n,s}\left[\psi(W_{s};\eta,\hat{\alpha}_s)\right]=\sqrt{n}\left[\frac{1}{n} \sum_{i\in \mathcal{I}_s}\psi(W_{i};\eta,\hat{\alpha}_s)-\mathbb{E}\left[\psi(W;\eta,\hat{\alpha}_s)\right]\right].
\]
Since
\begin{equation*}
    \begin{aligned}
        &\sup_{\eta\in \mathcal{N}_{\Delta_n}(\eta_{\tau})}\left|\frac{1}{S}\sum_{s=1}^S\mathbb{E}_{n,s}\left[\psi(W;\eta,\hat{\alpha})\right]-\mathbb{E}\left[\psi^*(W;\eta)\right]\right|\\
        &\leq \frac{1}{S}\sum_{s=1}^S\bigg\{\frac{1}{\sqrt{n}}\cdot \sup_{\eta\in \mathcal{N}_{\Delta_n}(\eta_{\tau})}\left|\mathbb{G}_{n,s}\left[\psi(W;\eta,\hat{\alpha}_s)\right]\right|+\sup_{\eta\in \mathcal{N}_{\Delta_n}(\eta_{\tau})}\left|\mathbb{E}\left[\psi(W;\eta,\hat{\alpha}_s)\right]-\mathbb{E}\left[\psi^*(W;\eta)\right]\right|\bigg\},
    \end{aligned}
\end{equation*}
it suffices to show that for any $s\in \{1,\dots,S\}$, the two terms on the RHS of the above inequality are $o_p(1)$. For the brevity of content, we omit the subscript $s$ in $\mathbb{G}_{n,s}$ and $\hat{\alpha}_s$ in the following proof to illustrate the general results that hold for any $s$. 

\noindent \textbf{Step 1}: Prove that $\sup_{\eta\in N_{\Delta_n}(\eta_{\tau})}\left|\mathbb{G}_{n}\left[\psi(W;\eta,\hat{\alpha})\right]\right|=o_p(\sqrt{n})$.

Define $\mathcal{F}=\{\psi(W;{\eta},\hat{\alpha}):\widehat{\eta}_{\tau}^{\text{DR}}\in \mathcal{N}_{\Delta_n}(\eta_{\tau})\}$. $\mathcal{F}$ is a VC Class. By Conditioning on $(W_i)_{i\in \mathcal{I}^c}$, $\hat{\alpha}$ can be regarded as a fixed value. Therefore, according to the Maximal Inequality in VC type classes \citep{chernozhukov2014gaussian,chernozhukov2018double}, there exist a sufficiently large constant $C$, such that with probability $1-o(1)$,
\begin{equation}
    \sup_{f\in \mathcal{F}}\left|\mathbb{G}_n(f)\right|\leq C\cdot \left(r_n\log^{1/2}(1/r_n)+n^{-1/2+1/q}\log n \right),
\end{equation}
where $r_n=\sup_{\eta\in N_{\Delta_n}(\eta_{\tau})}\left\|\psi(W;{\eta},\hat{\alpha})\right\|_{P,2}$, and $q\geq 2$ is an integer that can be arbitrarily large. To be more specific, 
\begin{equation*}
\small
    \begin{aligned}
      r_n=&sup_{\eta\in N_{\Delta_n}(\eta_{\tau})}\left\|\psi(W;{\eta},\hat{\alpha})\right\|_{P,2}\\
    =&sup_{\eta\in N_{\Delta_n}(\eta_{\tau})}\left\|\frac{\pi_1(A_{1,i}|H_{1,i})\pi_2(A_{2,i}|H_{2,i})}{\widehat{b}_1(A_{1,i}|H_{1,i},\hat{\alpha})\widehat{b}_2(A_{2,i}|H_{2,i},\hat{\alpha})}(\mathbb{I}\{R_{1,i}+{R}_{2,i}<\eta\}-\tau)\right.\\
    & \left.+\left(1-\frac{\pi_1(A_{1,i}|H_{1,i})}{\widehat{b}_1(A_{1,i}|H_{1,i})}\right)\widehat{\mathbb{E}}[(\mathbb{I}\{\widehat{R}_{1}+\widehat{R}_{2}<\eta\}-\tau)|X_{1,i},(A_{1,i},A_{2,i})\sim \pi]\right.\\
    &\left.+\left(\frac{\pi_1(A_{1,i}|H_{1,i})}{\widehat{b}_1(A_{1,i}|H_{1,i})}\right)\left(1-\frac{\pi_2(A_{2,i}|H_{2,i})}{\widehat{b}_2(A_{2,i}|H_{2,i})}\right)\widehat{\mathbb{E}}[(\mathbb{I}\{{R}_{1,i}+\widehat{R}_{2}<\eta\}-\tau)|H_{2,i},A_{2,i}\sim \pi]\right\|_{P,2}
    \end{aligned}
\end{equation*}
\begin{equation*}
\small
    \begin{aligned}
    \leq &sup_{\eta\in N_{\Delta_n}(\eta_{\tau})}\left\|\frac{\pi_1(A_{1,i}|H_{1,i})\pi_2(A_{2,i}|H_{2,i})}{\widehat{b}_1(A_{1,i}|H_{1,i},\hat{\alpha})\widehat{b}_2(A_{2,i}|H_{2,i},\hat{\alpha})}\right\|_{P,2}\cdot\left\|\mathbb{I}\{R_{1,i}+{R}_{2,i}<\eta\}-\tau\right\|_{P,2}\\
    & +\left\|1-\frac{\pi_1(A_{1,i}|H_{1,i})}{\widehat{b}_1(A_{1,i}|H_{1,i})}\right\|_{P,2}\cdot\left\|\widehat{\mathbb{E}}[(\mathbb{I}\{\widehat{R}_{1}+\widehat{R}_{2}<\eta\}-\tau)|X_{1,i},(A_{1,i},A_{2,i})\sim \pi]\right\|_{P,2}\\
    &+\left\|\left(\frac{\pi_1(A_{1,i}|H_{1,i})}{\widehat{b}_1(A_{1,i}|H_{1,i})}\right)\left(1-\frac{\pi_2(A_{2,i}|H_{2,i})}{\widehat{b}_2(A_{2,i}|H_{2,i})}\right)\right\|_{P,2}\cdot\left\|\widehat{\mathbb{E}}[(\mathbb{I}\{{R}_{1,i}+\widehat{R}_{2}<\eta\}-\tau)|H_{2,i},A_{2,i}\sim \pi]\right\|_{P,2}\\
    \leq & \frac{1}{\epsilon^2}\cdot 2+\frac{1}{\epsilon}\cdot 2+\frac{1}{\epsilon^2}\cdot 2\leq \frac{6}{\epsilon^2},
    \end{aligned}
\end{equation*}
where the second last inequality is obtained from the Positivity Assumption (C3) and Assumption (A5). Therefore, $r_n=O(1)$, and we have $\sup_{f\in \mathcal{F}}\left|\mathbb{G}_n(f)\right|\leq O_p(1)=o(\sqrt{n})$. 

\vspace{0.1in}
\noindent\textbf{Step 2}: Prove that $\sup_{\eta\in N_{\Delta_n}(\eta_{\tau})}\left|\mathbb{E}\left[\psi(W;\eta,\hat{\alpha})\right]-\mathbb{E}\left[\psi^*(W;\eta)\right]\right|=o_p(1)$.

In the following proof, we use $\widehat{\mathbb{E}}$ to denote the error aggregated by Monte Carlo Method, and use $\widehat{R}_k$ to denote the estimated reward obtained by MDN. Likewise, $\widehat{R}_k^*(\pi)$ is to denote the estimated potential outcome under policy $\pi$ at the $k$th stage. When $M\rightarrow \infty$, no error comes from the Monte Carlo method, and thus $\widehat{\mathbb{E}}=\mathbb{E}$. 

Let's first start with $\mathbb{E}\left[\psi(W;\eta,\hat{\alpha})\right]$.
\begin{eqnarray}\label{eq:lemma2.1}
\small
    \begin{aligned}
      &\mathbb{E}\left[\psi(W;\eta,\hat{\alpha})\right]\\
      =&\mathbb{E}\left[\left(\frac{\pi_1(A_{1,i}|H_{1,i})}{\widehat{b}_1(A_{1,i}|H_{1,i})}\right)\left(\frac{\pi_2(A_{2,i}|H_{2,i})}{\widehat{b}_2(A_{2,i}|H_{2,i})}\right)\cdot(\mathbb{I}\{R_{1,i}+R_{2,i}<\eta\}-\tau)\right.\\
    &\left.+\left(1-\frac{\pi_1(A_{1,i}|H_{1,i})}{\widehat{b}_1(A_{1,i}|H_{1,i})}\right)\cdot\widehat{\mathbb{E}}\Big[\mathbb{I}\{\widehat{R}_{1}+\widehat{R}_{2}<\eta\}-\tau\big|X_{1,i},(A_{1,i},A_{2,i})\sim \pi\Big]\right.\\
    &\left. + \left(\frac{\pi_1(A_{1,i}|H_{1,i})}{\widehat{b}_1(A_{1,i}|H_{1,i})}\right)\left(1-\frac{\pi_2(A_{2,i}|H_{2,i})}{\widehat{b}_2(A_{2,i}|H_{2,i})}\right)\cdot\widehat{\mathbb{E}}\Big[\mathbb{I}\{{R}_{1,i}+\widehat{R}_{2}<\eta\}-\tau\big| H_{2,i},A_{2,i}\sim\pi\Big]\right]\\
    =&\mathbb{E}\Bigg[\left(\frac{\pi_1(A_{1,i}|H_{1,i})}{\widehat{b}_1(A_{1,i}|H_{1,i})}\right)\left(\frac{\pi_2(A_{2,i}|H_{2,i})}{\widehat{b}_2(A_{2,i}|H_{2,i})}\right) \cdot \left(\mathbb{E}[\mathbb{I}\{R_{1,i}+R^{*}_{2}(\pi)<\eta\}\big| H_{2,i}]-\widehat{\mathbb{E}}[\mathbb{I}\{R_{1,i}+\widehat{R}^{*}_{2}(\pi)<\eta\}\big| H_{2,i}]\right)\\
    &+\left(\frac{\pi_1(A_{1,i}|H_{1,i})}{\widehat{b}_1(A_{1,i}|H_{1,i})}\right)\widehat{\mathbb{E}}[\mathbb{I}\{R_{1,i}+\widehat{R}^{*}_{2}(\pi)<\eta\}-\tau\big| H_{2,i}]
    \end{aligned}
\end{eqnarray}
\begin{eqnarray*}
\small
    \begin{aligned}
    &+\left(1-\frac{\pi_1(A_{1,i}|H_{1,i})}{\widehat{b}_1(A_{1,i}|H_{1,i})}\right)\widehat{\mathbb{E}}[\mathbb{I}\{\widehat{R}^{*}_{1}(\pi)+\widehat{R}^{*}_{2}(\pi)<\eta\}-\tau\big| H_{1,i}]\Bigg]:=L1+L2+L3,
    \end{aligned}
\end{eqnarray*}
where we define the last three lines in (\ref{eq:lemma2.1}) as $L1$, $L2$ and $L3$. Then
\begin{equation*}
\small
\begin{aligned}
    L2=&\mathbb{E}\left[\left(\frac{\pi_1(A_{1,i}|H_{1,i})}{\widehat{b}_1(A_{1,i}|H_{1,i})}\right)\cdot \widehat{\mathbb{E}}[\mathbb{I}\{R_{1,i}+\widehat{R}^{*}_{2}(\pi)<\eta\}-\tau\big| H_{2,i}]\right]\\
    =&\mathbb{E}\left[\left(\frac{\pi_1(A_{1,i}|H_{1,i})}{\widehat{b}_1(A_{1,i}|H_{1,i})}\right)\left(\frac{\pi_2(A_{2,i}|H_{2,i})}{{b}_2(A_{2,i}|H_{2,i})}\right)\cdot\widehat{\mathbb{E}}[\mathbb{I}\{R_{1,i}+\widehat{R}^{*}_{2}(\pi)<\eta\}-\tau\big| H_{2,i}]\right]\\
    =&\mathbb{E}\left[\left(\frac{\pi_1(A_{1,i}|H_{1,i})}{\widehat{b}_1(A_{1,i}|H_{1,i})}\right)\left(\frac{\pi_2(A_{2,i}|H_{2,i})}{{b}_2(A_{2,i}|H_{2,i})}\right) \cdot \left(\widehat{\mathbb{E}}[\mathbb{I}\{R_{1,i}+\widehat{R}^{*}_{2}(\pi)<\eta\}\big| H_{2,i}]-\widehat{\mathbb{E}}[\mathbb{I}\{R_{1,i}+{R}^{*}_{2}(\pi)<\eta\}\big| H_{2,i}]\right)\right.\\
    &\left.+\left(\frac{\pi_1(A_{1,i}|H_{1,i})}{\widehat{b}_1(A_{1,i}|H_{1,i})}\right) \widehat{\mathbb{E}}[\mathbb{I}\{R_{1,i}+{R}^{*}_{2}(\pi)<\eta\}-\tau\big| H_{2,i}]\right],
\end{aligned}
\end{equation*}
where we add and subtract a term to maintain the equivalence of the formula.

Plug in $L2$ to Formula (\ref{eq:lemma2.1}), we have
\begin{equation*}\label{eq:lemma2.2}
\small
    \begin{aligned}
      &\mathbb{E}\left[\psi(W;\eta,\hat{\alpha})\right]\\
      =&\mathbb{E}\left[\left(\frac{\pi_1(A_{1,i}|H_{1,i})}{\widehat{b}_1(A_{1,i}|H_{1,i})}\right)\left(\frac{\pi_2(A_{2,i}|H_{2,i})}{\widehat{b}_2(A_{2,i}|H_{2,i})}\right) \cdot \left(\mathbb{E}[\mathbb{I}\{R_{1,i}+R^{*}_{2}(\pi)<\eta\}\big| H_{2,i}]-\widehat{\mathbb{E}}[\mathbb{I}\{R_{1,i}+\widehat{R}^{*}_{2}(\pi)<\eta\}\big| H_{2,i}]\right)\right.\\
    &\left. +\left(\frac{\pi_1(A_{1,i}|H_{1,i})}{\widehat{b}_1(A_{1,i}|H_{1,i})}\right)\left(\frac{\pi_2(A_{2,i}|H_{2,i})}{{b}_2(A_{2,i}|H_{2,i})}\right) \cdot \left(\widehat{\mathbb{E}}[\mathbb{I}\{R_{1,i}+\widehat{R}^{*}_{2}(\pi)<\eta\}\big| H_{2,i}]-\widehat{\mathbb{E}}[\mathbb{I}\{R_{1,i}+{R}^{*}_{2}(\pi)<\eta\}\big| H_{2,i}]\right)\right.\\
    &\left.+\left(\frac{\pi_1(A_{1,i}|H_{1,i})}{\widehat{b}_1(A_{1,i}|H_{1,i})}\right) \widehat{\mathbb{E}}[\mathbb{I}\{R_{1,i}+{R}^{*}_{2}(\pi)<\eta\}-\tau\big| H_{2,i}]\right.\\
    &\left.+\left(1-\frac{\pi_1(A_{1,i}|H_{1,i})}{\widehat{b}_1(A_{1,i}|H_{1,i})}\right)\widehat{\mathbb{E}}[\mathbb{I}\{\widehat{R}^{*}_{1}(\pi)+\widehat{R}^{*}_{2}(\pi)<\eta\}-\tau\big| H_{1,i}]\right]\\
    =&\mathbb{E}\Bigg[\left(\frac{\pi_1(A_{1,i}|H_{1,i})}{\widehat{b}_1(A_{1,i}|H_{1,i})}\right)\left(\frac{\pi_2(A_{2,i}|H_{2,i})}{\widehat{b}_2(A_{2,i}|H_{2,i})}-\frac{\pi_2(A_{2,i}|H_{2,i})}{{b}_2(A_{2,i}|H_{2,i})}\right)\\
    &\qquad\cdot \left(\mathbb{E}[\mathbb{I}\{R_{1,i}+R^{*}_{2}(\pi)<\eta\}\big| H_{2,i}]-\widehat{\mathbb{E}}[\mathbb{I}\{R_{1,i}+\widehat{R}^{*}_{2}(\pi)<\eta\}\big| H_{2,i}]\right)\\
    &+\left(\frac{\pi_1(A_{1,i}|H_{1,i})}{\widehat{b}_1(A_{1,i}|H_{1,i})}\right) \widehat{\mathbb{E}}[\mathbb{I}\{R_{1,i}+{R}^{*}_{2}(\pi)<\eta\}-\tau\big| H_{2,i}]+\left(1-\frac{\pi_1(A_{1,i}|H_{1,i})}{\widehat{b}_1(A_{1,i}|H_{1,i})}\right)\widehat{\mathbb{E}}[\mathbb{I}\{\widehat{R}^{*}_{1}(\pi)+\widehat{R}^{*}_{2}(\pi)<\eta\}-\tau\big| H_{1,i}]\Bigg]
    \end{aligned}
\end{equation*}

Also, similar to the logic in Lemma 1, it's easy to show that the augmentation terms in $\psi^*(W;\eta)$ are $0$. Therefore,
\begin{equation*}
\small
\begin{aligned}
    \mathbb{E}\left[\psi^*(W;\eta)\right]&=\mathbb{E}\left[\frac{\pi_1(A_{1,i}|H_{1,i})\pi_2(A_{2,i}|H_{2,i})}{{b}_1(A_{1,i}|H_{1,i}){b}_2(A_{2,i}|H_{2,i})}\cdot\Big(\mathbb{I}\{(R_{1,i}+R_{2,i}<\eta\}-\tau\Big)\right]\\
    &=\mathbb{E}\left[\mathbb{I}\{(R_{1}^*(\pi)+R_{2}^*(\pi)<\eta\}-\tau\right].
\end{aligned}
\end{equation*}
Then we have

\begin{equation}\label{eq:lemma2.3}
\small
    \begin{aligned}
     &\mathbb{E}\left[\psi(W;\eta,\hat{\alpha})\right]-\mathbb{E}\left[\psi^*(W;\eta)\right]= \mathbb{E}\left[\psi(W;\eta,\hat{\alpha})\right]-\mathbb{E}\left[\mathbb{I}\{(R_{1}^*(\pi)+R_{2}^*(\pi)<\eta\}-\tau\right]\\
     =&\mathbb{E}\left[\left(\frac{\pi_1(A_{1,i}|H_{1,i})}{\widehat{b}_1(A_{1,i}|H_{1,i})}\right)\left(\frac{\pi_2(A_{2,i}|H_{2,i})}{\widehat{b}_2(A_{2,i}|H_{2,i})}-\frac{\pi_2(A_{2,i}|H_{2,i})}{{b}_2(A_{2,i}|H_{2,i})}\right)\right.\\
    &\left.\qquad\cdot \left(\mathbb{E}[\mathbb{I}\{R_{1,i}+R^{*}_{2}(\pi)<\eta\}\big| H_{2,i}]-\widehat{\mathbb{E}}[\mathbb{I}\{R_{1,i}+\widehat{R}^{*}_{2}(\pi)<\eta\}\big| H_{2,i}]\right)\right.\\
    &\left.+\left(\frac{\pi_1(A_{1,i}|H_{1,i})}{\widehat{b}_1(A_{1,i}|H_{1,i})}\right) \widehat{\mathbb{E}}[\mathbb{I}\{R_{1,i}+{R}^{*}_{2}(\pi)<\eta\}-\tau\big| H_{2,i}]\right.\\
    &+\left.\left(1-\frac{\pi_1(A_{1,i}|H_{1,i})}{\widehat{b}_1(A_{1,i}|H_{1,i})}\right)\widehat{\mathbb{E}}[\mathbb{I}\{\widehat{R}^{*}_{1}(\pi)+\widehat{R}^{*}_{2}(\pi)<\eta\}-\tau\big| H_{1,i}]-\mathbb{E}\left[\mathbb{I}\{(R_{1}^*(\pi)+R_{2}^*(\pi)<\eta\}-\tau|H_{1,i}\right]\right].
    \end{aligned}
\end{equation}
Let's then consider the last two lines in formula (\ref{eq:lemma2.3}).
\begin{equation}\label{eq:lemma2.4}
\small
    \begin{aligned}
      &\mathbb{E}\left[\left(\frac{\pi_1(A_{1,i}|H_{1,i})}{\widehat{b}_1(A_{1,i}|H_{1,i})}\right) \widehat{\mathbb{E}}[\mathbb{I}\{R_{1,i}+{R}^{*}_{2}(\pi)<\eta\}-\tau\big| H_{2,i}]\right.\\
    &\left.+\left(1-\frac{\pi_1(A_{1,i}|H_{1,i})}{\widehat{b}_1(A_{1,i}|H_{1,i})}\right)\widehat{\mathbb{E}}[\mathbb{I}\{\widehat{R}^{*}_{1}(\pi)+\widehat{R}^{*}_{2}(\pi)<\eta\}-\tau\big| H_{1,i}]-\mathbb{E}\left[\mathbb{I}\{(R_{1}^*(\pi)+R_{2}^*(\pi)<\eta\}-\tau|H_{1,i}\right]\right]\\
    &=\mathbb{E}\left[\left(\frac{\pi_1(A_{1,i}|H_{1,i})}{\widehat{b}_1(A_{1,i}|H_{1,i})}\right)\left(\widehat{\mathbb{E}}[\mathbb{I}\{(R^*_{1}(\pi)+{R}^{*}_{2}(\pi)<\eta\}\big| H_{1,i}]-\widehat{\mathbb{E}}[\mathbb{I}\{\widehat{R}^{*}_{1}(\pi)+\widehat{R}^{*}_{2}(\pi)<\eta\}\big| H_{1,i}]\right)\right.\\
    &\left. + \widehat{\mathbb{E}}[\mathbb{I}\{\widehat{R}^{*}_{1}(\pi)+\widehat{R}^{*}_{2}(\pi)<\eta\}\big| H_{1,i}]- \widehat{\mathbb{E}}[\mathbb{I}\{R^*_{1}(\pi)+R^*_{2}(\pi)<\eta\}\big|H_{1,i}]\right]
    \end{aligned}
\end{equation}
\begin{equation*}
\small
    \begin{aligned}
    &=\mathbb{E}\left[\left(\frac{\pi_1(A_{1,i}|H_{1,i})}{\widehat{b}_1(A_{1,i}|H_{1,i})}\right)\left(\widehat{\mathbb{E}}[\mathbb{I}\{R^*_{1}(\pi)+{R}^{*}_{2}(\pi)<\eta\}\big| H_{1,i}]-\widehat{\mathbb{E}}[\mathbb{I}\{\widehat{R}^{*}_{1}(\pi)+\widehat{R}^{*}_{2}(\pi)<\eta\}\big| H_{1,i}]\right)\right.\\
    &\left. - \left(\frac{\pi_1(A_{1,i}|H_{1,i})}{{b}_1(A_{1,i}|H_{1,i})}\right)\left( \widehat{\mathbb{E}}[\mathbb{I}\{R^*_{1}(\pi)+R^*_{2}(\pi)<\eta\}\big|H_{1,i}]-\widehat{\mathbb{E}}[\mathbb{I}\{\widehat{R}^{*}_{1}(\pi)+\widehat{R}^{*}_{2}(\pi)<\eta\}\big| H_{1,i}]\right)\right]\\
    &=\mathbb{E}\left[\left(\frac{\pi_1(A_{1,i}|H_{1,i})}{\widehat{b}_1(A_{1,i}|H_{1,i})}-\frac{\pi_1(A_{1,i}|H_{1,i})}{{b}_1(A_{1,i}|H_{1,i})}\right)\left( \widehat{\mathbb{E}}[\mathbb{I}\{R^*_{1}(\pi)+R^*_{2}(\pi)<\eta\}\big|H_{1,i}]-\widehat{\mathbb{E}}[\mathbb{I}\{\widehat{R}^{*}_{1}(\pi)+\widehat{R}^{*}_{2}(\pi)<\eta\}\big| H_{1,i}]\right)\right].
    \end{aligned}
\end{equation*}

Combining the result of Formula (\ref{eq:lemma2.3}) and Formula (\ref{eq:lemma2.4}), we have
\begin{equation}\label{eq:lemma2.5}
\small
    \begin{aligned}
    &\left|\mathbb{E}\left[\psi(W;\eta,\hat{\alpha})\right]-\mathbb{E}\left[\psi^*(W;\eta)\right]\right|\\
     = &\left|\mathbb{E}\left[\left(\frac{\pi_1(A_{1,i}|H_{1,i})}{\widehat{b}_1(A_{1,i}|H_{1,i})}\right)\left(\frac{\pi_2(A_{2,i}|H_{2,i})}{\widehat{b}_2(A_{2,i}|H_{2,i})}-\frac{\pi_2(A_{2,i}|H_{2,i})}{{b}_2(A_{2,i}|H_{2,i})}\right)\right.\right.\\
    &\left.\left.\qquad\cdot \left(\mathbb{E}[\mathbb{I}\{R_{1,i}+R^{*}_{2}(\pi)<\eta\}\big| H_{2,i}]-\widehat{\mathbb{E}}[\mathbb{I}\{R_{1,i}+\widehat{R}^{*}_{2}(\pi)<\eta\}\big| H_{2,i}]\right)\right]\right.\\
    +&\left.\mathbb{E}\left[\left(\frac{\pi_1(A_{1,i}|H_{1,i})}{\widehat{b}_1(A_{1,i}|H_{1,i})}-\frac{\pi_1(A_{1,i}|H_{1,i})}{{b}_1(A_{1,i}|H_{1,i})}\right)\right.\right.\\
    &\left.\left.\qquad\cdot \left( \widehat{\mathbb{E}}[\mathbb{I}\{R^*_{1}(\pi)+R^*_{2}(\pi)<\eta\}\big|H_{1,i}]-\widehat{\mathbb{E}}[\mathbb{I}\{\widehat{R}^{*}_{1}(\pi)+\widehat{R}^{*}_{2}(\pi)<\eta\}\big| H_{1,i}]\right)\right]\right|\\
    &\leq \frac{1}{\epsilon^3}\cdot \left\|\widehat{b}_2(A_{2,i}|H_{2,i})-{b}_2(A_{2,i}|H_{2,i})\right\|_{P,2}\cdot\left\|\delta\left(\widehat{F}_{{R}^{*}_2({\pi})|H_2},F_{{R}_2^{*}({\pi})|H_2}\right)\right\|_{P,2}\\
      & + \frac{1}{\epsilon^2}\cdot \left\|\widehat{b}_1(A_{1,i}|H_{1,i})-{b}_1(A_{1,i}|H_{1,i})\right\|_{P,2}\cdot\left\|\delta\left(\widehat{F}_{{R}^{*}({\pi})|H_1},F_{{R}^{*}({\pi})|H_1}\right)\right\|_{P,2}
    \end{aligned}
\end{equation}
where the last inequality holds by condition (C3$'$) and Assumption (A5).

Therefore, as long as the propensity score models are consistently estimated, or the outcome regression models for $\widehat{R}_1$ and $\widehat{R}_2$ obtained by MDN are consistently estimated, we have
\begin{equation*}
    \left|\mathbb{E}\left[\psi(W;\eta,\hat{\alpha})\right]-\mathbb{E}\left[\psi^*(W;\eta)\right]\right|=o(1),
\end{equation*}
according to formula (\ref{eq:lemma2.5}).

To be more specific, we only need one of the following conditions to be satisfied:

\begin{enumerate}
    \item $\left\|\widehat{b}_1(A_1|H_1)-{b}_1(A_1|H_1)\right\|_{P,2}=o(1)$, and $\left\|\widehat{b}_2(A_2|H_2)-{b}_2(A_2|H_2)\right\|_{P,2}=o(1)$;
    \item $\left\|\delta\left(\widehat{F}_{{R}^{*}(\pi)|H_1},F_{{R}^{*}({\pi})|H_1}\right)\right\|_{P,2}=o(1)$, and $\left\|\delta\left(\widehat{F}_{{R}_2^{*}(\pi)|H_2},F_{{R}_2^{*}(\pi)|H_2}\right)\right\|_{P,2}=o(1)$.
\end{enumerate}
The proof of Lemma 2 is thus complete.

Notice that Assumption (A1) and (A2) are not necessary to prove the double robustness. These assumptions are imposed to further prove the asymptotic normality of our quantile estimator in Section \ref{sup:4}.
\end{proof}

\subsection{Lemma \ref{lemma:3}: Double Robustness of our Variance Estimator}\label{sup:3.5}

\begin{lemma}\label{lemma:3}
In kernel density estimation, suppose the kernel function $K$ is a real-valued, Borel-measurable function in $L^\infty(\mathbb{R})$ which satisfies $\lim_{x\rightarrow \infty}|xK(x)|=\lim_{x\rightarrow \infty}|xK^2(x)|=0$. Define the bandwidth as $h_N$, which can be a function of the sample size. We assume $\lim_{N\rightarrow \infty} h_N=0$, and $\lim_{N\rightarrow \infty} Nh_N=\infty$. 

We claim that $\widehat{\sigma}^2_{\text{DR}}$ is a consistent estimator of ${\sigma}^2$, as long as one of the following two parts of models is consistently estimated: 
\begin{enumerate}
    \item[(1)] The propensity score functions at each stage, i.e. $\{\widehat{b}_k(H_{k})\}_{k=1}^K$.
    \item[(2)] The conditional expectation functions at each stage, i.e. $\{\widehat{L}_k(H_{k})\}_{k=1}^K$,\\
    or equivalently the data generating models for $\{\widehat{R}_k|(H_k,A_k)\}_{k=1}^K$.
\end{enumerate}
\end{lemma}

\begin{proof}
Firstly, we show that when $\lim_{N\rightarrow \infty} h_N=0$, and $\lim_{N\rightarrow \infty} Nh_N=\infty$, the kernel density estimator for $f(r):=f_{R^*(\pi)}(r)$ attains weak consistency. 

Let's first define $f_N(r)$ as the KDE of $f(r)$. By Chebyshev's Inequality, at each point of continuity $r$ of $f$ and for any $\epsilon>0$,
\[
\mathbb{P}(|f_N(r)-f(r)|>\epsilon)\leq \frac{\mathbb{E}[(f_N(r)-f(r))^2]}{\epsilon^2}=\frac{\text{Var}[f_N(r)]+\text{Bias}^2[f_N(r)]}{\epsilon^2}.
\]
According to the result in \citet{parzen1962estimation}, $
lim_{N\rightarrow \infty}\text{Bias}^2[f_N(r)]=0$ when $\lim_{N\rightarrow \infty} h_N=0$. That is, $f_N(r)$ is asymptotically unbiased when $\lim_{N\rightarrow \infty} h_N=0$.

According to the mild assumptions on $K$ stated in Lemma 3, the asymptotic variance of the density estimator is given by
\[
lim_{N\rightarrow \infty} Nh_N\text{Var}[f_N(r)]=f(r)\int_{-\infty}^\infty K^2(u) dy<\infty.
\]
Therefore, both the asymptotic bias and variance go to $0$ as $N\rightarrow\infty$. This illustrates the asymptotic consistency of the kernel density estimator in $\widehat{\sigma}^2_{\text{DR}}$.

The rest of the proof is trivial. One can follow the idea in Lemma 2 to similarly prove the double robustness of our variance estimator. For the brevity of content, we omit the details for this part.
\end{proof}

\subsection{Proof of Theorem 1}\label{sup:4}

\begin{proof}
same as the proof in Lemma 1-3, we will only illustrate the case when $K=2$ for the brevity of content. We will give a brief sketch in Step 1, and leave the rest of the details in Step 2-5 accordingly.

Based on our estimating procedure in theory section, we have
\begin{equation}
\begin{aligned}
    &\sum_{s=1}^S\mathbb{E}_{n,s}[\psi(W_s;\widehat{\eta}_{\tau}^{\text{DR}},\hat{\alpha}_s)]=\frac{1}{n}\sum_{s=1}^S\sum_{i\in \mathcal{I}_s}\psi(W_{i,s};\widehat{\eta}_{\tau}^{\text{DR}},\hat{\alpha}_s)=0;  \\
    &\mathbb{E}[\psi^*(W_s;\eta_{\tau})] =\mathbb{E}[a+b+c]=0\quad \text{ for any }s\in\{1,\dots,S\}.
\end{aligned}
\end{equation}

\noindent\textbf{Step 1}: 

First, we apply Taylor series expansion on $ \mathbb{E}[\psi^*(W;\eta)]$ around $\eta_\tau$.
\begin{equation}\label{eq:36}
\begin{aligned}
    \mathbb{E}[\psi^*(W;\eta)]&=\mathbb{E}[\psi^*(W;\eta_{\tau})]+J_0(\eta-\eta_{\tau})+\frac{1}{2}H_0(\eta-\eta_{\tau})^2+o(\left\|\eta-\eta_{\tau}\right\|^2)\\
    &=J_0(\eta-\eta_{\tau})+\frac{1}{2}H_0(\eta-\eta_{\tau})^2+o(\left\|\eta-\eta_{\tau}\right\|^2),
\end{aligned}
\end{equation}
where $J_0=\partial_{\eta}\{\mathbb{E}_W[\psi^*(W;\eta)]\}|_{\eta=\eta_{\tau}}$ and $H_0=\partial^2_{\eta}\{\mathbb{E}_W[\psi^*(W;\eta)]\}|_{\eta=\eta_{\tau}}$.

Suppose that $\mathbb{E}[\sum_{s=1}^S\psi(W_s;\widehat{\eta}_{\tau}^{\text{DR}},\hat{\alpha}_s)]=\mathbb{E}[\sum_{s=1}^S\psi^*(W_s;\widehat{\eta}_{\tau}^{\text{DR}})]+o(n^{-1/2})$ (We will prove it in Step 4). Since $\mathbb{E}[\psi^*(W_s;{\eta}_{\tau})]=0$, by plugging in the definition of empirical operator, we got
\begin{equation}\label{eq:37}
\begin{aligned}
    & \sqrt{n}\sum_{s=1}^S\mathbb{E}_{n,s}[\psi(W_s;\widehat{\eta}_{\tau}^{\text{DR}},\hat{\alpha}_s)-\psi^*(W_s;\eta_{\tau})]\\
    &=\sum_{s=1}^S\mathbb{G}_{n,s}[\psi(W_s;\widehat{\eta}_{\tau}^{\text{DR}},\hat{\alpha}_s)-\psi^*(W_s;\eta_{\tau})]+\sqrt{n}\sum_{s=1}^S\mathbb{E}[\psi(W_s;\widehat{\eta}_{\tau}^{\text{DR}},\hat{\alpha}_s)]\\
    & =\sum_{s=1}^S\mathbb{G}_{n,s}[\psi(W_s;\widehat{\eta}_{\tau}^{\text{DR}},\hat{\alpha}_s)-\psi^*(W_s;\eta_{\tau})]+\sqrt{n}\sum_{s=1}^S\mathbb{E}[\psi^*(W_s;\widehat{\eta}_{\tau}^{\text{DR}})]+o_p(1). 
\end{aligned}
\end{equation}

Let $\eta=\widehat{\eta}_{\tau}^{\text{DR}}$ in formula (\ref{eq:36}) and combine it with formula (\ref{eq:37}). Then
\begin{equation}
\begin{aligned}
    \sqrt{n}\sum_{s=1}^S&J_0(\widehat{\eta}_{\tau}^{\text{DR}}-\eta_{\tau})+\sqrt{n}\sum_{s=1}^S\mathbb{E}_{n,s}[\psi^*(W_s;\eta_{\tau})]=-\sum_{s=1}^S\mathbb{G}_{n,s}[\psi(W_s;\widehat{\eta}_{\tau}^{\text{DR}},\hat{\alpha}_s)-\psi^*(W_s;\eta_{\tau})]\\
    &-\frac{1}{2}\sqrt{n}\sum_{s=1}^S H_0 (\widehat{\eta}_{\tau}^{\text{DR}}-\eta_{\tau})^2-\sqrt{n}\cdot o(\left\|\widehat{\eta}_{\tau}^{\text{DR}}-\eta_{\tau}\right\|^2)+o_p(1).
\end{aligned}
\end{equation}
Therefore,
\begin{eqnarray*}\label{eq:thm.1}
\begin{aligned}
    &\sqrt{S}\cdot \sqrt{N}\left\|J_0(\widehat{\eta}_{\tau}^{\text{DR}}-\eta_{\tau})+\frac{1}{S}\sum_{s=1}^S\mathbb{E}_{n,s}[\psi^*(W_s;\eta_{\tau})]\right\|\\
    =&\sqrt{n}\left\|\sum_{s=1}^S J_0(\widehat{\eta}_{\tau}^{\text{DR}}-\eta_{\tau})+\sum_{s=1}^S\mathbb{E}_{n,s}[\psi^*(W_s;\eta_{\tau})]\right\|\leq  \left\|\sum_{s=1}^S\mathbb{G}_{n,s}[\psi(W_s;\widehat{\eta}_{\tau}^{\text{DR}},\hat{\alpha}_s)-\psi^*(W_s;\eta_{\tau})]\right\|\\
    +&\frac{1}{2}\sqrt{n}\sum_{s=1}^S H_0 (\widehat{\eta}_{\tau}^{\text{DR}}-\eta_{\tau})^2+\sqrt{n}\cdot o(\left\|\widehat{\eta}_{\tau}^{\text{DR}}-\eta_{\tau}\right\|^2)+o_p(1).
\end{aligned}
\end{eqnarray*}
To finish the proof of this theorem, it suffices to show that
\begin{equation}
    \sqrt{N}\left\|J_0(\widehat{\eta}_{\tau}^{\text{DR}}-\eta_{\tau})+\frac{1}{S}\sum_{s=1}^S\mathbb{E}_{n,s}[\psi^*(W_s;\eta_{\tau})]\right\|=o_p(1),
\end{equation}
which is satisfied if we can prove that all terms on the RHS of the inequality in formula (\ref{eq:thm.1}) are $o_p(1)$. That is, we only need to show
\begin{enumerate}
    \item $\left\|\sum_{s=1}^S\mathbb{G}_{n,s}[\psi(W_s;\widehat{\eta}_{\tau}^{\text{DR}},\hat{\alpha}_s)-\psi^*(W_s;\eta_{\tau})]\right\|=o_p(1)$, which follows directly from triangle inequality if we can prove for any $s\in\{1,\dots,S\}$, \[\left\|\mathbb{G}_{n,s}[\psi(W_s;\widehat{\eta}_{\tau}^{\text{DR}},\hat{\alpha}_s)-\psi^*(W_s;\eta_{\tau})]\right\|=o_p(1).\]
    \item $\left\|\widehat{\eta}_{\tau}^{\text{DR}}-\eta_{\tau}\right\|= O_p(n^{-1/2})$.
    \item $\mathbb{E}[\sum_{s=1}^S\psi(W_s;\widehat{\eta}_{\tau}^{\text{DR}},\hat{\alpha}_s)]=\mathbb{E}[\sum_{s=1}^S\psi^*(W_s;\widehat{\eta}_{\tau}^{\text{DR}})]+o(n^{-1/2})$, which follows naturally if we can prove for any $s\in\{1,\dots,S\}$,
    \[\mathbb{E}[\psi(W_s;\widehat{\eta}_{\tau}^{\text{DR}},\hat{\alpha}_s)]=\mathbb{E}[\psi^*(W_s;\widehat{\eta}_{\tau}^{\text{DR}})]+o(n^{-1/2}).\]
    \item $H_0$ is bounded.
\end{enumerate}
Detailed proofs of these four conditions are provided in Step 2-5. Assuming condition 1-4 hold, we have
\begin{eqnarray}
   \sqrt{N}\left\|J_0(\widehat{\eta}_{\tau}^{\text{DR}}-\eta_{\tau})+\frac{1}{S}\sum_{s=1}^S\mathbb{E}_{n,s}[\psi^*(W_s;\eta_{\tau})]\right\|=o_p(1),
\end{eqnarray}
then
\begin{eqnarray*}
   \sqrt{N}(\widehat{\eta}_{\tau}^{\text{DR}}-\eta_{\tau})=-J_0^{-1}\frac{1}{S}\sum_{s=1}^S\mathbb{E}_{n,s}[\psi^*(W_s;\eta_{\tau})]+o_p(1)=-J_0^{-1}\frac{1}{\sqrt{N}}\sum_{i=1}^N \psi^*(W_i;\eta_{\tau})+o_p(1).
\end{eqnarray*}
The claim of this theorem thus follows.

Next, we will prove that all of these four conditions are satisfied in Step 2, Step 3, Step 4 and Step 5 respectively.

In the following proof, when the result holds for any fold $s\in\{1,\dots,S\}$, the subscript $s$ in $W_s$, $\hat{\alpha}_s$, $\mathbb{G}_{n,s}$ and $\mathbb{E}_{n,s}$ will be omitted to avoid the redundancy of notations.

\noindent\textbf{Step 2}: Prove that $\left\|\mathbb{G}_{n}[\psi(W;\widehat{\eta}_{\tau}^{\text{DR}},\hat{\alpha})-\psi^*(W;\eta_{\tau})]\right\|=o_p(1)$.

Define $\mathcal{N}_{\tau_n}(\eta_{\tau})=\{\eta:\left\|{\eta}-\eta_{\tau}\right\|\leq\tau_n\}$ as a neighborhood of $\eta_{\tau}$ with size $\tau_n$, and define $\mathcal{F}=\{\psi(W;\eta,\hat{\alpha})-\psi(W;{\eta}_{\tau},\hat{\alpha}):\eta\in \mathcal{N}_{\tau_n}(\eta_{\tau})\}$, where $\tau_n$ is a sequence of positive constants converging to zero. (According to Step 3, $\tau_n$ can be any sequence with order $\tau_n\geq O(n^{-1/2})$).

Since
\begin{equation}\label{eq:43}
    \begin{aligned}
        \left\|\mathbb{G}_n[\psi(W;\widehat{\eta}_{\tau}^{\text{DR}},\hat{\alpha})-\psi^*(W;\eta_{\tau})]\right\| &\leq \sup_{\eta\in\mathcal{N}_{\tau_n}(\eta_\tau)}\left\|\mathbb{G}_n[\psi(W;\eta,\hat{\alpha})-\psi(W;{\eta}_{\tau},\hat{\alpha})]\right\|\\
        &+\left\|\mathbb{G}_n[\psi(W;{\eta}_{\tau},\hat{\alpha})-\psi^*(W;{\eta}_{\tau})]\right\|,
    \end{aligned}
\end{equation}
It suffices to show both terms on the RHS of Formula (\ref{eq:43}) is $o_p(1)$.

\noindent\textbf{Step 2.1}: Prove that $\sup_{\eta\in\mathcal{N}_{\tau_n}(\eta_\tau)}\left\|\mathbb{G}_n[\psi(W;{\eta},\hat{\alpha})-\psi(W;{\eta}_{\tau},\hat{\alpha})]\right\|=o_p(1)$.

First, it can be shown that $\mathcal{F}_{0}^{\hat{\alpha}}=\{\psi(W;{\eta},\hat{\alpha}):\eta\in \mathcal{N}_{\tau_n}(\eta_{\tau})\}$ is a VC Class with VC-index $V(\mathcal{F}_{0}^{\hat{\alpha}})$. Also, for some $q> 2$, the entropy is bounded by some constant $C_2$, i.e. $\left\|F_{0}^{\hat{\alpha}}\right\|_{P,q}\leq C_2$. According to Lemma 2.6.18 in \citet{van1996weak}, $\mathcal{F}$ is also a VC Class. By Conditioning on $(W_i)_{i\in \mathcal{I}^c}$, $\hat{\alpha}$ can be regarded as a fixed value. By applying the Maximal Inequality specialized to VC type classes \citep{chernozhukov2014gaussian}, there exists a sufficiently large constant $C$, such that with probability $1-o(1)$,
\begin{equation}
    \sup_{f\in \mathcal{F}}\left|\mathbb{G}_n(f)\right|\leq C\cdot \left(r_N\log^{1/2}(1/r_N)+n^{-1/2+1/q}\log n \right),
\end{equation}
where 
\begin{equation}
    r_N=\sup_{\left\|{\eta}-\eta_{\tau}\right\|\leq\tau_N}\left\|\psi(W;\eta,\hat{\alpha})-\psi(W;{\eta}_{\tau},\hat{\alpha})\right\|_{P,2}.
\end{equation}

If we can show that $r_N=o(1)$, then $\sup_{f\in \mathcal{F}}\left|\mathbb{G}_n(f)\right|=o(1)$, and the claim of Step 2.1 follows.

To prove this, we split $\psi$ into three parts and show that each part is indeed $o(1)$. Let $\psi(W;\eta,\hat{\alpha})=a(\eta,\hat{\alpha})+b(\eta,\hat{\alpha})+c(\eta,\hat{\alpha})$, where
\begin{equation}\label{eq:thm.4}
\begin{aligned}
    a(\eta,\hat{\alpha}) &=\frac{\pi_1(A_{1}|H_{1})\pi_2(A_{2}|H_{2})}{\widehat{b}_1(A_{1}|H_{1})\widehat{b}_2(A_{2}|H_{2})}(\mathbb{I}\{R_{1}+{R}_{2}<\eta\}-\tau), \\
    b(\eta,\hat{\alpha}) &=a_1(\hat{\alpha})\cdot \widehat{\mathbb{E}}[(\mathbb{I}\{\widehat{R}_{1}+\widehat{R}_{2}<\eta\}-\tau)|X_{1},(A_{1},A_{2})\sim \pi], \\
    c(\eta,\hat{\alpha}) &=a_2(\hat{\alpha})\cdot \widehat{\mathbb{E}}[(\mathbb{I}\{{R}_{1}+\widehat{R}_{2}<\eta\}-\tau)|H_{2},A_{2}\sim \pi],
\end{aligned}
\end{equation}
and $a_1(\hat{\alpha})$, $a_2(\hat{\alpha})$ are defined as
\begin{equation}\label{eq:13}
    \begin{split}
        a_1(\hat{\alpha})&=\left(1-\frac{\pi_1(A_{1,i}|H_{1,i})}{\widehat{b}_1(A_{1,i}|H_{1,i})}\right),\quad a_2(\hat{\alpha})=\left(\frac{\pi_1(A_{1,i}|H_{1,i})}{\widehat{b}_1(A_{1,i}|H_{1,i})}\right)\left(1-\frac{\pi_2(A_{2,i}|H_{2,i})}{\widehat{b}_2(A_{2,i}|H_{2,i})}\right).
    \end{split}
\end{equation}
Specifically, when the estimating equation is $\psi^*(W;\eta)$, we let $\psi^*(W;\eta):=a^*(\eta)+b^*(\eta)+c^*(\eta)$.

\noindent\textbf{Step 2.1.a}: Prove that $\sup_{\eta\in \mathcal{N}_{\tau_n}(\eta_{\tau})}\left\|a(\eta,\hat{\alpha})-a({\eta}_{\tau},\hat{\alpha})\right\|_{P,2}=o(1)$.
\begin{equation}\label{eq:48}
\small
    \begin{aligned}
        &\sup_{\eta\in \mathcal{N}_{\tau_n}(\eta_{\tau})}\left\|a(\eta,\hat{\alpha})-a({\eta}_{\tau},\hat{\alpha})\right\|_{P,2}\\
        &=\sup_{\eta\in \mathcal{N}_{\tau_n}(\eta_{\tau})}\left\|\frac{\pi_1(A_1|H_1)\pi_2(A_2|H_2)}{\widehat{b}_1(A_{1}|H_{1})\widehat{b}_2(A_{2}|H_{2})}\Big(\mathbb{I}\{R_{1}+R_{2}<\eta\}-\mathbb{I}\{R_{1}+R_{2}<\eta_\tau\}\Big)\right\|_{P,2}\\
        &=\sup_{\eta\in \mathcal{N}_{\tau_n}(\eta_{\tau})}\left\|\frac{\pi_1(A_1|H_1)\pi_2(A_2|H_2)}{\widehat{b}_1(A_{1}|H_{1})\widehat{b}_2(A_{2}|H_{2})}\mathbb{I}\big\{R_{1}+R_{2}\in \left[\min(\eta_{\tau},\eta),\max(\eta_{\tau},\eta)\right)\big\}\right\|_{P,2}\\
        &\leq \left\|\frac{\pi_1(A_1|H_1)\pi_2(A_2|H_2)}{\widehat{b}_1(A_{1}|H_{1})\widehat{b}_2(A_{2}|H_{2})}\right\|_{P,2}\cdot \sup_{\eta\in \mathcal{N}_{\tau_n}(\eta_{\tau})}\left(\mathbb{E}\left|\mathbb{I}\big\{R^{*}({\pi})\in \left[\min(\eta_{\tau},\eta),\max(\eta_{\tau},\eta)\right]\big\}\right|^2\right)^{\frac{1}{2}}\\
        &=\left\|\frac{\pi_1(A_1|H_1)\pi_2(A_2|H_2)}{\widehat{b}_1(A_{1}|H_{1})\widehat{b}_2(A_{2}|H_{2})}\right\|_{P,2}\cdot \sup_{\eta\in \mathcal{N}_{\tau_n}(\eta_{\tau})}\left(\mathbb{E}\Big[\mathbb{I}\big\{R^{*}({\pi})\in [\min(\eta_{\tau},\eta),\max(\eta_{\tau},\eta)]\big\}\Big]\right)^{\frac{1}{2}}\\
        &\leq \frac{1}{\epsilon^2}  \cdot \sup_{\eta\in \mathcal{N}_{\tau_n}(\eta_{\tau})}\left|F_{R^{*}({\pi})}(\eta)-F_{R^{*}({\pi})}({\eta}_{\tau})\right|^{\frac{1}{2}}\\
        &=\frac{1}{\epsilon^2}  \cdot \sup_{\eta\in \mathcal{N}_{\tau_n}(\eta_{\tau})}\left|\mathbb{E}_{H_1\sim \mathbb{G}}\left[F_{R^{*}(\pi)|H_1}(\eta|H_1)-F_{R^{*}(\pi)|H_1}({\eta}_{\tau}|H_1)\right]\right|^{\frac{1}{2}},
    \end{aligned}
\end{equation}
where $F_{{R}^{*}(\pi)}(\eta)$ is the marginal cumulative density function of random variable $R^{*}(\pi)=R_1^{*}(\pi)+R_2^{*}(\pi)$. Notice that the first inequality in (\ref{eq:48}) holds because of Cauchy-Schwarz inequality, and the second inequality holds by Assumption (A5).

According to Assumption (A3), there exists a constant $C_3$, such that for all $r\in\mathbb{R}$ and $H_1$, $\left|f_{R^{*}(\pi)|H_1}(r|H_1)\right|\leq C_3$. Thus,
\begin{equation}\label{eq:thm.3}
        \left|F_{R^{*}(\pi)|H_1}(\eta|H_1)-F_{R^{*}(\pi)|H_1}({\eta}_{\tau}|H_1)\right|\leq C_3 \cdot \left\|\eta-\eta_{\tau}\right\|.
\end{equation}
Then we have
\begin{equation}
    \begin{aligned}
        &\sup_{\eta\in \mathcal{N}_{\tau_n}(\eta_{\tau})}\left\|a(\eta,\hat{\alpha})-a({\eta}_{\tau},\hat{\alpha})\right\|_{P,2}\\
        &\leq\frac{1}{\epsilon^2}  \cdot \sup_{\eta\in \mathcal{N}_{\tau_n}(\eta_{\tau})}\left|\mathbb{E}_{H_1\sim \mathbb{G}}\left[F_{R^{*}(\pi)|H_1}(\eta|H_1)-F_{R^{*}(\pi)|H_1}({\eta}_{\tau}|H_1)\right]\right|^{\frac{1}{2}}\\
        &\leq \frac{1}{\epsilon^2}  \cdot \sup_{\eta\in \mathcal{N}_{\tau_n}(\eta_{\tau})}\left(C_3\cdot \left\|\eta-\eta_{\tau}\right\|\right)^{\frac{1}{2}} =\frac{1}{\epsilon^2}  \cdot o(1)=o(1),
    \end{aligned}
\end{equation}
which finishes the proof of this Step 2.1.a.

\noindent\textbf{Step 2.1.b}: Prove that $\sup_{\eta\in \mathcal{N}_{\tau_n}(\eta_{\tau})}\left\|b(\eta,\hat{\alpha})-b({\eta}_{\tau},\hat{\alpha})\right\|_{P,2}=o(1)$.
\begin{equation}\label{eq:thm.5}
    \begin{aligned}
       &\sup_{\eta\in \mathcal{N}_{\tau_n}(\eta_{\tau})}\left\|b(\eta,\hat{\alpha})-b({\eta}_{\tau},\hat{\alpha})\right\|_{P,2}\\
       =&\sup_{\eta\in \mathcal{N}_{\tau_n}(\eta_{\tau})}\left\|\left(1-\frac{\pi_1(A_{1,i}|H_{1,i})}{\widehat{b}_1(A_{1,i}|H_{1,i})}\right)\widehat{\mathbb{E}}[(\mathbb{I}\{\widehat{R}^*_{1}(\pi)+\widehat{R}^*_{2}(\pi)<\eta\}-\mathbb{I}\{\widehat{R}^*_{1}(\pi)+\widehat{R}^*_{2}(\pi)<{\eta}_{\tau}\})\big| H_1]\right\|_{P,2}\\
       =&\sup_{\eta\in \mathcal{N}_{\tau_n}(\eta_{\tau})}\left\|\left(1-\frac{\pi_1(A_{1,i}|H_{1,i})}{\widehat{b}_1(A_{1,i}|H_{1,i})}\right)\cdot\left(\widehat{F}_{R^*(\pi)|H_1}(\eta;\hat{\alpha}|H_1)-\widehat{F}_{R^*(\pi)|H_1}({\eta}_{\tau};\hat{\alpha}|H_1)\right)\right\|_{P,2}\\
       &\leq \left\|1-\frac{\pi_1(A_{1,i}|H_{1,i})}{\widehat{b}_1(A_{1,i}|H_{1,i})}\right\|_{P,2}\cdot \sup_{\eta\in \mathcal{N}_{\tau_n}(\eta_{\tau})}\left\|\widehat{F}_{R^*(\pi)|H_1}(\eta;\hat{\alpha}|H_1)-\widehat{F}_{R^*(\pi)|H_1}({\eta}_{\tau};\hat{\alpha}|H_1)\right\|_{P,2}\\
       &\leq \frac{2}{\epsilon}\cdot \sup_{\eta\in \mathcal{N}_{\tau_n}(\eta_{\tau})}\left\|\widehat{F}_{R^*(\pi)|H_1}(\eta;\hat{\alpha}|H_1)-\widehat{F}_{R^*(\pi)|H_1}({\eta}_{\tau};\hat{\alpha}|H_1)\right\|_{P,2},
    \end{aligned}
\end{equation}
where $\widehat{F}_{R^*(\pi)}(r;\hat{\alpha}|H_1)$ is the estimated cdf of $(\widehat{R}_1+\widehat{R}_2)$ given baseline information $H_1$ and policy $\pi$. Also, we have
\begin{eqnarray}\label{eq:thm.2}
   \begin{aligned}
    &\left\|\widehat{F}_{R^*(\pi)|H_1}(\eta;\hat{\alpha}|H_1)-\widehat{F}_{R^*(\pi)|H_1}({\eta}_\tau;\hat{\alpha}|H_1)\right\|_{P,2} \leq \left\|\widehat{F}_{R^*(\pi)|H_1}(\eta;\hat{\alpha}|H_1)-{F}_{R^*(\pi)|H_1}(\eta|H_1)\right\|_{P,2}\\
    &+\left\|\widehat{F}_{R^*(\pi)|H_1}({\eta}_\tau;\hat{\alpha}|H_1)-{F}_{R^*(\pi)|H_1}({\eta}_\tau|H_1)\right\|_{P,2}
    +\left\|{F}_{R^*(\pi)|H_1}(\eta|H_1)-{F}_{R^*(\pi)|H_1}({\eta}_\tau|H_1)\right\|_{P,2}.
\end{aligned}
\end{eqnarray}
Now let's analyze the three terms on the RHS of formula (\ref{eq:thm.2}).

\noindent On the one hand, according to Assumption (A2), for any $\eta$,
\begin{equation}\label{eq:49}
    \left\|\widehat{F}_{R^*(\pi)|H_1}(\eta;\hat{\alpha}|H_1)-{F}_{R^*(\pi)|H_1}(\eta|H_1)\right\|_{P,2}\leq \left\|\delta\left(\widehat{F}_{{R}^{*}(\pi)|H_1},F_{{R}^{*}(\pi)|H_1}\right)\right\|_{P,2}=o(n^{-1/4}).
\end{equation}
The first and second terms in (\ref{eq:thm.2}) can thus be bounded.

\noindent On the other hand, by assuming (A3) and utilizing the same proof as in formula (\ref{eq:thm.3}), 
\begin{equation}\label{eq:50}
    \sup_{\eta\in \mathcal{N}_{\tau_n}(\eta_{\tau})}\left\|{F}_{R^*(\pi)|H_1}(\eta|H_1)-{F}_{R^*(\pi)|H_1}({\eta}_\tau|H_1)\right\|_{P,2}=o(1).
\end{equation}
Combining the result of (\ref{eq:49}), (\ref{eq:50}) with formula (\ref{eq:thm.2}), we have
\begin{equation*}
\begin{aligned}
    &\sup_{\eta\in \mathcal{N}_{\tau_n}(\eta_{\tau})}\left\|b(\eta,\hat{\alpha})-b({\eta}_{\tau},\hat{\alpha})\right\|_{P,2}\\
    &\leq\frac{2}{\epsilon}\cdot \sup_{\eta\in \mathcal{N}_{\tau_n}(\eta_{\tau})}\left\|\widehat{F}_{R^*(\pi)|H_1}(\eta;\hat{\alpha}|H_1)-\widehat{F}_{R^*(\pi)|H_1}({\eta}_{\tau};\hat{\alpha}|H_1)\right\|_{P,2}
    =\frac{2}{\epsilon}\cdot o(1)=o(1).
    \end{aligned}
\end{equation*}
The proof of Step 2.1.b is thus complete.

\noindent\textbf{Step 2.1.c}: Prove that $\sup_{\eta\in \mathcal{N}_{\tau_n}(\eta_{\tau})}\left\|c(\eta,\hat{\alpha})-c({\eta}_{\tau},\hat{\alpha})\right\|_{P,2}=o(1)$.

Similar to Step 2.1.b, it follows from Condition (C3$'$) that
\begin{equation}\label{eq:51}
    \begin{aligned}
      & \sup_{\eta\in \mathcal{N}_{\tau_n}(\eta_{\tau})}\left\|c(\eta,\hat{\alpha})-c({\eta}_{\tau},\hat{\alpha})\right\|_{P,2}\\
      &=\sup_{\eta\in \mathcal{N}_{\tau_n}(\eta_{\tau})}\left\|a_2(\hat{\alpha})\cdot \widehat{\mathbb{E}}[(\mathbb{I}\{{R}_{1}+\widehat{R}_{2}^*(\pi)<\eta\}-\mathbb{I}\{{R}_{1}+\widehat{R}_{2}^*(\pi)<{\eta}_{\tau}\})\big| H_{2,i}]\right\|_{P,2}\\
      &=\sup_{\eta\in \mathcal{N}_{\tau_n}(\eta_{\tau})}\left\|a_2(\hat{\alpha})\cdot \widehat{\mathbb{E}}\left[\mathbb{I}\big\{{R}_{1}+\widehat{R}_{2}^*(\pi)\in \left[\min(\eta_{\tau},\eta),\max(\eta_{\tau},\eta)\right]\big\}\big| H_{2,i}\right]\right\|_{P,2}\\
      &\leq \left\|a_2(\hat{\alpha})\right\|_{P,2}\cdot \sup_{\eta\in \mathcal{N}_{\tau_n}(\eta_{\tau})}\left\| \widehat{F}_{{R}_{2}^*(\pi)|H_2}(\eta-R_{1};\hat{\alpha}|H_{2,i})-\widehat{F}_{{R}_{2}^*(\pi)|H_2}({\eta}_{\tau}-R_{1};\hat{\alpha}|H_{2,i})\right\|_{P,2},
    \end{aligned}
\end{equation}
where $\widehat{F}_{{R}_{2}^*(\pi)|H_2}(r;\hat{\alpha}|H_{2,i})$ is the estimated cdf of $\widehat{R}_2$ given historical data $H_2$ and policy $\pi$. 
We know that
\begin{equation}
    \begin{aligned}
      \left\|a_2(\hat{\alpha})\right\|_{P,2}=\left\|\left(\frac{\pi_1(A_{1,i}|H_{1,i})}{\widehat{b}_1(A_{1,i}|H_{1,i})}\right)\left(1-\frac{\pi_2(A_{2,i}|H_{2,i})}{\widehat{b}_2(A_{2,i}|H_{2,i})}\right)\right\|_{P,2}\leq \frac{2}{\epsilon^2}.
    \end{aligned}
\end{equation}
Also, 
\begin{equation}\label{eq:52}
\begin{aligned}
    &\left\| \widehat{F}_{{R}_{2}^*(\pi)|H_2}(\eta-R_{1};\hat{\alpha}|H_{2,i})-\widehat{F}_{{R}_{2}^*(\pi)|H_2}({\eta}_{\tau}-R_{1};\hat{\alpha}|H_{2,i})\right\|_{P,2}\\
    &\leq \left\|\widehat{F}_{{R}_{2}^*(\pi)|H_2}(\eta-R_{1};\hat{\alpha}|H_{2,i})-{F}_{{R}_{2}^*(\pi)|H_2}(\eta-R_{1}|H_{2,i})\right\|_{P,2}\\
    &+\left\|\widehat{F}_{{R}_{2}^*(\pi)|H_2}({\eta}_{\tau}-R_{1};\hat{\alpha}|H_{2,i})-{F}_{{R}_{2}^*(\pi)|H_2}({\eta}_{\tau}-R_{1}|H_{2,i})\right\|_{P,2}\\
    &+\left\|{F}_{{R}_{2}^*(\pi)|H_2}(\eta-R_{1}|H_{2,i})-{F}_{{R}_{2}^*(\pi)|H_2}({\eta}_{\tau}-R_{1}|H_{2,i})\right\|_{P,2},
\end{aligned}
\end{equation}
where according to Assumption (A2), it holds for any $\eta$ that
\begin{equation}\label{eq:53}
\begin{aligned}
   \left\|\widehat{F}_{{R}_{2}^*(\pi)|H_2}(\eta-R_{1};\hat{\alpha}|H_{2,i})-{F}_{{R}_{2}^*(\pi)|H_2}(\eta-R_{1}|H_{2,i})\right\|_{P,2}\leq \left\|\delta\left(\widehat{F}_{{R}_2^{*}(\pi)|H_2},F_{{R}_2^{*}(\pi)|H_2}\right)\right\|_{P,2}=o_p(n^{-1/4}),   
\end{aligned}
\end{equation}
and by using Assumption (A3) and the proof in formula (\ref{eq:thm.3}),
\begin{equation}\label{eq:54}
    \sup_{\eta\in \mathcal{N}_{\tau_n}(\eta_{\tau})}\left\|{F}_{{R}_{2}^*(\pi)|H_2}(\eta-R_{1}|H_{2,i})-{F}_{{R}_{2}^*(\pi)|H_2}({\eta}_{\tau}-R_{1}|H_{2,i})\right\|_{P,2}=o(1).
\end{equation}

By summarizing the result of formula (\ref{eq:53}) and (\ref{eq:54}), the LHS of formula (\ref{eq:52}) is thus $o_p(1)$. Therefore, we have
\begin{equation}
\begin{aligned}
    &\sup_{\eta\in \mathcal{N}_{\tau_n}(\eta_{\tau})}\left\|c(\eta,\hat{\alpha})-c({\eta}_{\tau},\hat{\alpha})\right\|_{P,2}\\
    &\leq \left\|a_2(\hat{\alpha})\right\|_{P,2}\cdot \sup_{\eta\in \mathcal{N}_{\tau_n}(\eta_{\tau})}\left\| \widehat{F}_{{R}_{2}^*(\pi)|H_2}(\eta-R_{1};\hat{\alpha}|H_{2,i})-\widehat{F}_{{R}_{2}^*(\pi)|H_2}({\eta}_{\tau}-R_{1};\hat{\alpha}|H_{2,i})\right\|_{P,2}\\
    &\leq \frac{2}{\epsilon^2}\cdot o(1)=o(1).  
\end{aligned}
\end{equation}

Combining the results in Step 2.1.a-2.1.c, we have
\[
\sup_{\eta\in\mathcal{N}_{\tau_n}(\eta_\tau)}\left\|\mathbb{G}_n[\psi(W;{\eta},\hat{\alpha})-\psi(W;{\eta}_{\tau},\hat{\alpha})]\right\|=o_p(1).
\]
The claim of step 2.1 thus follows.

\noindent\textbf{Step 2.2}: Prove that $\left\|\mathbb{G}_n[\psi(W;{\eta}_{\tau},\hat{\alpha})-\psi^*(W;{\eta}_{\tau})]\right\|=o_p(1)$.

Notice that under each Fold $s$, $\hat{\alpha}$ is independent with $W$. By Chebyshev's Inequality, it suffices to show that
\begin{equation}
    \text{Var}[\mathbb{G}_n[\psi(W;{\eta}_{\tau},\hat{\alpha})-\psi^*(W;{\eta}_{\tau})]]=\text{Var}[\psi(W;{\eta}_{\tau},\hat{\alpha})-\psi^*(W;{\eta}_{\tau})]=o(1).
\end{equation}
According to our definition in formula (\ref{eq:thm.4}), it follows from triangular inequality that
\begin{equation}\label{eq:65}
    \begin{aligned}
      &(\text{Var}[\psi(W;{\eta}_{\tau},\hat{\alpha})-\psi^*(W;{\eta}_{\tau})])^{1/2}\leq (\text{Var}[a({\eta}_{\tau},\hat{\alpha})-a^*({\eta}_{\tau})])^{1/2}\\
      &+(\text{Var}[b({\eta}_{\tau},\hat{\alpha})-b^*({\eta}_{\tau})])^{1/2}+(\text{Var}[c({\eta}_{\tau},\hat{\alpha})-c^*({\eta}_{\tau})])^{1/2},
    \end{aligned}
\end{equation}
in order to show the LHS is $o_p(1)$, we only need to prove that the three parts on the RHS are all $o_p(1)$.
\begin{equation}
\small
    \begin{aligned}
      &(\text{Var}[a({\eta}_{\tau},\hat{\alpha})-a^*({\eta}_{\tau})])^{1/2}\leq \left\|a({\eta}_{\tau},\hat{\alpha})-a^*({\eta}_{\tau})\right\|_{P,2}\\ &=\left\|\left(\frac{\pi_1(A_{1}|H_{1})\pi_2(A_{2}|H_{2})}{\widehat{b}_1(A_{1}|H_{1})\widehat{b}_2(A_{2}|H_{2})}-\frac{\pi_1(A_{1}|H_{1})\pi_2(A_{2}|H_{2})}{{b}_1(A_{1}|H_{1}){b}_2(A_{2}|H_{2})}\right)(\mathbb{I}\{R_{1}+{R}_{2}<{\eta}_{\tau}\}-\tau)\right\|_{P,2}\\
      &\leq \frac{1}{\epsilon^4}\left\|\widehat{b}_1(A_{1}|H_{1})\widehat{b}_2(A_{2}|H_{2})-{b}_1(A_{1}|H_{1}){b}_2(A_{2}|H_{2})\right\|_{P,2}\cdot \left\|\mathbb{I}\{R_{1}+{R}_{2}<{\eta}_{\tau}\}-\tau \right\|_{P,2}\\
      &\leq \frac{2}{\epsilon^4}\left\|\widehat{b}_1(A_{1}|H_{1})\widehat{b}_2(A_{2}|H_{2})-{b}_1(A_{1}|H_{1}){b}_2(A_{2}|H_{2}) \right\|_{P,2}\\
      &\leq\frac{2}{\epsilon^4}\left[\left\|\widehat{b}_1(A_{1}|H_{1})\Big\{\widehat{b}_2(A_{2}|H_{2})-{b}_2(A_{2}|H_{2})\Big\}\right\|_{P,2}+\left\|{b}_2(A_{2}|H_{2})\Big\{\widehat{b}_1(A_{1}|H_{1})-{b}_1(A_{1}|H_{1})\Big\}\right\|_{P,2}\right]\\
      &\leq \frac{2}{\epsilon^{5}}\left[\left\|\widehat{b}_1(A_{1}|H_{1})-{b}_1(A_{1}|H_{1})\right\|_{P,2}+\left\|\widehat{b}_2(A_{2}|H_{2})-{b}_2(A_{2}|H_{2})\right\|_{P,2}\right]=o(1),
    \end{aligned}
\end{equation}
where we used Condition (C3$'$), Assumption (A5), triangle inequality, and last equality holds by Assumption (A1). 

Likewise,
\begin{equation}\label{eq:56}
\small
    \begin{aligned}
      &(\text{Var}[b({\eta}_{\tau},\hat{\alpha})-b^*({\eta}_{\tau})])^{1/2}\leq \left\|b({\eta}_{\tau},\hat{\alpha})-b^*({\eta}_{\tau})\right\|_{P,2}\\
      &= \left\| a_1(\hat{\alpha})\cdot \widehat{\mathbb{E}}[(\mathbb{I}\{\widehat{R}_{1}^*(\pi)+\widehat{R}_{2}^*(\pi)<\eta_\tau\}-\tau)|H_{1}]-a_1^*\cdot \widehat{\mathbb{E}}[(\mathbb{I}\{{R}_{1}^*(\pi)+{R}_{2}^*(\pi)<\eta_\tau\}-\tau)|H_{1}]\right\|_{P,2}\\
      &\leq \left\| a_1(\hat{\alpha})\cdot\Big\{ \widehat{\mathbb{E}}[\mathbb{I}\{\widehat{R}_{1}^*(\pi)+\widehat{R}_{2}^*(\pi)<\eta_\tau\}-\mathbb{I}\{{R}_{1}^*(\pi)+{R}_{2}^*(\pi)<\eta_\tau\}|H_{1}]\Big\}\right\|_{P,2}\\
      &+\left\| \Big\{a_1(\hat{\alpha})-a_1^*\Big\}\cdot \widehat{\mathbb{E}}[(\mathbb{I}\{{R}_{1}^*(\pi)+{R}_{2}^*(\pi)<\eta_\tau\}-\tau)|H_{1}]\right\|_{P,2}\\
      &\leq\left\|\frac{\pi_1(A_{1,i}|H_{1,i})}{\widehat{b}_1(A_{1,i}|H_{1,i})}\right\|_{P,2}\cdot \left\|\widehat{\mathbb{E}}[\mathbb{I}\{\widehat{R}_{1}^*(\pi)+\widehat{R}_{2}^*(\pi)<\eta_\tau\}-\mathbb{I}\{{R}_{1}^*(\pi)+{R}_{2}^*(\pi)<\eta_\tau\}|H_{1}]\right\|_{P,2}^2\\
      & +\left\|\frac{\pi_1(A_{1,i}|H_{1,i})}{\widehat{b}_1(A_{1,i}|H_{1,i})}-\frac{\pi_1(A_{1,i}|H_{1,i})}{{b}_1(A_{1,i}|H_{1,i})}\right\|_{P,2}\cdot \left\|\widehat{\mathbb{E}}[(\mathbb{I}\{{R}_{1}^*(\pi)+{R}_{2}^*(\pi)<\eta_\tau\}-\tau)|H_{1}]\right\|_{P,2}\\
      &\leq \frac{2}{\epsilon}\cdot \left\|\widehat{F}_{R^*(\pi)|H_1}(\eta_\tau;\hat{\alpha}|H_1)-{F}_{R^*(\pi)|H_1}(\eta_\tau|H_1)\right\|_{P,2}+\frac{2}{\epsilon^2}\left\|{\widehat{b}_1(A_{1,i}|H_{1,i})}-{{b}_1(A_{1,i}|H_{1,i})}\right\|_{P,2}=o(1).
    \end{aligned}
\end{equation}
Under Assumption (A1), (A2) and (A5), the last equality in Formula (\ref{eq:56}) holds because
\begin{equation*}
    \begin{aligned}
      & \left\|\widehat{F}_{R^*(\pi)|H_1}(\eta_\tau;\hat{\alpha}|H_1)-{F}_{R^*(\pi)|H_1}(\eta_\tau|H_1)\right\|_{P,2}\leq \left\|\delta\left(\widehat{F}_{{R}^{*}(\pi)|H_1},F_{{R}^{*}(\pi)|H_1}\right)\right\|_{P,2}=o(n^{-1/4}).
    \end{aligned}
\end{equation*}
Similarly,
\begin{equation}
\small
    \begin{aligned}
      &(\text{Var}[c({\eta}_{\tau},\hat{\alpha})-c^*({\eta}_{\tau})])^{1/2}\leq \left\|c({\eta}_{\tau},\hat{\alpha})-c^*({\eta}_{\tau})\right\|_{P,2}\\
      &= \left\| a_2(\hat{\alpha})\cdot \widehat{\mathbb{E}}[(\mathbb{I}\{{R}_{1}+\widehat{R}_{2}^*(\pi)<\eta_\tau\}-\tau)|H_{2}]-a_2^*\cdot \widehat{\mathbb{E}}[(\mathbb{I}\{{R}_{1}+{R}_{2}^*(\pi)<\eta_\tau\}-\tau)|H_{2}]\right\|_{P,2}\\
      &\leq \left\| a_2(\hat{\alpha})\cdot\Big\{ \widehat{\mathbb{E}}[\mathbb{I}\{{R}_{1}+\widehat{R}_{2}^*(\pi)<\eta_\tau\}-\mathbb{I}\{{R}_{1}+{R}_{2}^*(\pi)<\eta_\tau\}|H_{2}]\Big\}\right\|_{P,2}\\
      &+\left\| \Big\{a_2(\hat{\alpha})-a_2^*\Big\}\cdot \widehat{\mathbb{E}}[(\mathbb{I}\{{R}_{1}+{R}_{2}^*(\pi)<\eta_\tau\}-\tau)|H_{2}]\right\|_{P,2}\\
      &\leq \frac{2}{\epsilon^2}\cdot \left\|\widehat{\mathbb{E}}[\mathbb{I}\{{R}_{1}+\widehat{R}_{2}^*(\pi)<\eta_\tau\}-\mathbb{I}\{{R}_{1}+{R}_{2}^*(\pi)<\eta_\tau\}|H_{2}]\right\|_{P,2}\\
      & +\frac{4}{\epsilon^3}\cdot\Big\{\left\|\widehat{b}_1(A_{1}|H_{1})-{b}_1(A_{1}|H_{1})\right\|_{P,2}+\left\|\widehat{b}_2(A_{2}|H_{2})-{b}_2(A_{2}|H_{2})\right\|_{P,2}\Big\}\\
      &\leq \frac{2}{\epsilon^2}\cdot \left\|\widehat{F}_{R_2^*(\pi)|H_2}(\eta_\tau-R_1;\hat{\alpha}|H_2)-{F}_{R_2^*(\pi)|H_2}(\eta_\tau-R_1|H_2)\right\|_{P,2}+\frac{4}{\epsilon^3}\cdot o(n^{-1/4})\\
      &\leq \frac{2}{\epsilon^2}\cdot\left\|\delta\left(\widehat{F}_{{R}_2^{*}(\pi)|H_2},F_{{R}_2^{*}(\pi)|H_2}\right)\right\|_{P,2}+o(n^{-1/4})=o(1),
    \end{aligned}
\end{equation}
where we used Assumption (A1), (A2), (A5) and Positivity assumption (C3$'$).

The proof of Step 2.2 is thus complete.

\noindent\textbf{Step 3}: Prove that $\left\|\widehat{\eta}_{\tau}^{\text{DR}}-\eta_{\tau}\right\|= O_p(n^{-1/2})$.

Similar to $\psi(W_i;\eta,\alpha)$, we split the original optimization problem $\Psi(W_i;\eta,\alpha)$ into three parts.
\begin{equation}
    \begin{aligned}
      \Psi(W_i;\eta,\hat{\alpha}):=A(\eta,\hat{\alpha})+B(\eta,\hat{\alpha})+C(\eta,\hat{\alpha}),
    \end{aligned}
\end{equation}
where 
\begin{equation}\label{eq:thm.6}
\begin{aligned}
    A(\eta,\hat{\alpha}) &=\frac{\pi_1(A_{1}|H_{1})\pi_2(A_{2}|H_{2})}{\widehat{b}_1(A_{1}|H_{1})\widehat{b}_2(A_{2}|H_{2})}\rho_\tau(R_{1}+{R}_{2}-\eta), \\
    B(\eta,\hat{\alpha}) &=a_1(\hat{\alpha})\cdot \widehat{\mathbb{E}}[\rho_\tau(\widehat{R}_{1}+\widehat{R}_{2}-\eta)|X_{1},(A_{1},A_{2})\sim \pi], \\
    C(\eta,\hat{\alpha}) &=a_2(\hat{\alpha})\cdot \widehat{\mathbb{E}}[\rho_\tau({R}_{1}+\widehat{R}_{2}-\eta)|H_{2},A_{2}\sim \pi],
\end{aligned}
\end{equation}
and $a_1(\hat{\alpha})$, $a_2(\hat{\alpha})$ are defined in formula (\ref{eq:13}). Specifically, when the nuisance parameters $\hat{\alpha}$ is replaced by the true values, we let $\Psi^*(W_i;\eta):=A^*(\eta)+B^*(\eta)+C^*(\eta)$ in the same manner. 

Furthermore, we denote $M_n(\eta,\hat{\alpha})=\mathbb{E}_n[\Psi(W;\eta,\hat{\alpha})]$ as the empirical average of our optimization function, and $M(\eta,\hat{\alpha})=\mathbb{E}[\Psi(W;\eta,\hat{\alpha})]$. When the models with nuisance parameters are correctly specified, we define $M_n^*(\eta)=\mathbb{E}_n[\Psi^*(W;\eta)]$ and $M^*(\eta)=\mathbb{E}[\Psi^*(W;\eta)]$.

According to Theorem 3.2.5 of \cite{van1996weak}, to finish the proof of this step, it suffices to show that for every $n$ and a sufficiently small $\delta$, the following two conditions hold:
\begin{enumerate}
    \item $\mathbb{E}\left[\sup_{\left\|\eta-\eta_{\tau}\right\|<\delta}\left|M_n(\eta,\hat{\alpha})-M^*(\eta)-M_n(\eta_{\tau},\hat{\alpha})+M^*(\eta_{\tau})\right|\right]\leq\frac{{\delta}}{\sqrt{n}}$;
    \item There exists a constant $C_4$, such that $M^*(\eta)-M^*(\eta_{\tau})\geq C_4\cdot\left\|\eta-\eta_{\tau}\right\|^2$.
\end{enumerate}

*Note: This is an extension to the original proof of Theorem 3.2.5. The result of Theorem 3.2.5 doesn't change after involving nuisance parameter $\hat{\alpha}$ in this case. Moreover, notice that the two conditions listed above hold for any $s\in\{1,\dots,S\}$, and we omitted the subscript $s$ to avoid redundancy.

\noindent\textbf{Step 3.1}: Prove $\mathbb{E}\left[\sup_{\left\|\eta-\eta_{\tau}\right\|<\delta}\left|M_n(\eta,\hat{\alpha})-M^*(\eta)-M_n(\eta_{\tau},\hat{\alpha})+M^*(\eta_{\tau})\right|\right]\leq\frac{{\delta}}{\sqrt{n}}$.

Since
\begin{equation}\label{eq:59}
    \begin{aligned}
        & \mathbb{E}\left[\sup_{\left\|\eta-\eta_{\tau}\right\|<\delta}\left|M_n(\eta,\hat{\alpha})-M^*(\eta)-M_n(\eta_{\tau},\hat{\alpha})+M^*(\eta_{\tau})\right|\right]\\
        &\leq \mathbb{E}\left[\sup_{\left\|\eta-\eta_{\tau}\right\|<\delta}\left|M_n(\eta,\hat{\alpha})-M(\eta,\hat{\alpha})-M_n(\eta_{\tau},\hat{\alpha})+M(\eta_{\tau},\hat{\alpha})\right|\right]\\
        &+ \sup_{\left\|\eta-\eta_{\tau}\right\|<\delta}\left|\left(M(\eta,\hat{\alpha})-M(\eta_{\tau},\hat{\alpha})\right)-\left(M^*(\eta)-M^*(\eta_{\tau})\right)\right|,
    \end{aligned}
\end{equation}
to finish the proof of this step, it suffices to show that the first term on the RHS of Formula (\ref{eq:59}) is less than or equal to $\frac{\delta}{\sqrt{n}}$, and the second term is actually $o(n^{-\frac{1}{2}}\left\|\eta-\eta_{\tau}\right\|)$. We will prove them in {Step 3.1.a} and {Step 3.1.b} accordingly.

\noindent\textbf{Step 3.1.a}: Prove that
$\mathbb{E}\left[\sup_{\left\|\eta-\eta_{\tau}\right\|<\delta}\left|M_n(\eta,\hat{\alpha})-M(\eta,\hat{\alpha})-M_n(\eta_{\tau},\hat{\alpha})+M(\eta_{\tau},\hat{\alpha})\right|\right]\leq\frac{\delta}{\sqrt{n}}$.

 Thus, it is equivalent to proving that 
\begin{equation}
    \begin{aligned}
      \mathbb{E}\left[\sup_{\eta}\left|M_n(\eta,\hat{\alpha})-M(\eta,\hat{\alpha})-M_n(\eta_{\tau},\hat{\alpha})+M(\eta_{\tau},\hat{\alpha})\right|\right]\leq O(n^{-1/2}\left\|\widehat{\eta}_{\tau}^{\text{DR}}-\eta_{\tau}\right\|).
    \end{aligned}
\end{equation}

Let $\mathcal{F}_{\delta}=\{\Psi(W;{\eta},\hat{\alpha})-\Psi(W;{\eta}_{\tau},\hat{\alpha}):\left\|\eta-\eta_{\tau}\right\|\leq\delta\}$,and  $\mathcal{F}_\eta=\{\Psi(W;{\eta},\hat{\alpha}):\left\|\eta-\eta_{\tau}\right\|\leq\delta\}$. 

We claim that for any $\eta_1$ and $\eta_2$ satisfying $\left\|\eta-\eta_{\tau}\right\|\leq \delta$, there exist a square-integrable function $M(W;\hat{\alpha})$, such that
\begin{equation}\label{eq:62}
    \left|\Psi(W;\eta_1,\hat{\alpha})-\Psi(W;\eta_2,\hat{\alpha})\right|\leq M(W;\hat{\alpha})\left\|\eta_1-\eta_2\right\|.
\end{equation}

To prove this claim, we split formula (\ref{eq:62}) into three parts:
\begin{equation}
    \begin{aligned}
      &\left|\Psi(W;\eta_1,\hat{\alpha})-\Psi(W;\eta_2,\hat{\alpha})\right|\leq\left|A(W;\eta_1,\hat{\alpha})-A(W;\eta_2,\hat{\alpha})\right|\\
      &\qquad +\left|B(W;\eta_1,\hat{\alpha})-B(W;\eta_2,\hat{\alpha})\right|+\left|C(W;\eta_1,\hat{\alpha})-C(W;\eta_2,\hat{\alpha})\right|.
    \end{aligned}
\end{equation}
Without the loss of generality, we assume that $\eta_1<\eta_2$. Recall that ${R}^*({\pi})$ is the conditional cumulative reward given $H_{1}$. Since
\begin{equation}
\small
    ({R}^*(\pi)-\eta_2)\mathbb{I}\{{R}^*(\pi)<\eta_2\}-({R}^*(\pi)-\eta_1)\mathbb{I}\{{R}^*(\pi)<\eta_1\}=\left\{\begin{matrix}\eta_1-\eta_2, & {R}^*(\pi)<\eta_1\\ {R}^*(\pi)-\eta_2, & \eta_1<{R}^*(\pi)<\eta_2\\ 0, & {R}^*(\pi)>\eta_2\end{matrix}\right.
\end{equation}
we have $\left|({R}^*(\pi)-\eta_2)\mathbb{I}\{{R}^*(\pi)<\eta_2\}-({R}^*(\pi)-\eta_1)\mathbb{I}\{{R}^*(\pi)<\eta_1\}\right|\leq \left\|\eta_2-\eta_1\right\|$. 

Thus,
\begin{equation}\label{eq:63}
    \begin{aligned}
      &\left|A(W;\eta_1,\hat{\alpha})-A(W;\eta_2,\hat{\alpha})\right|\\
      &=\left|\frac{\pi_1(A_{1}|H_{1})\pi_2(A_{2}|H_{2})}{\widehat{b}_1(A_{1}|H_{1})\widehat{b}_2(A_{2}|H_{2})}\left(\rho_{\tau}(R_{1}+{R}_{2}-\eta_1)-\rho_{\tau}(R_{1}+{R}_{2}-\eta_2)\right)\right|\\ 
      &\leq \frac{1}{\epsilon^2}\cdot\left|\rho_{\tau}({R}^*(\pi)-\eta_1)-\rho_{\tau}({R}^*(\pi)-\eta_2)\right|\\
      &\leq\frac{1}{\epsilon^2}\cdot\Big\{\left|({R}^*(\pi)-\eta_2)\mathbb{I}\{{R}^*(\pi)<\eta_2\}-({R}^*(\pi)-\eta_1)\mathbb{I}\{{R}^*(\pi)<\eta_1\}\right|+\tau\left\|\eta_2-\eta_1\right\|\Big\}\\
      &\leq \frac{\tau+1}{\epsilon^2}\cdot\left\|\eta_2-\eta_1\right\|.
    \end{aligned}
\end{equation}
Similarly,
\begin{equation}\label{eq:64}
    \begin{aligned}
      &\left|B(W;\eta_1,\hat{\alpha})-B(W;\eta_2,\hat{\alpha})\right|\\
      &=\left|a_1(\hat{\alpha})\widehat{\mathbb{E}}\left[\rho_{\tau}(\widehat{R}_{1}+\widehat{R}_{2}-\eta_1)-\rho_{\tau}(\widehat{R}_{1}+\widehat{R}_{2}-\eta_2)\big| H_1\right]\right|\\
      &\leq\frac{2}{\epsilon}\cdot\widehat{\mathbb{E}}\left[\left|\rho_{\tau}(\widehat{R}_{1}+\widehat{R}_{2}-\eta_1)-\rho_{\tau}(\widehat{R}_{1}+\widehat{R}_{2}-\eta_2)\right|\big|H_1\right]\leq \frac{2(\tau+1)}{\epsilon}\cdot\left\|\eta_2-\eta_1\right\|.
    \end{aligned}
\end{equation}
and
\begin{equation}\label{eq:655}
    \begin{aligned}
      &\left|C(W;\eta_1,\hat{\alpha})-C(W;\eta_2,\hat{\alpha})\right|\\
      &=\left|a_2(\hat{\alpha})\cdot\widehat{\mathbb{E}}\left[\rho_{\tau}({R}_{1,i}+\widehat{R}_{2}-\eta_1)-\rho_{\tau}({R}_{1}+\widehat{R}_{2}-\eta_2)\big|H_{2}\right]\right|\\
      &\leq \frac{2}{\epsilon^2}\cdot\widehat{\mathbb{E}}\left[\left|\rho_{\tau}({R}_{1}+\widehat{R}_{2}-\eta_1)-\rho_{\tau}({R}_{1}+\widehat{R}_{2}-\eta_2)\right|\big|H_{2}\right]\leq \frac{2(\tau+1)}{\epsilon^2}\cdot\left\|\eta_2-\eta_1\right\|.
    \end{aligned}
\end{equation}

Combining the result of formula (\ref{eq:63}), (\ref{eq:64}) and (\ref{eq:655}), we can simply set  $M(W;\hat{\alpha})=5(\tau+1)/\epsilon^2$, so that the condition in Formula (\ref{eq:62}) is satisfied. More specifically, $\Psi(W;\eta,\hat{\alpha})$ is Lipschitz continuous of order 1. Therefore, $\mathcal{F}_{\delta}$ has envelope function $F:=\delta \cdot M(W;\hat{\alpha})$. Based on Theorem 2.7.11 and Lemma 2.7.8 of \citet{van1996weak}, its bracketing number satisfies
\begin{equation}
    N_{[\text{ }]}(2\epsilon\left\|F\right\|,\mathcal{F}_{\delta},\left\|\cdot\right\|)\leq N(\epsilon,\{\eta:\left\|\eta-\eta_{\tau}\right\|\leq \delta\},\left\|\cdot\right\|)\leq \frac{C\delta}{\epsilon}.
\end{equation}
Hence, according to the maximal inequality with bracketing numbers, we have
\begin{equation}
    \mathbb{E}\left[\left\|\sqrt{n}(M_n(\eta,\hat{\alpha})-M(\eta,\hat{\alpha}))\right\|_{\mathcal{F}_{\delta}}\right]\ \preceq \int_{0}^{\left\|F\right\|} \sqrt{\log N_{[\text{ }]}(\epsilon,\mathcal{F}_{\delta},\left\|\cdot\right\|)}\text{ }d\epsilon \preceq \delta.
\end{equation}
Plugging in the definition of $\|\cdot\|_{\mathcal{F}_\delta}$, it shows that
\begin{equation*}
    \mathbb{E}\left[\sup_{\left\|\eta-\eta_{\tau}\right\|<\delta}\left|M_n(\eta,\hat{\alpha})-M(\eta,\hat{\alpha})-M_n(\eta_{\tau},\hat{\alpha})+M(\eta_{\tau},\hat{\alpha})\right|\right]\leq O(n^{-1/2}\left\|\widehat{\eta}_{\tau}^{\text{DR}}-\eta_{\tau}\right\|)\leq \frac{\delta}{\sqrt{n}}.
\end{equation*}
Thus, the claim of Step 3.1.a follows.

\noindent\textbf{Step 3.1.b:} Prove $\sup_{\left\|\eta-\eta_{\tau}\right\|<\delta}\left|\left(M(\eta,\hat{\alpha})-M(\eta_{\tau},\hat{\alpha})\right)-\left(M^*(\eta)-M^*(\eta_{\tau})\right)\right|=o(n^{-\frac{1}{2}}\left\|{\eta}-\eta_{\tau}\right\|)$.

By using exactly the same strategy as deriving formula (\ref{eq:lemma2.5}) in Lemma 2, we can obtain similar results for $\Psi(W;\eta,\hat{\alpha})$:
\begin{equation}\label{eq:70}
\small
    \begin{aligned}
    &M(\eta,\hat{\alpha})-M^*(\eta)=\mathbb{E}[\Psi(W;\eta,\hat{\alpha})]-\mathbb{E}[\Psi^*(W;\eta)]\\
    =&\mathbb{E}\left[\left(\frac{\pi_1(A_{1}|H_{1})}{\widehat{b}_1(A_{1}|H_{1})}\right)\left(\frac{\pi_2(A_{2}|H_{2})}{\widehat{b}_2(A_{2}|H_{2})}-\frac{\pi_2(A_{2}|H_{2})}{{b}_2(A_{2}|H_{2})}\right) \right.\\
    &\qquad \left.\cdot \left(\widehat{\mathbb{E}}[\rho_{\tau}(R_{1}+{R}^{*}_{2}(\pi)-\eta)\big| H_{2}]-\widehat{\mathbb{E}}[\rho_{\tau}(R_{1}+\widehat{R}^{*}_{2}(\pi)-\eta)\big| H_{2}]\right)\right]\\
    +&\mathbb{E}\left[\left(\frac{\pi_1(A_{1}|H_{1})}{\widehat{b}_1(A_{1}|H_{1})}-\frac{\pi_1(A_{1}|H_{1})}{{b}_1(A_{1}|H_{1})}\right) \right.\\
    &\qquad \left.\cdot \left(\widehat{\mathbb{E}}[\rho_
    {\tau}(R^*_{1}(\pi)+{R}^{*}_{2}(\pi)-\eta)\big| H_{1}]-\widehat{\mathbb{E}}[\rho_{\tau}(\widehat{R}^{*}_{1}(\pi)+\widehat{R}^{*}_{2}(\pi)-\eta)\big| H_{1}]\right)\right].
    \end{aligned}
\end{equation}

Similarly, we can get the expression for $M(\eta_{\tau},\hat{\alpha})-M^*(\eta_{\tau})$ by substituting $\eta_{\tau}$ for $\eta$ in Formula (\ref{eq:70}).

Therefore, 

\begin{equation}\label{eq:71}
\small
    \begin{aligned}
    &\left|M(\eta,\hat{\alpha})-M^*(\eta)-\left(M(\eta_{\tau},\hat{\alpha})-M^*(\eta_{\tau})\right)\right|\\
    &\leq \left|\mathbb{E}\left[\left(\frac{\pi_1(A_{1}|H_{1})}{\widehat{b}_1(A_{1}|H_{1})}\right)\left(\frac{\pi_2(A_{2}|H_{2})}{\widehat{b}_2(A_{2}|H_{2})}-\frac{\pi_2(A_{2}|H_{2})}{{b}_2(A_{2}|H_{2})}\right) \cdot \widehat{\mathbb{E}}[\rho_{\tau}(R_{1}+{R}^{*}_{2}(\pi)-\eta)\right.\right.\\
    &\left.\left.\qquad-\rho_{\tau}(R_{1}+\widehat{R}^{*}_{2}(\pi)-\eta)-\rho_{\tau}(R_{1}+{R}^{*}_{2}(\pi)-\eta_{\tau})+\rho_{\tau}(R_{1}+\widehat{R}^{*}_{2}(\pi)-\eta_{\tau})\big| H_{2}]\right]\right|\\
    &+\left|\mathbb{E}\left[\left(\frac{\pi_1(A_{1}|H_{1})}{\widehat{b}_1(A_{1}|H_{1})}-\frac{\pi_1(A_{1}|H_{1})}{{b}_1(A_{1}|H_{1})}\right) \cdot \widehat{\mathbb{E}}[\rho_{\tau}(R^*_{1}(\pi)+{R}^{*}_{2}(\pi)-\eta)\right.\right.\\
    &\left.\left.\qquad-\rho_{\tau}(\widehat{R}^{*}_{1}(\pi)+\widehat{R}^{*}_{2}(\pi)-\eta)-\rho_
    {\tau}(R^*_{1}(\pi)+{R}^{*}_{2}(\pi)-\eta_{\tau})+\rho_{\tau}(\widehat{R}^{*}_{1}(\pi)+\widehat{R}^{*}_{2}(\pi)-\eta_{\tau})\big| H_{1}]\right]\right|.
    \end{aligned}
\end{equation}
Let's first focus on bounding the last term in formula (\ref{eq:71}).
\begin{equation}
\small
    \begin{aligned}
    &  \left|\widehat{\mathbb{E}}\left[\rho_{\tau}(R^*_{1}(\pi)+{R}^{*}_{2}(\pi)-\eta)-\rho_{\tau}(\widehat{R}^{*}_{1}(\pi)+\widehat{R}^{*}_{2}(\pi)-\eta)\right.\right.\\
    &\left.\left.\qquad-\rho_
    {\tau}(R^*_{1}(\pi)+{R}^{*}_{2}(\pi)-\eta_{\tau})+\rho_{\tau}(\widehat{R}^{*}_{1}(\pi)+\widehat{R}^{*}_{2}(\pi)-\eta_{\tau})\big| H_{1}\right]\right|\\
    &= \left|\int_{r\in \mathbb{R}} \left(\rho_\tau(r-\eta_{\tau})-\rho_\tau(r-\eta)\right)\cdot\left(\widehat{f}_{R^*(\pi)|H_1}(r|H_{1})-{f}_{R^*(\pi)|H_1}(r|H_{1})\right) dr\right|\\
    & \leq \int_{r\in \mathbb{R}} \left|\rho_\tau(r-\eta_{\tau})-\rho_\tau(r-\eta)\right|\cdot\left|\widehat{f}_{R^*(\pi)|H_1}(r|H_{1})-{f}_{R^*(\pi)|H_1}(r|H_{1})\right| dr.
    \end{aligned}
\end{equation}
By mean value theorem, for any fixed $r$, there exists a constant \\$\eta_r\in \left(\min\{\eta,\eta_{\tau}\},\max\{\eta,\eta_{\tau}\}\right)$, such that
\begin{equation}
    \begin{aligned}
      \rho_\tau(r-\eta_{\tau})-\rho_\tau(r-\eta)=\left(\tau-\mathbb{I}\{r< \eta_r\}\right)\cdot (\eta-\eta_{\tau}).
    \end{aligned}
\end{equation}

Therefore, for any $r$, it holds that
\begin{equation}
    \begin{aligned}
      \left|\rho_\tau(r-\eta_{\tau})-\rho_\tau(r-\eta)\right|=\left|\left(\tau-\mathbb{I}\{r< \eta_r\}\right)\right|\cdot \left\|\eta-\eta_{\tau}\right\|\leq \left\|\eta-\eta_{\tau}\right\|.
    \end{aligned}
\end{equation}
Under mild conditions and Assumption (A2),
\begin{equation}\label{eq:711}
\small
    \begin{aligned}
    &  \left|\widehat{\mathbb{E}}\left[\rho_{\tau}(R^*_{1}(\pi)+{R}^{*}_{2}(\pi)-\eta)-\rho_{\tau}(\widehat{R}^{*}_{1}(\pi)+\widehat{R}^{*}_{2}(\pi)-\eta)\right.\right.\\
    &\left.\left.\qquad-\rho_
    {\tau}(R^*_{1}(\pi)+{R}^{*}_{2}(\pi)-\eta_{\tau})+\rho_{\tau}(\widehat{R}^{*}_{1}(\pi)+\widehat{R}^{*}_{2}(\pi)-\eta_{\tau})\big| H_{1}\right]\right|\\
    & \leq \int_{r\in \mathbb{R}} \left|\rho_\tau(r-\eta_{\tau})-\rho_\tau(r-\eta)\right|\cdot\left|\widehat{f}_{R^*(\pi)|H_1}(r|H_{1})-{f}_{R^*(\pi)|H_1}(r|H_{1})\right| dr.\\
    &\leq \left\|\eta-\eta_{\tau}\right\|\cdot \int_{r\in \mathbb{R}} \left|\widehat{f}_{R^*(\pi)|H_1}(r|H_{1})-{f}_{R^*(\pi)|H_1}(r|H_{1})\right| dr\\
    &=\left\|\eta-\eta_{\tau}\right\|\cdot 2\delta\left(\widehat{F}_{{R}^*(\pi)|H_1}-F_{R^*(\pi)|H_1}\right).
    \end{aligned}
\end{equation}

Similarly, we can prove by Assumption (A2) that
\begin{equation}\label{eq:712}
\small
    \begin{aligned}
    &  \left|\widehat{\mathbb{E}}[\rho_{\tau}(R_{1}+{R}^{*}_{2}(\pi)-\eta)-\rho_{\tau}(R_{1}+\widehat{R}^{*}_{2}(\pi)-\eta)-\rho_{\tau}(R_{1}+{R}^{*}_{2}(\pi)-\eta_{\tau})+\rho_{\tau}(R_{1}+\widehat{R}^{*}_{2}(\pi)-\eta_{\tau})\big| H_{2}]\right|\\
    & \leq \int_{r\in \mathbb{R}} \left|\rho_\tau(R_{1}+r-\eta_{\tau})-\rho_\tau(R_{1}+r-\eta)\right|\cdot\left|\widehat{f}_{R_2^*(\pi)|H_2}(r|H_{2})-{f}_{R_2^*(\pi)|H_2}(r|H_{2})\right| dr\\
    &\leq \left\|\eta-\eta_{\tau}\right\|\cdot \int_{r\in \mathbb{R}} \left|\widehat{f}_{R_2^*(\pi)|H_2}(r|H_{2})-{f}_{R_2^*(\pi)|H_2}(r|H_{2})\right| dr\\
    &=\left\|\eta-\eta_{\tau}\right\|\cdot 2\delta\left(\widehat{F}_{{R}_2^*(\pi)|H_2}-F_{R_2^*(\pi)|H_2}\right).
    \end{aligned}
\end{equation}

By summarizing the result of Formula (\ref{eq:71}), (\ref{eq:711}) and (\ref{eq:712}), we have
\begin{equation}\label{eq:713}
\small
    \begin{aligned}
    &\left|M(\eta,\hat{\alpha})-M^*(\eta)-\left(M(\eta_{\tau},\hat{\alpha})-M^*(\eta_{\tau})\right)\right|\\
    &\leq \left\|\left(\frac{\pi_1(A_{1}|H_{1})}{\widehat{b}_1(A_{1}|H_{1})}\right)\left(\frac{\pi_2(A_{2}|H_{2})}{\widehat{b}_2(A_{2}|H_{2})}-\frac{\pi_2(A_{2}|H_{2})}{{b}_2(A_{2}|H_{2})}\right)\right\|_{P,2} \cdot 2\left\|\eta-\eta_{\tau}\right\|\cdot\left\|\delta\left(\widehat{F}_{{R}_2^{*}(\pi)|H_2},F_{{R}_2^{*}(\pi)|H_2}\right)\right\|_{P,2}\\
    &+\left\|\frac{\pi_1(A_{1}|H_{1})}{\widehat{b}_1(A_{1}|H_{1})}-\frac{\pi_1(A_{1}|H_{1})}{{b}_1(A_{1}|H_{1})}\right\|_{P,2}\cdot 2\left\|\eta-\eta_{\tau}\right\|\cdot \left\|\delta\left(\widehat{F}_{{R}^{*}(\pi)|H_1},F_{{R}^{*}(\pi)|H_1}\right)\right\|_{P,2}\\
    & \leq \frac{1}{\epsilon^3}\left\|\widehat{b}_2(A_{2}|H_{2})-{b}_2(A_{2}|H_{2})\right\|_{P,2}\cdot 2\left\|\eta-\eta_{\tau}\right\|\cdot o(n^{-\frac{1}{4}})\\
    &+\frac{1}{\epsilon^2}\left\|\widehat{b}_1(A_{1}|H_{1})-{b}_1(A_{1}|H_{1})\right\|_{P,2}\cdot 2\left\|\eta-\eta_{\tau}\right\|\cdot o(n^{-\frac{1}{4}})=o(n^{-\frac{1}{2}}\left\|\eta-\eta_{\tau}\right\|),
    \end{aligned}
\end{equation}
where the last inequality holds according to Assumption (A1) and (A2).

\noindent\textbf{Step 3.2}: Prove that $M^*(\eta)-M^*(\eta_{\tau})\geq C_4\cdot\left\|\eta-\eta_{\tau}\right\|^2$.

According to Taylor series expansion, we have
\begin{equation}
      M^*(\eta)=M^*(\eta_{\tau})+\mathbb{E}[\psi^*(W;\eta_{\tau})](\eta-\eta_{\tau})+\frac{1}{2}\partial_{\eta}\mathbb{E}[\psi^*(W;\eta_{\tau})](\eta-\eta_{\tau})^2+o(\left\|\eta-\eta_{\tau}\right\|^2),
\end{equation}
in which $\mathbb{E}[\psi^*(W;\eta_{\tau})]=0$, and 
\begin{equation}\label{eq:72}
    \begin{aligned}
      &\partial_{\eta}\mathbb{E}[\psi^*(W_i;\eta)]=\partial_{\eta}\mathbb{E}\left[\frac{\pi_1(A_{1}|H_{1})\pi_2(A_{2}|H_{2})}{\widehat{b}_1(A_{1}|H_{1})\widehat{b}_2(A_{2}|H_{2})}(\mathbb{I}\{R_{1}+{R}_{2}<\eta\}-\tau)\right]\\
      &+\partial_{\eta}\mathbb{E}\left[\mathbb{E}\left[a_1^*|H_1\right]\cdot\widehat{\mathbb{E}}[(\mathbb{I}\{{R}_{1}^*(\pi)+{R}_{2}^*(\pi)<\eta\}-\tau)|H_1]\right]\\
      &+ \partial_{\eta}\mathbb{E}\left[\mathbb{E}\left[a_2^*|H_2\right]\cdot\widehat{\mathbb{E}}[(\mathbb{I}\{{R}_{1}+{R}_{2}^*(\pi)<\eta\}-\tau)|H_{2}]\right]\\
      &= \mathbb{E}_{H_1\sim \mathbb{G}}\left[{f}_{{R}^{*}(\pi)|H_1}(\eta|H_1)\right]+0+0={f}_{{R}^{*}(\pi)}(\eta)\geq 0
    \end{aligned}
\end{equation}

where ${f}_{{R}^{*}(\pi)|H_1}(\eta|H_1)$ is the pdf of random variable $R_1^{*}(\pi_1)+R_2^{*}(\pi)$ given baseline information $H_1$, and ${f}_{{R}^{*}(\pi)}(\eta)$ is the marginal pdf of $R_1^{*}(\pi_1)+R_2^{*}(\pi)$ given baseline distribution $H_1\sim \mathbb{G}$.

Therefore, by Assumption (A4), 
\begin{equation}
    M^*(\eta)-M^*(\eta_{\tau})=\frac{1}{2}\cdot{f}_{{R}^{*}(\pi)}(\eta)(\eta-\eta_{\tau})^2+o(\left\|\eta-\eta_{\tau}\right\|^2)\geq C_4\cdot \left\|\eta-\eta_{\tau}\right\|^2,
\end{equation}
where $C_4$ can be defined as $C_1/2-\delta$ with a small enough $\delta$ such that $C_4>0$.

Therefore, the proof of Step 3 is complete.

\noindent\textbf{Step 4}: Prove that $\sup_{\|\eta-\eta_{\tau}\|\leq \delta}\left|\mathbb{E}[\psi(W;\eta,\hat{\alpha})-\psi^*(W;\eta)]\right|=o(n^{-1/2})$.

According to the derivation we did in formula (\ref{eq:lemma2.5}),
\begin{equation}
\small
    \begin{aligned}
      &\sup_{\|\eta-\eta_{\tau}\|\leq \delta}\cdot\left|\mathbb{E}[\psi(W;\eta,\hat{\alpha})-\psi^*(W;\eta)]\right|\\
      = &\sup_{\|\eta-\eta_{\tau}\|\leq \delta}\cdot\left|\mathbb{E}\left[\left(\frac{\pi_1(A_{1,i}|H_{1,i})}{\widehat{b}_1(A_{1,i}|H_{1,i})}\right)\left(\frac{\pi_2(A_{2,i}|H_{2,i})}{\widehat{b}_2(A_{2,i}|H_{2,i})}-\frac{\pi_2(A_{2,i}|H_{2,i})}{{b}_2(A_{2,i}|H_{2,i})}\right)\right.\right.\\
    &\left.\left.\qquad\qquad\qquad\cdot \left(\mathbb{E}[\mathbb{I}\{R_{1,i}+R^{*}_{2}(\pi)<\eta\}\big| H_{2,i}]-\widehat{\mathbb{E}}[\mathbb{I}\{R_{1,i}+\widehat{R}^{*}_{2}(\pi)<\eta\}\big| H_{2,i}]\right)\right]\right.\\
    &\qquad\qquad+\left.\mathbb{E}\left[\left(\frac{\pi_1(A_{1,i}|H_{1,i})}{\widehat{b}_1(A_{1,i}|H_{1,i})}-\frac{\pi_1(A_{1,i}|H_{1,i})}{{b}_1(A_{1,i}|H_{1,i})}\right)\right.\right.\\
    &\left.\left.\qquad\qquad\qquad\cdot \left( \widehat{\mathbb{E}}[\mathbb{I}\{R^*_{1}(\pi)+R^*_{2}(\pi)<\eta\}\big|H_{1,i}]-\widehat{\mathbb{E}}[\mathbb{I}\{\widehat{R}^{*}_{1}(\pi)+\widehat{R}^{*}_{2}(\pi)<\eta\}\big| H_{1,i}]\right)\right]\right|\\
    &\leq \frac{1}{\epsilon^3}\cdot \left\|\widehat{b}_2(A_{2}|H_{2})-{b}_2(A_{2}|H_{2})\right\|_{P,2}\cdot\left\|\delta\left(\widehat{F}_{{R}^{*}_2({\pi})|H_2},F_{{R}_2^{*}({\pi})|H_2}\right)\right\|_{P,2}\\
    & + \frac{1}{\epsilon^2}\cdot \left\|\widehat{b}_1(A_{1}|H_{1})-{b}_1(A_{1}|H_{1})\right\|_{P,2}\cdot\left\|\delta\left(\widehat{F}_{{R}^{*}({\pi})|H_1},F_{{R}^{*}({\pi})|H_1}\right)\right\|_{P,2}\\
    &=o(n^{-1/4})\cdot o(n^{-1/4})+o(n^{-1/4})\cdot o(n^{-1/4})=o(n^{-1/2}),
      \end{aligned}
\end{equation}
where the order in the second last equality comes from Assumption (A1) and (A2).

\textbf{Step 5}: Prove that $H_0$ is bounded.

In Formula (\ref{eq:72}), we obtained that $\partial_{\eta}\mathbb{E}[\psi^*(W_i;\eta)]=\mathbb{E}_{H_1\sim \mathbb{G}}\left[{f}_{{R}^{*}(\pi)}(\eta|H_1)\right]={f}_{{R}^{*}(\pi)}(\eta)$.
Therefore,
\begin{equation}
    \begin{aligned}
      H_0=\partial^2_{\eta}\mathbb{E}[\psi^*(W_i;\eta)]\Big|_{\eta=\eta_{\tau}}= \partial_{\eta}{f}_{{R}^{*}(\pi)}(\eta)\Big|_{\eta=\eta_{\tau}}.
    \end{aligned}
\end{equation}
This condition holds naturally according to Assumption (A6).

\vspace{0.15in}
By combining the results in Step 2-5, one can follow the structure in Step 1 and the whole proof of Theorem 1 is thus complete.
\end{proof}

\end{document}